%% file: example_paper.tex
\documentclass{article}

\usepackage{microtype}
\usepackage{graphicx}
\usepackage{subfigure}
\usepackage{booktabs} 

\usepackage{hyperref}
\input{math_commands.tex}

\newcommand{\widgraph}[2]{\includegraphics[keepaspectratio,width=#1]{#2}}

\usepackage[accepted]{icml2022}

\usepackage{amsmath}
\usepackage{amssymb,bbm}
\usepackage{mathtools}
\usepackage{amsthm}
\usepackage{amsfonts}
\usepackage{multirow}
\usepackage{multicol}

\usepackage[capitalize,noabbrev]{cleveref}

\newtheorem{theorem}{Theorem}
\newtheorem{example}{Example}

\newtheorem{definition}{Definition}
\def\RR{\mathbb{R}}

\usepackage[textsize=tiny]{todonotes}

\icmltitlerunning{Improving Mini-batch Optimal Transport via Partial Transportation}

\begin{document}

\twocolumn[
\icmltitle{Improving Mini-batch Optimal Transport via Partial Transportation}



\icmlsetsymbol{equal}{*}

\begin{icmlauthorlist}
\icmlauthor{Khai Nguyen}{yyy,equal}
\icmlauthor{Dang Nguyen}{comp,equal}
\icmlauthor{The-Anh Vu-Le}{comp}
\icmlauthor{Tung Pham}{comp}
\icmlauthor{Nhat Ho}{yyy}
\end{icmlauthorlist}

\icmlaffiliation{yyy}{Department of Statistics and Data Sciences, The University of Texas at Austin}
\icmlaffiliation{comp}{VinAI Research}

\icmlcorrespondingauthor{Khai Nguyen}{khainb@utexas.edu}

\icmlkeywords{Optimal Transport, Wasserstein, Minibatching, Generative models, Domain Adaptation}

\vskip 0.3in
]



\printAffiliationsAndNotice{\icmlEqualContribution} 

\begin{abstract}
Mini-batch optimal transport (m-OT) has been widely used recently to deal with the \emph{memory issue} of OT in large-scale applications. Despite their practicality, m-OT suffers from misspecified mappings, namely, mappings that are optimal on the mini-batch level but are partially wrong in the comparison with the optimal transportation plan between the original measures. Motivated by the misspecified mappings issue, we propose a novel mini-batch method by using partial optimal transport (POT) between mini-batch empirical measures, which we refer to as \emph{mini-batch partial optimal transport} (m-POT). Leveraging the insight from the partial transportation, we explain the source of misspecified mappings from the m-OT and motivate why limiting the amount of transported masses among mini-batches via POT can alleviate the incorrect mappings. Finally, we carry out extensive experiments on various applications such as deep domain adaptation, partial domain adaptation, deep generative model, color transfer, and gradient flow to demonstrate the favorable performance of m-POT compared to current mini-batch methods.

\end{abstract}

\section{Introduction}
\label{sec:introduction}
From its origin in mathematics and economics, optimal transport (OT) has recently become a useful and popular tool in machine learning applications, such as (deep) generative models~\cite{arjovsky2017wasserstein, tolstikhin2018wasserstein}, computer vision/ graphics~\cite{solomon2016entropic, Nguyen_3D}, clustering problem~\cite{ho2017multilevel, Ho_Probabilistic}, and domain adaptation~\cite{courty2016optimal, damodaran2018deepjdot, le2021lamda}. An important impetus for such popularity is the recent advance in the computation of OT. In particular, the works of~\cite{altschuler2017near, Dvurechensky-2018-Computational, lin2019efficient} demonstrate that we can approximate optimal transport with a computational complexity  of the order $\mathcal{O}(n^2/ \varepsilon^2)$ where $n$ is the maximum number of supports of probability measures and $\varepsilon > 0$ is the desired tolerance. It is a major improvement over the standard computational complexity $\mathcal{O}(n^3 \log n)$ of computing OT via interior point methods~\cite{pele2009}.
Another approach  named sliced OT in ~\cite{bonneel2015sliced, nguyen2021distributional, nguyen2021improving, deshpande2019max} is able to reduce greatly the computational cost to the order of $\mathcal{O}(n \log n)$  by exploiting the closed-form of one dimensional OT on the projected data. However, due to the projection step, sliced optimal transport is not able to keep all the main differences between high dimensional measures, which leads to a decline in practical performance.

 
Despite these major improvements on computation, it is still impractical to use optimal transport and its variants in large-scale applications when $n$ can be as large as a few million, such as deep domain adaptation and deep generative model. The bottleneck mainly comes from the \emph{memory issue}, namely, it  is impossible for us to store $n\times n$ cost matrix when $n$ is extremely large. To deal with this issue, practitioners often replace the original large-scale computation of OT with cheaper computation on subsets of the whole dataset, which is widely referred to as mini-batch approaches \cite{arjovsky2017wasserstein, genevay2018learning,sommerfeld2019optimal}. In more detail, instead of computing the original large-scale optimal transport problem, the mini-batch approach splits it into smaller transport problems (sub-problems). Each sub-problem is to solve optimal transport between subsets of supports (or mini-batches), which belong to the original measures. After solving all the sub-problems, the final transportation cost (plan) is obtained by aggregating evaluated mini-batch transportation costs (plans) from these sub-problems. Due to the small number of samples in each mini-batch, computing OT between mini-batch empirical measures is doable when memory constraints and computational constraints exist. This mini-batch approach is formulated rigorously under the name of mini-batch OT (m-OT) and some of its statistical properties are investigated in~\cite{fatras2020learning, fatras2021minibatch}.

Despite its computational practicality, the trade-off of the m-OT approach is the  \emph{misspecified mappings} in comparison with the full OT transportation plan. In particular, a mini-batch is a sparse representation of the original supports; hence, solving an optimal transport problem between empirical mini-batch measures tends to create transport mappings that do not match  the global optimal transport mapping between the original probability measures. The existence of misspecified mappings had been noticed in \cite{fatras2021unbalanced}, hence,~\cite{fatras2021unbalanced} proposed to use unbalanced optimal transport (UOT)~\cite{chizat2018unbalanced} as a replacement of OT for the transportation between mini-batches, which they referred to as mini-batch UOT (m-UOT), and showed that m-UOT can reduce the effect of misspecified mappings from the m-OT. They further provided an objective loss based on m-UOT,  which achieves considerably better results on deep domain adaptation. However, it is intuitively hard to understand the transportation plan from m-UOT. Also, we observe that m-UOT suffers from an issue that its hyperparameter ($\tau$) is not robust to the scale of the cost matrix, namely, the value of m-UOT's hyperparameter needs to be chosen to be large if the distances between supports are large, and vice versa. Therefore, 
in applications such as  generative models where the supports of measures change significantly, it remains a challenge to use m-UOT.  

\textbf{Contributions.} In this work, we develop a novel  mini-batch framework to alleviate the misspecified matchings issue of m-OT, the aforementioned issue of m-UOT, and to be robust to the scale of the cost matrix. In short, our contributions can be summarized as follows:

1. We propose to use partial optimal transport (POT) between empirical measures formed by mini-batches that can reduce the effect of misspecified mappings. The new mini-batch framework, named \emph{mini-batch partial optimal transport} (m-POT), has two practical applications: (i) the first application is an efficient mini-batch transportation cost used for objective functions in deep learning problems; (ii) the second application is a meaningful mini-batch transportation plan used for barycentric mappings in color transfer problem. Finally, via some simple examples, we further argue why partial optimal transport (POT) can be a natural solution for the misspecified mappings.

2. We conduct extensive experiments on applications that benefit from using mini-batches, including deep domain adaption, partial domain adaptation, deep generative model, color transfer, and gradient flow to compare m-POT with m-OT and its variant m-UOT~\cite{fatras2021unbalanced}. From experimental results, we observe that m-POT is better than m-OT on all deep domain adaptation tasks. In particular, m-POT gives at least $5.28$ higher on the average accuracy than m-OT. Comparison with m-UOT, it further pushes the accuracy a little bit higher on all datasets. On partial domain adaptation tasks, the m-POT also leads to a remarkable improvement of $2.02$ on average over recent methods in the literature. Finally, experiments on the deep generative model, color transfer, and gradient flow also demonstrate the favorable performance of m-POT.

3. We introduce a new training approach for deep domain adaptation with optimal transport losses. In particular, we propose the \textit{two-stage} implementation that first solves a relatively large mini-batch optimal transport problem on computer memory (RAM) then utilizes the obtained alignment to estimate the gradient from smaller mini-batches on high-speed computational devices (e.g., GPU). Combining with partial transportation, we observe considerable improvement in terms of classification accuracy in target domains.


\textbf{Organization.} The paper is organized as follows. In Section~\ref{sec:background}, we provide backgrounds on (unbalanced) optimal transport and their mini-batch versions, then we highlight the limitations of each mini-batch method. In Section~\ref{sec:mPOT}, we propose mini-batch partial optimal transport to alleviate the limitations of previous mini-batch methods. 
We provide extensive experiment results of mini-batch POT in Section~\ref{sec:experiment} and conclude the paper with a few discussions in Section~\ref{sec:Discussion}. Extra experimental, theoretical results and settings are deferred to the Appendices.

\textbf{Notation:} For two discrete probability measures $\boldsymbol{\alpha}$ and $\boldsymbol{\beta}$, $\Pi(\boldsymbol{\alpha},\boldsymbol{\beta}):=\big\{ \pi \in \mathbb{R}_+^{|\boldsymbol{\alpha}|\times|\boldsymbol{\beta}|} : \pi 1_{|\boldsymbol{\beta}|} = \boldsymbol{\alpha}, \pi^\top 1_{|\boldsymbol{\alpha}|} = \boldsymbol{\beta} \big\}$ is the set of transportation plans between $\mu$ and $\nu$, where $|\boldsymbol{\alpha}|$ denotes the number of supports of $\boldsymbol{\alpha}$, $||\boldsymbol{\alpha}||$ denotes total masses of $\boldsymbol{\alpha}$. Also, we denote $\boldsymbol{u}_n$ as the uniform distribution over $n$ supports (similar definition with $\boldsymbol{u}_m$). For a set of $m$ samples $X^m:=\{x_1,\ldots,x_m\}$, $P_{X^m}$ denotes the empirical measures $\frac{1}{m} \sum_{i=1}^m \delta_{x_i}$. For any $x \in \mathbb{R}$, we denote by $\floor{x}$ the greatest integer less than or equal to $x$. For any probability measure $\mu$ on the Polish measurable space $(\mathcal{X},\Sigma)$, we denote $\mu^{\otimes m} (m \geq 2)$ as the product measure on the product measurable space $(\mathcal{X}^{ m},\Sigma^{m})$.

\section{Background}
\label{sec:background}

In this section, we first restate the definition of Kantorovich optimal transport (OT) and unbalanced Optimal transport (UOT) between two empirical measures. After that, we review the definition of the mini-batch optimal transport (m-OT), mini-batch unbalanced optimal transport (m-UOT), and discuss misspecified matchings issue of m-OT and some challenges of m-UOT. 

\begin{figure*}[!t]
    \begin{center}
    \begin{tabular}{c}
    \widgraph{0.95\textwidth}{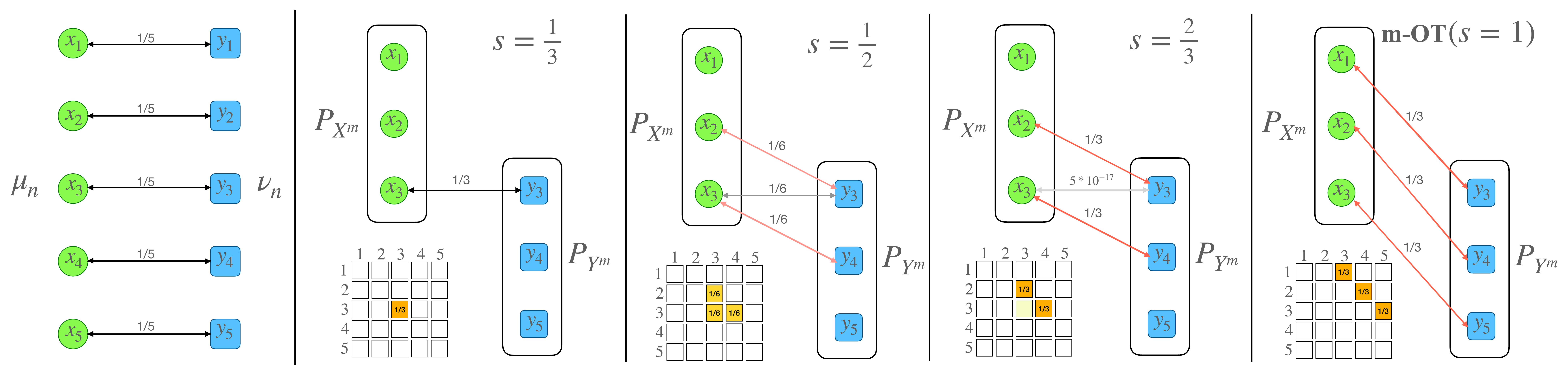} 
    \end{tabular}
    \end{center}
    \vskip -0.1in
    \caption{
        \footnotesize{ The illustration of Example~\ref{example:illustration} for m-POT for different values of $s$ including $s=1$ (m-OT). The green points and blue points are respectively the supports of the empirical measures $\mu_n$ and $\nu_n$. Black solid arrows and associate weights represent the optimal mappings between $\mu_n$ and $\nu_n$. Red solid arrows represent misspecified mappings. The $5\times 5$ matrix is the incomplete transportation matrix $\pi^{\text{POT}_s}_{P_{X^m},P_{Y^m}}$ which is created from solving POT between $P_{X^m}$ and $P_{Y^m}$. 
    }} 
    \label{fig:mOT}
    \vskip -0.2in
\end{figure*}

\subsection{Optimal Transport and Unbalanced Optimal Transport}
\label{subsec:OTdistances}
Let $\mathcal{X}:=\{x_i\}_{i=1}^n$, $\mathcal{Y}:=\{y_j\}_{j=1}^{n}$ be two interested samples. The corresponding empirical measures are denoted by $\mu_n:=\frac{1}{n} \sum_{i=1}^n \delta_{x_i}$ and $\nu_n:=\frac{1}{n}\sum_{j=1}^{n}  \delta_{y_j}$.

\textbf{Optimal Transport:} The Kantorovich optimal transport \cite{villani2009optimal,peyre2019computational} between $\mu_n$ and $\nu_n$ is defined as follows: 
$
    \text{OT}(\mu_n,\nu_n) := \min_{\pi \in \Pi(\boldsymbol{u}_n,\boldsymbol{u}_n)} \langle  C,\pi\rangle, 
$ where $C$ is the distance matrix (or equivalently cost matrix) between $\mathcal{X}$ and $\mathcal{Y}$ that is produced by a ground metric (e.g., Euclidean distance or other designed distances).

\textbf{Unbalanced Optimal Transport:} The unbalanced optimal transport \cite{chizat2018unbalanced} between $\mu_n$ and $\nu_n$ is defined as follows: 
$
    \text{UOT}^{\tau}_\phi(\mu_n,\nu_n)  := \min_{\pi \in \mathbb{R}_+^{n\times n}} \langle C,\pi \rangle + \tau \text{D}_\phi (\pi_{1} , \mu_{n}) 
    + \tau \text{D}_\phi (\pi_{2} ,\nu_{n}),
$
where $C$ is the distance matrix, $\tau > 0$ is a regularized parameter, $D_\phi$ is a certain probability divergence (e.g., KL divergence and total variational distance), and $\pi_{1}$, $\pi_{2}$ are respectively the marginal distributions of non-negative measure $\pi$.
The computational cost of both OT and UOT problems are of order $\mathcal{O}(n^2/\varepsilon^2)$ and $\mathcal{O}(n^2/\varepsilon)$, respectively. Its solutions are obtained by running the Sinkhorn algorithm \cite{cuturi2013sinkhorn}, which updates directly the whole $n\times n$ matrix. It means that storing an $n\times n$ matrix is unavoidable in this approach, thus the memory capacity needs to match the matrix size.

\subsection{Mini-batch Optimal Transport}
\label{subsec:mOT_mUOT}

In real applications, the number of samples $n$ is usually very large (e.g., millions). It is due to the large-scale empirical measures or detailed discretization of continuous measures. Therefore, solving directly OT between $\mu_n$ and $\nu_n$ is generally impractical due to the limitation of computational devices, namely, memory constraints, and vast computation. As a solution, the original $n$ samples of two measures are divided (via sampling with or without replacement) into subsets of $m$ samples, which we refer to as mini-batches. The mini-batch size $m$ is often chosen to be the largest number that the computational device can process. Then, a mini-batch framework is developed to aggregate the optimal transport between pairs of the corresponding mini-batches into a global result.

\textbf{Motivating examples: } We now provide some motivating examples to further illustrate the practical importance of mini-batch methods. The first example is regarding training a deep learning model. In practice, it is trained by a loss that requires computing a large-scale OT, e.g., deep generative models~\cite{genevay2018learning} and deep domain adaptation~\cite{damodaran2018deepjdot}. 
The size of the cost matrix cannot be large in practice since memory is also utilized  to store models and data. The second example is color transfer application when the numbers of pixels in both source and target images are very large (e.g., millions). The mini-batch approach is used to transport a small number of pixels from source images to a small number of pixels from target images~\cite{fatras2020learning}. That process is repeated for a large number of iterations.

\textbf{Related methods}: As another option, stochastic optimization can be utilized to solve the Kantorovich dual form with parametric functions, i.e., Wasserstein GAN~\cite{arjovsky2017wasserstein,leygonie2019adversarial}. Due to the limitation of parametrization, it has been shown that this approach provides a very different type of discrepancy from the original Wasserstein distance~\cite{mallasto2019well,stanczuk2021wasserstein}. Recently, input convex neural networks are developed to approximate the Brenier potential~\cite{makkuva2020optimal}. Nevertheless, due to limited power in approximating Brenier potential~\cite{korotin2021continuous}, recent works have indicated that input convex neural networks are not sufficient for computing OT. Finally, both approaches require special choices of the ground metric of OT. Namely, Wasserstein GAN has to use the $\mathcal{L}_1$ norm  to make the constraint of dual form into the Lipchitz constraint and Brenier potential exists only when the ground metric is $\mathcal{L}_2$.

Next, we revise the definition of mini-batch optimal transport (m-OT)~\cite{fatras2020learning, fatras2021minibatch}. To ease the ensuing presentation, some notations in that paper are adapted into our paper. To build a mini-batch of $1 \leq m \leq n$ points, we sample $X^m:=\{x_1,\ldots,x_m\}$ with or without replacement from $\mathcal{X}^m$ (similarly, $Y^m$ are drawn from $\mathcal{Y}^m$) where $m$ is the mini-batch size.
\begin{definition}
\label{def:mOT}
(Mini-batch Optimal Transport) For $1 \leq m\leq n$ and $k \geq 1$, $X^m_1,\ldots,X^m_k$ and $Y^m_1,\ldots,Y^m_k$ are sampled with or without replacement from $\mathcal{X}^m$ and $\mathcal{Y}^m$, respectively. The m-OT transportation cost and transportation plan between $\mu_n$ and $\nu_n$ are defined as follow:
\begin{align}
    \text{m-OT}_k(\mu_n,\nu_n) &= \frac{1}{k} \sum_{i=1}^k \text{OT}(P_{X_i^m},P_{Y_{i}^m}); \nonumber \\
    \pi^{\text{m-OT}_k} &= \frac{1}{k} \sum_{i=1}^k \pi^{\text{OT}}_{P_{X_i^m},P_{Y_i^m}}, \label{eq: mini-batch_OT}
\end{align}
where $\pi^{\text{OT}}_{P_{X_i^m},P_{Y_i^m}}$ is a transportation matrix that is returned by solving $\text{OT}(P_{X_i^m},P_{Y_i^m})$. Note that, $\pi^{\text{OT}}_{P_{X_i^m},P_{Y_i^m}}$ is expanded to a $n \times n$ matrix that has padded zero entries to indices which are different from those of $X_i^m$ and $Y_i^m$.
\end{definition}

We would like to recall that $k=1$ is the choice that practitioners usually used in real applications.

\textbf{Misspecified matchings issue of m-OT: } m-OT suffers from the problem which we refer to as \emph{misspecified mappings}. In particular, misspecified mappings are non-zero entries in $\pi_k^{\text{m-OT}}$ while they have values of zero in the optimal transport plan $\pi$ between original measures $\mu_n$ and $\nu_n$. We consider the following simple example:
\begin{example}
\label{example:illustration} Let $\mu_n, \nu_n$ be two empirical distributions with 5 supports on 2D: $\left\{(0,1),(0,2),(0,3),(0,4), (0,5)\right\}$ and $\{(1,1),(1,2),(1,3),(1,4),(1,5)\}$. The optimal mappings between $\mu_n$ and $\nu_n$, $\{(0,i)-(1,i)\}_{i=1}^5$ are shown in Figure~\ref{fig:mOT} . Assuming that we use mini-batches of size 3 for m-OT. We specifically consider a pair of mini-batches $X^m=\{(0,1),(0,2),(0,3)\}$ and $Y^m=\{(1,3),(1,4),(1,5)\}$. Solving OT between $X^m$ and $Y^m$ turns into 3 misspecified mappings $(0,1)-(1,3)$, $(0,2)-(1,4)$, and $(0,3)-(1,5)$ that have masses $1/3$ (see Figure~\ref{fig:mOT}).
\end{example}

\subsection{Mini-batch Unbalanced Optimal Transport}
Currently,~\cite{fatras2021minibatch} mitigated the misspecified matching issue by proposing to use unbalanced optimal transport as the transportation type between samples of mini-batches. The mini-batch unbalanced optimal transport is defined as follow:

\begin{definition}
\label{mUOT}
(Mini-batch Unbalanced Optimal Transport) For $1 \leq m\leq n$, $k \geq 1$, $\tau >0$, a given divergence $D_\phi$, $X^m_1,\ldots,X^m_k$ and $Y^m_1,\ldots,Y^m_k$ are sampled with or without replacement from $\mathcal{X}^m$ and $\mathcal{Y}^m$, respectively. The m-UOT transportation cost and transportation plan between $\mu_n$ and $\nu_n$ are defined as follow:
\begin{align}
    \text{m-UOT}_k^{\phi,\tau}(\mu_n,\nu_n) &= \frac{1}{k} \sum_{i=1}^k \text{UOT}_\phi^\tau(P_{X_i^m},P_{Y_i^m}); \nonumber \\
    \pi^{\text{m-UOT}_k^{\phi,\tau}} &= \frac{1}{k} \sum_{i=1}^k \pi^{\text{UOT}_\phi^\tau}_{P_{X_i^m},P_{Y_i^m}}, \label{eq: mini-batch_UOT}
\end{align}
where  $\pi^{\text{UOT}_\phi^\tau}_{P_{X_i^m},P_{Y_i^m}}$ is a transportation matrix that is returned by solving $\text{UOT}_\phi^\tau(P_{X_i^m},P_{Y_i^m})$. Note that $\pi^{\text{UOT}_\phi^\tau}_{P_{X_i^m},P_{Y_i^m}}$ is expanded to a $n \times n$ matrix that has padded zero entries to indices which are different from of $X_i^m$ and $Y_i^m$.
\end{definition}

\textbf{Example: } In Example~\ref{example:illustration}, UOT can reduce masses on misspecified matchings by relaxing the marginals of the transportation plan. Using $D_\phi$ as KL divergence, we show the illustration of m-UOT results in Figure~\ref{fig:mUOT} in Appendix~\ref{subsec:visual}. We would like to recall that the regularized coefficient $\tau$ controls the degree of the marginal relaxation in m-UOT.

\textbf{Discussion on m-UOT:} 
The m-UOT has some issues which originally come from the nature of UOT. First, the ``transport plan" for the UOT is hard to interpret since the UOT is developed for measures with different total masses. Second, the magnitude of regularized parameter $\tau$ depends on the cost matrix in order to make the regularization effective. Hence we need to search for $\tau$ in the wide range of $\mathbb{R}^+$, which is a problem when the cost matrix changes its magnitude. We illustrate a simulation to demonstrate that the transportation plan of UOT for a fixed parameter $\tau$ changes after scaling supports by a constant in Figure~\ref{fig:mUOTscale} in Appendix~\ref{subsec:visual}. In deep partial DA, m-UOT needs to regularize the scale of the feature space of the feature extractor. Also, m-UOT has not been applied to deep generative models and color transfer due to this challenge.

\section{Mini-batch Partial Optimal Transport}
\label{sec:mPOT}
In this section we propose a novel mini-batch approach, named \emph{mini-batch partial optimal transport} (m-POT), that uses \emph{partial optimal transport} (POT) as the transportation at the mini-batch level. We first review the definition of partial optimal transport in Section~\ref{subsec:POT}. Then, we define mini-batch partial optimal transport and discuss its properties in Section~\ref{subsec:mPOT}. Moreover, we illustrate that POT can be a natural choice of transportation among samples of mini-batches via simple simulations.
\subsection{Partial Optimal Transport}
\label{subsec:POT}
Now, we restate the definition of partial optimal transport (POT) that is defined in \cite{figalli2010optimal}. Similar to the definition of transportation plans, we define the notion of partial transportation plans. Let $0 < s \leq 1$ to be transportation fraction. Partial transportation plan between two discrete probability measures $\boldsymbol{\alpha}$ and $\boldsymbol{\beta}$ is  $\Pi_s(\boldsymbol{\alpha},\boldsymbol{\beta}):=\big\{ \pi \in \mathbb{R}_+^{|\boldsymbol{\alpha}|\times|\boldsymbol{\beta}|} : \pi 1_{|\boldsymbol{\beta}|} \leq \boldsymbol{\alpha}, \pi^\top 1_{|\boldsymbol{\alpha}|} \leq \boldsymbol{\beta}, 1^\top \pi 1 = s \big\}$. With previous notations, the partial optimal transport between $\mu_n$ and $\nu_n$ is defined as follow:
\begin{align}
\label{eq:POT}
    \text{POT}_s(\mu_n,\nu_n) = \min_{\pi \in \Pi_s(\boldsymbol{u}_n,\boldsymbol{u}_n)} \langle C, \pi \rangle,
\end{align}
where $C$ is the distance matrix.
Equation~(\ref{eq:POT}) can be solved by adding dummy points (according to \cite{chapel2020partial}) to expand the cost matrix $\overline{C}=\left[\begin{array}{cc}
C & 0 \\
0 & A
\end{array}\right]$, where $A>0$. In this case, solving the POT turns into solving the following OT problem: 
\begin{align}
    \min_{\pi \in \Pi(\bar{\boldsymbol{\alpha}},\bar{\boldsymbol{\alpha}})} \langle \bar{C},\pi \rangle, \label{eq:POT_equivalence}
\end{align} with $\bar{\boldsymbol{\alpha}}=[\boldsymbol{u}_n,1-s]$. Furthermore, the optimal partial transportation plan in equation~(\ref{eq:POT}) can be derived from removing the last row and column of the optimal transportation plan in equation~(\ref{eq:POT_equivalence}).

\subsection{Mini-batch Partial Optimal Transport}
\label{subsec:mPOT}
The partial transportation naturally fits the mini-batch setting since it can decrease the transportation masses of misspecified mappings (cf. the illustration in Figure~\ref{fig:mOT} when two  mini-batches contain optimal matching of the original transportation plan). Specifically, reducing the number of masses to be transported, i.e., reducing $s$ (from right images to left images in Figure~\ref{fig:mOT}), returns globally better mappings. With the right choice of the transport fraction, we can select mappings between samples that are as optimal as doing full optimal transportation. Moreover, POT is also stable to compute since it boils down to OT. Therefore, there are several solvers that can be utilized to compute POT. 



\begin{figure*}[t!]
    \begin{center}
        \begin{tabular}{c}
        \widgraph{0.85\textwidth}{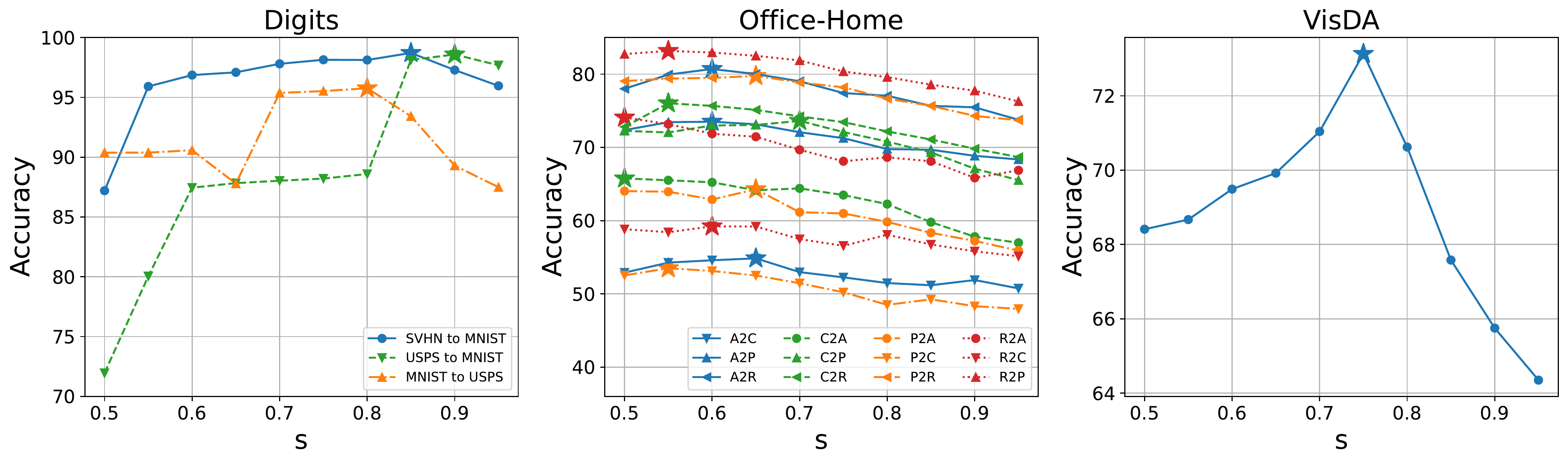} 
        \end{tabular}
    \end{center}
    \vskip -0.2in
    \caption{
        \footnotesize{Performance of m-POT on the deep DA when changing the fraction of masses $s$. The optimal values of $s$, which achieve the best accuracy, are marked by the $\star$ symbol. In the left figure, the optimal ratios for digits datasets lie between $0.8$ and $0.9$. In the middle figure, the best performing values are smaller, from $0.5$ to $0.7$, for the Office-Home dataset. On the VisDA dataset in the right figure, the optimal fraction of masses is $0.75$.}
    } 
    \label{fig:DA_s}
    \vskip -0.2in
\end{figure*}
 
Now, we define \emph{mini-batch partial optimal transport} (m-POT) between $\mu_n$ and $\nu_n$ as follow:
\begin{definition}
\label{mPOT}
(Mini-batch Partial Optimal Transport) For $1 \leq m\leq n$, $k \geq 1$, $0 < s \leq 1$, $X^m_1,\ldots,X^m_k$ and $Y^m_1,\ldots,Y^m_k$ are sampled with or without replacement from $\mathcal{X}^m$ and $\mathcal{Y}^m$, respectively. The m-POT transportation cost and transportation plan between $\mu_n$ and $\nu_n$ are defined as follow:
\begin{align}
    \text{m-POT}_k^{s}(\mu_n,\nu_n) &= \frac{1}{k} \sum_{i=1}^k POT_s(P_{X_i^m},P_{Y_{i}^m}); \nonumber \\
    \pi^{\text{m-POT}_k^s} &= \frac{1}{k} \sum_{i=1}^k \pi^{POT_s}_{P_{X_i^m},P_{Y_i^m}},
\end{align}
where $\pi^{POT_s}_{P_{X_i^m},P_{Y_i^m}}$ is a transportation matrix that is returned by solving $POT_s(P_{X_i^m},P_{Y_i^m})$. Note that $\pi^{POT_s}_{P_{X_i^m},P_{Y_i^m}}$ is expanded to a $n \times n$ matrix that has padded zero entries to indices which are different from of $X_i^m$ and $Y_i^m$.
\end{definition}

\paragraph{Computational complexity of m-POT:} 
From the equivalence form of POT in equation~(\ref{eq:POT_equivalence}), we have an equivalent form of m-POT in Definition~\ref{mPOT} as follows:
\begin{align}
    \text{m-POT}_k^{s}(\mu_n,\nu_n) = \frac{1}{k} \sum_{i = 1}^{k} \min_{\pi \in \Pi(\bar{\boldsymbol{\alpha}}_{i},\bar{\boldsymbol{\alpha}}_{i})} \langle \bar{C}_{i},\pi \rangle. \label{eq:equivalent_mPOT}
\end{align}
Here, $\overline{C}_{i}=\left[\begin{array}{cc}
C_{i} & 0 \\
0 & A_{i}
\end{array}\right] \in \mathbb{R}_{+}^{(m+1) \times (m+1)}$ and $\bar{\boldsymbol{\alpha}}_{i} = [\boldsymbol{u}_{m},1-s]$ where $C_{i}$ is a cost matrix formed by the differences of elements of $X_{i}^{m}$ and $Y_{i}^{m}$ and $A_{i} > 0$ for all $i \in [k]$. The computational complexity of approximating each OT problem in equation~(\ref{eq:equivalent_mPOT}) using entropic regularization is at the order of $\mathcal{O} \left(\frac{(m+1)^2}{\varepsilon^2} \right)$~\cite{altschuler2017near, lin2019efficient} where $\varepsilon > 0$ is the tolerance. Therefore, the total computational complexity of approximating mini-batch POT is at the order of $\mathcal{O} \left(\frac{k (m+1)^2}{\varepsilon^2} \right)$. It is comparable to the computational complexity of m-OT, which is of the order of $\mathcal{O} \left(\frac{k \cdot m^2}{\varepsilon^2} \right)$ and slightly larger than that of m-UOT, which is $\mathcal{O} \left(\frac{k \cdot m^2}{\varepsilon} \right)$~\cite{pham2020unbalanced}, in terms of $\varepsilon$.

\textbf{Concentration of m-POT:} We first provide a guarantee on the concentration of the m-POT's value for any given mini-batch size $m$ and given the number of mini-batches $k$.
\begin{theorem}
\label{theorem:concentration_bound}
For any given number of minibatches $k \geq 1$ and minibatch size $1 \leq m \leq n$, 
assume that the entries of $C_{ij}$ are obtained from the distance $d(X_i,Y_j)$. Furthermore, we assume that  $d(X_i,\mathbb{E}X_i)$ and $d(Y_j,\mathbb{E}Y_j)$ have sub-exponential distribution $SE(v^2,\gamma)$ (see Definition~\ref{def:supexponent} in Appendix~\ref{sec:concentration}). Then
\begin{align*}
     \mathbb{P} \Bigg{(}&\big|\text{m-POT}_k^{s}(\mu_n,\nu_n) - \text{m-POT}^{s}(\mu,\nu)\big| \geq \\ &D_n \sqrt{\frac{8\log(4/\delta)}{\lfloor n/m\rfloor}} + D_n \sqrt{\frac{2\log(4/\delta)}{k}} \Bigg{)} \leq \delta
\end{align*}
where $D_n = s\Big[d(\mathbb{E}X_i,\mathbb{E}Y_j)+ 2\max\Big\{\gamma \big[\log(2n) + \log(8/\delta)\big], \frac{v^2}{\gamma} \Big\} \Big]$ and \\
$\text{m-POT}^{s}(\mu,\nu) : = \mathbb{E}_{X \sim \mu^{\otimes m}, Y \sim \nu^{\otimes m}} \left[\text{POT}_{s}(P_{X^{m}}, P_{Y^{m}}) \right]$.
\end{theorem}

The proof of Theorem~\ref{theorem:concentration_bound} is in Appendix~\ref{sec:concentration_value}. Furthermore, in Appendix~\ref{sec:concentration_plan}, we study the concentration of the m-POT's transportation plan. We demonstrate that the row/ column sum of the m-POT's transportation plan concentrates around the row/ column sum of the full m-POT's transportation plan (cf. Definition~\ref{def:full_bactch_mPOT} in Appendix~\ref{def:full_bactch_mPOT}).

\noindent
\textbf{Practical consideration for m-POT: } First of all, as indicated in equation~(\ref{eq:equivalent_mPOT}), m-POT can be converted to m-OT with mini-batch size $m+1$. Therefore, it is slightly more expensive than m-OT and m-UOT in terms of memory and computation. The second issue of m-POT is the dependence on the choice of fraction of masses $s$ because $s$ plays a vital role in alleviating misspecified mappings from m-OT. At the first glance, choosing $s$ may seem as challenging as choosing $\tau$ in m-UOT; however, it appears that searching for $s$ is actually easier than $\tau$. 
For example, we show that the transportation plan of POT for a fixed parameter $s$ is the same while the transportation plan of UOT for a fixed parameter $\tau$ changes significantly when we scale the supports of two measures by a constant in Figure~\ref{fig:mPOTscale} in Appendix~\ref{subsec:visual}. Moreover, in partial deep domain adaptation, m-UOT needs to use an additional regularizer coefficient for controlling the scale of the feature space of the neural networks while m-POT does not need that parameter.

\begin{table}[!t]
    \caption{DA results in classification accuracy on digits datasets (higher is better). Experiments were run 3 times. Detailed results for each run are given in Table~\ref{table:DA_digits_details} in Appendix~\ref{subsec:addexp_DA}.}
    \vskip 0.1in
    \label{table:DA_digits_summary}
    \centering
    \scalebox{0.75}{
        \begin{tabular}{ccccc}
            \toprule
            Method & \scalebox{0.8}{SVHN to MNIST} & \scalebox{0.8}{USPS to MNIST} & \scalebox{0.8}{MNIST to USPS} & Avg \\
            \midrule
            DANN & 95.80 $\pm$ 0.29 & 94.71 $\pm$ 0.12 & 91.63 $\pm$ 0.53 & 94.05 \\
            ALDA & 98.81 $\pm$ 0.08 & 98.29 $\pm$ 0.07 & 95.29 $\pm$ 0.16 & 97.46 \\
            m-OT & 94.18 $\pm$ 0.32 & 96.71 $\pm$ 0.24 & 86.93 $\pm$ 1.16  & 92.60 \\
            m-UOT & 98.89 $\pm$ 0.13 & 98.54 $\pm$ 0.20 & 95.83 $\pm$ 0.05 & 97.75 \\
            m-POT (Ours) & \textbf{98.98 $\pm$ 0.08} & \textbf{98.63 $\pm$ 0.13} & \textbf{96.04 $\pm$ 0.02} & \textbf{97.88} \\
            \bottomrule
        \end{tabular}
    }
    \vskip -0.2in
\end{table}
 \begin{figure*}[!t]
\begin{center}

  \begin{tabular}{c}
\widgraph{0.8\textwidth}{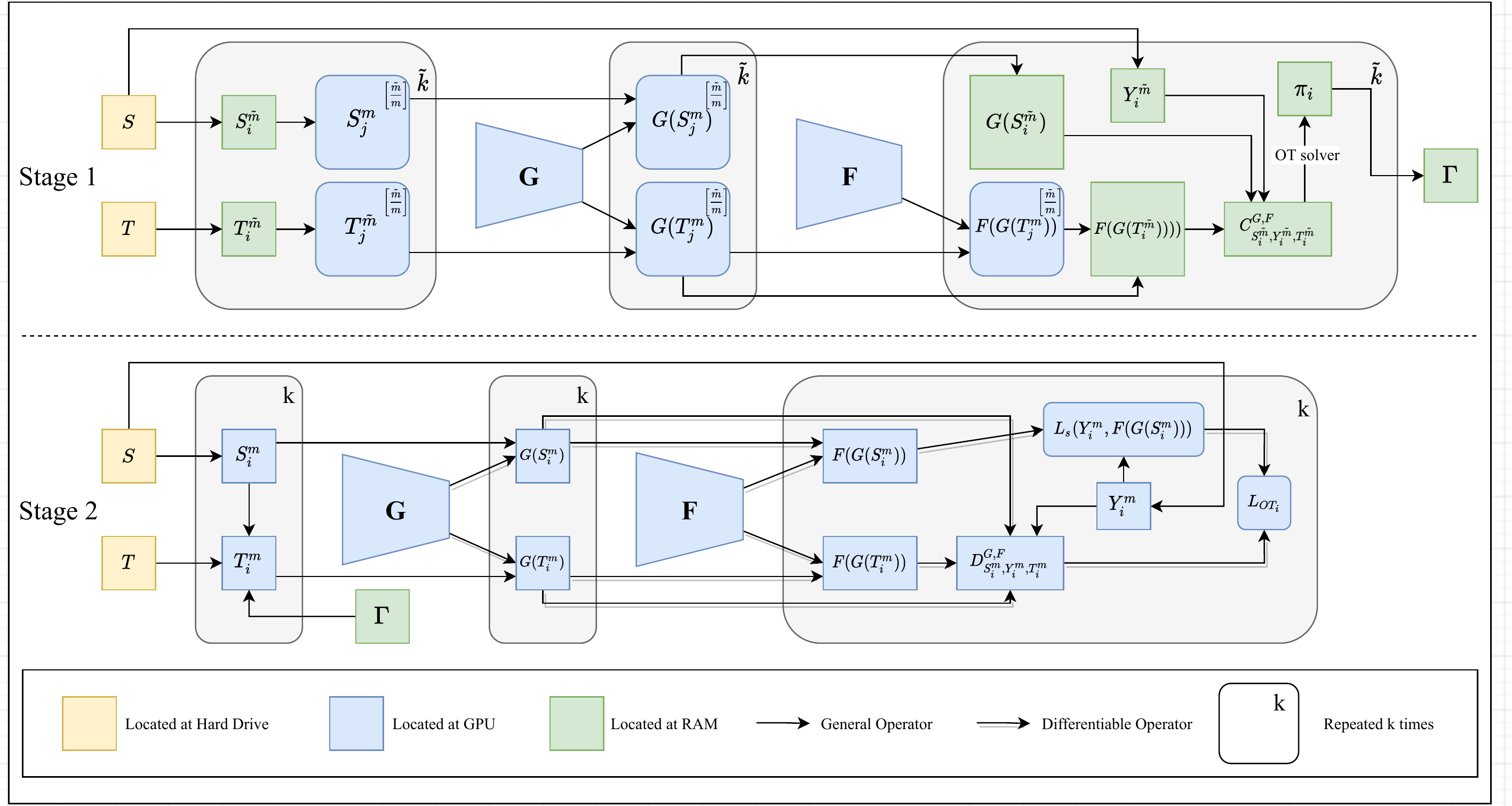} 
  \end{tabular}
  \end{center}
  \vskip -0.2in
  \caption{
  \footnotesize{The pseudo computational graph for the two-stage deep domain adaptation.
}
} 
  \label{fig:twostageDA}
  \vskip -0.2in
\end{figure*}
\begin{table}
    \caption{DA results in classification accuracy on the VisDA dataset (higher is better). Experiments were run 3 times. Detailed results for each run are given in Table~\ref{table:DA_visda_details} in Appendix~\ref{subsec:addexp_DA}.}
    \vskip 0.1in
    \label{table:DA_visda_summary}
    \centering
    \scalebox{0.8}{
        \begin{tabular}{cc}
            \toprule
            Method & Accuracy \\
            \midrule
            DANN & 67.63 $\pm$ 0.34 \\
            ALDA & 71.22 $\pm$ 0.12 \\
            m-OT & 62.42 $\pm$ 0.12 \\
            m-UOT & 72.34 $\pm$ 0.32 \\
            m-POT (Ours) & 73.59 $\pm$ 0.15 \\
            TS-OT (Ours) & 69.14 $\pm$ 0.72 \\
            TS-UOT (Ours) & 70.91 $\pm$ 0.11 \\
            TS-POT (Ours) & \textbf{75.96 $\pm$ 0.44} \\
            \bottomrule
        \end{tabular}
    }
    \vskip -0.2in
\end{table}

\section{Experiments}
\label{sec:experiment}

In this section, we focus on discussing experimental results on two applications, namely, deep domain adaption (deep DA) and partial domain adaptation (PDA). In particular, we compare the performance of m-POT with previous mini-batch optimal transport approaches, m-OT, and m-UOT, and some baseline methods on deep domain adaptation and partial domain adaptation. In the supplementary, we show that m-POT is better than m-OT on the deep generative model in Appendix~\ref{subsec:exp_dgm}. Furthermore, we also conduct simulations on mini-batch transportation matrices and experiments on color transfer application to show that m-POT can provide a good barycentric mapping in Appendix~\ref{subsec:color_transfer}.  Finally, a simple experiment on gradient flow is also carried out to illustrate the benefit of m-POT compared to m-OT and m-UOT in Appendix~\ref{subsec:gflow}. The details of applications and their algorithms are given in Appendix~\ref{sec:mini-batchapps}. In Appendix~\ref{sec:add_exp}, we report detailed results on all applications. The detailed experimental settings of all applications are deferred to Appendix~\ref{sec:setting}.\footnote{Python code is published at \url{https://github.com/UT-Austin-Data-Science-Group/Mini-batch-OT}.} 

\subsection{Deep Domain Adaptation}
\label{subsec:expDA}

In the following experiments, we conduct deep domain adaptation on three datasets including digits, Office-Home, and VisDA. To evaluate the performance, we use the classification accuracy in the adapted domain of the best performing checkpoint. Details about datasets, architectures of neural networks, and hyper-parameters settings are given in Appendix~\ref{subsec:setting_DA}. We compare our method against other DA methods: DANN~\cite{ganin2016domain}, CDAN-E~\cite{long2018conditional}, m-OT (DeepJDOT)~\cite{damodaran2018deepjdot}, ALDA~\cite{chen2020adversarial}, ROT~\cite{balaji2020robust}, and m-UOT (JUMBOT)~\cite{fatras2021unbalanced}. The same pre-processing techniques as in m-UOT are applied. We run each experiment 3 times and report the average accuracy in Tables~\ref{table:DA_digits_summary}-\ref{table:DA_office_summary} while the detailed results for each run are provided in Tables~\ref{table:DA_digits_details}-\ref{table:DA_visda_details} in Appendix~\ref{subsec:addexp_DA}.

\begin{table*}[t]
    \caption{DA results in classification accuracy on the Office-Home dataset (higher is better). Experiments were run 3 times. Detailed results for each run are given in Table~\ref{table:DA_office_details} in Appendix~\ref{subsec:addexp_DA}. (*) denotes the result taken from JUMBOT's  paper~\citep{fatras2021unbalanced}. }
    \vskip 0.1in
    \label{table:DA_office_summary}
    \centering
    \scalebox{0.75}{
        \begin{tabular}{cccccccccccccc}
        \toprule
        Method & A2C & A2P & A2R & C2A & C2P & C2R & P2A & P2C & P2R & R2A & R2C & R2P & Avg \\
        \midrule
        RESNET-50 (*) & 34.90 & 50.00 & 58.00 & 37.40 & 41.90 & 46.20 & 38.50 & 31.20 & 60.40 & 53.90 & 41.20 & 59.90 & 46.1 \\
        DANN & 47.92 & 67.08 & 74.85 & 53.80 & 63.47 & 66.42 & 52.99 & 44.35 & 74.43 & 65.53 & 52.96 & 79.41 & 61.93 \\
        CDAN-E (*) & 52.50 & 71.40 & 76.10 & 59.70 & 69.90 & 71.50 & 58.70 & 50.30 & 77.50 & 70.50 & 57.90 & 83.50 & 66.60 \\
        ALDA & 54.04 & 74.89 & 77.14 & 61.37 & 70.62 & 72.75 & 60.32 & 51.03 & 76.66 & 67.90 & 55.94 & 81.87 & 67.04 \\
        ROT (*) & 47.20 & 71.80 & 76.40 & 58.60 & 68.10 & 70.20 & 56.50 & 45.00 & 75.80 & 69.40 & 52.10 & 80.60 & 64.30 \\
        m-OT & 51.75 & 70.01 & 75.79 & 59.60 & 66.46 & 70.07 & 57.60 & 47.88 & 75.29 & 66.82 & 55.71 & 78.11 & 64.59 \\
        m-UOT & 54.99 & 74.45 & 80.78 & 65.66 & \textbf{74.93} & 74.91 & 64.70 & 53.42 & 80.01 & 74.58 & 59.88 & 83.73 & 70.17 \\
        m-POT (Ours) & 55.65 & 73.80 & 80.76 & 66.34 & 74.88 & 76.16 & 64.46 & 53.38 & \textbf{80.60} & 74.55 & 59.71 & 83.81 & 70.34 \\
        TS-OT (Ours) & 53.89 & 71.01 & 77.13 & 59.82 & 69.20 & 71.95 & 59.18 & 51.17 & 76.54 & 66.46 & 56.97 & 80.19 & 66.13 \\
        TS-UOT (Ours) & 56.35 & 73.56 & 80.16 & 65.02 & 73.12 & 76.50 & 63.66 & 54.49 & 79.97 & 71.24 & 60.11 & 82.92 & 69.76 \\
        TS-POT (Ours) & \textbf{57.06} & \textbf{76.13} & \textbf{81.53} & \textbf{68.44} & 72.82 & \textbf{76.53} & \textbf{66.21} & \textbf{54.87} & 80.39 & \textbf{75.57} & \textbf{60.50} & \textbf{84.31} & \textbf{71.20} \\
        \bottomrule
        \end{tabular}
    }
    \vskip -0.2in
\end{table*}

\begin{table*}[t]
    \caption{PDA results in classification accuracy on the Office-Home dataset (higher is better). Experiments were run 3 times. Detailed results for each run are given in Table~\ref{table:PDA_office_details} in Appendix~\ref{subsec:addexp_PDA}. (*) denotes the result taken from JUMBOT's paper~\citep{fatras2021unbalanced}.}
    \vskip 0.1in
    \label{table:PDA_office_summary}
    \centering
    \scalebox{0.79}{
        \begin{tabular}{cccccccccccccc}
        \toprule
        Method & A2C & A2P & A2R & C2A & C2P & C2R & P2A & P2C & P2R & R2A & R2C & R2P & Avg \\
        \midrule
        RESNET-50 (*) & 46.30 & 67.50 & 75.90 & 59.10 & 59.90 & 62.70 & 58.20 & 41.80 & 74.90 & 67.40 & 48.20 & 74.20 & 61.40 \\
        PADA (*) & 51.90 &  67.00 & 78.70 & 52.20 & 53.80 & 59.00 & 52.60 & 43.20 & 78.80 & 73.70 & 56.60 & 77.10 & 62.10 \\
        ETN (*) & 59.20 & 77.00 & 79.50 & 62.90 & 65.70 & 75.00 & 68.30 & 55.40 & 84.40 & 75.70 & 57.70 & 84.50 & 70.40 \\
        BA3US & 59.34 & 78.73 & \textbf{88.42} & 72.70 & 72.34 & 83.54 & 73.19 & 60.20 & 85.92 & 79.13 & 63.00 & 85.90 & 75.20 \\
        m-OT & 48.00 & 65.99 & 77.47 & 59.23 & 57.85 & 66.57 & 58.43 & 45.25 & 74.10 & 68.08 & 49.89 & 74.25 & 62.09 \\
        m-UOT & 61.53 & 80.34 & 85.33 & 75.60 & 72.89 & 79.79 & 74.56 & 61.95 & 86.49 & 80.78 & 67.38 & 84.89 & 75.96 \\
        m-POT (Ours) & \textbf{64.60} & \textbf{80.62} & 87.17 & \textbf{76.43} & \textbf{77.61} & \textbf{83.58} & \textbf{77.07} & \textbf{63.74} & \textbf{87.63} & \textbf{81.42} & \textbf{68.50} & \textbf{87.38} & \textbf{77.98} \\
        \bottomrule
        \end{tabular}
    }
    \vskip -0.2in
\end{table*}


\textbf{Classification results:} According to Tables~\ref{table:DA_digits_summary}-\ref{table:DA_office_summary}, m-POT is substantially better than m-OT on all datasets. In greater detail, m-POT results in significant increases of $5.28, 5.75,$ and $11.17$ on digits, Office-home, and VisDA datasets, respectively. This phenomenon is expected since m-POT can mitigate the issue of misspecified matchings while m-OT cannot. 
Compared to m-UOT, m-POT yields slightly higher performance, namely, m-POT gives 0.13 higher accuracy (average) on digits datasets and 0.17 higher accuracy (average) on the Office-Home dataset. For the VisDA dataset, m-POT beats m-UOT by a safe margin of 1.25.   

\textbf{Performance of baselines: } We reproduce some baseline methods and the rest are taken from ~\citep[Section~5.2]{fatras2021unbalanced}. Although our results do not match their numbers, we still manage to have some better results than those in the original papers. For digits datasets, both m-OT and m-UOT achieve higher accuracy on the adaptation from USPS to MNIST. On the Office-Home dataset, m-OT has higher accuracies in 12 out of 12 scenarios while ALDA and m-UOT lead to better performance in 8 and 7 scenarios.

\textbf{TSNE embedding: } In Figure~\ref{fig:tsne_digits} in Appendix~\ref{subsec:addexp_DA}, we visualize the TSNE embedding on the latent space of three mini-batch methods on the SVHN to MNIST adaptation task. It can be seen clearly from Figure~\ref{fig:tsne_digits} that many classes overlap each other in the representation of m-OT. For m-UOT, it also suffers from overlapping clusters though its problem is less severe than that of m-OT. The visualization of m-POT shows that it is more discriminative than the other two methods, leading to the highest accuracy of approximately $99\%$ on the task from SVHN to MNIST.

\textbf{The role of $s$: } We demonstrate the effect of changing the value $s$ on the classification accuracy of m-POT. In this experiment, the number of mini-batches $k$ is set to 1 for all datasets. We observe a general pattern in all DA experiments as follows. When $s$ comes closer to the optimal mass, the accuracy increases then drops as s moves away. As discussed in Section~\ref{sec:mPOT}, the right value of $s$ is important in making m-POT perform well. In Figure~\ref{fig:DA_s}, we show the accuracy of different values of $s$ on all datasets. From that figure, the best value of $s$ is the value that is not too big and not too small. The reason is that a large value of $s$ creates more misspecified matchings while a small value of $s$ might drop off too many mappings. The latter leads to a lazy gradient, thus slowing the learning process.


\textbf{The two-stage implementation: } The conventional implementation utilizes mini-batch methods on the GPU level, namely, optimal mappings are solved on GPU's memory. However, GPU's memory is also used to store deep neural nets and computational graphs which are memory-consuming. Therefore, the scale of the OT problem is limited. By the observation that the CPU's memory (RAM) is normally much bigger than GPU's memory, utilizing mini-batch methods on the CPU level can have a much bigger size of mini-batches that improves the quality of the sample alignment. After having the sample mapping, smaller mini-batches can be used for a fast differentiable loss computation on GPU. The pseudo computational graph is given in Figure~\ref{fig:twostageDA} and the detailed description and algorithms are given in Appendix~\ref{subsec: mini-batchDA}. The term TS-OT, TS-UOT, and TS-POT is utilized to refer to the two-stage version of m-OT, m-UOT, and m-POT, respectively. To show the effectiveness of the two-stage implementation, we vary the number of mini-batches $k \in \{1, 2, 4\}$ and compare its performance with the conventional implementation. Because m-POT already has a high accuracy ($\sim 98\%$) and a large mini-batch size ($m = 500$) on digits datasets, we only compare m-POT and TS-POT on Office-Home and VisDA datasets. Our focus here is not to obtain state-of-the-art performance, but rather to compare between the new and old implementation on the deep DA. Therefore, we do not fine-tune the fraction of masses for TS-POT and simply adapt the same value of $s$ from m-POT. Similarly, the same set of hyperparameters of m-UOT is applied to TS-UOT. Detailed experiments and comparisons can be found in Tables~\ref{table:DA_efficient_office_detail} and~\ref{table:DA_efficient_visda_detail} in Appendix~\ref{subsec:addexp_DA}. 
According to Tables~\ref{table:DA_visda_summary} and~\ref{table:DA_office_summary}, the two-stage implementation also consistently yields better performance on both Office-Home and VisDA datasets than the conventional one when the transportation type is either OT or POT. On the VisDA dataset, for example, the two-stage implementation shows an increase of 6.72 and 2.67 in classification accuracy of m-OT and m-POT, respectively. Furthermore, when combined with m-POT, TS-POT achieves the best accuracy on both datasets, resulting in an increase of at least 0.86 and 2.67 over competitors on Office-Home and VisDA datasets, respectively. On the other hand, the two-stage implementation does not work well with UOT, i.e., the performance of TS-UOT does not surpass m-UOT because its transportation plan is not as sparse as that of OT or POT.

\subsection{Partial Domain Adaptation}
\label{subsec:expPDA}
In this section, we compare three mini-batch methods on the partial domain adaptation (PDA) application which is a challenging variant of the deep DA because the source and target domains do not share the same label space. Following the setting in BA3US~\citep{liang2020balanced}, we select the first 25 categories (in alphabetic order) of the Office-Home dataset as the partial target domain. The neural network architecture and training procedure are similar to the deep domain adaptation on the Office-Home dataset. In this experiment, we only use one mini-batch and the entropic regularization coefficient of m-POT is set to a much larger value than that on the deep DA. Therefore, the two-stage implementation of the PDA will be left for future works. Further settings can be found in Appendix~\ref{subsec:setting_PDA}. For baselines, we compare against other PDA methods: PADA~\cite{cao2018partial}, ETN~\cite{cao2019learning}, BA3US~\cite{liang2020balanced}, m-OT (DeepJDOT)~\cite{damodaran2018deepjdot}, and m-UOT (JUMBOT)~\cite{fatras2021unbalanced}. It can be seen clearly from Table~\ref{table:PDA_office_summary} that m-POT yields the highest performance on 11 out of 12 tasks, leading to an average improvement of 2.02 over competitors. In addition, m-POT outperforms m-OT with a huge margin of 15.89. This again strengthens the advantage of our proposed mini-batch scheme over the conventional m-OT. It is worth mentioning that on the PDA application, m-UOT introduces an additional hyperparameter to control the scale of the cost matrix while this hyperparameter can be removed for m-OT and m-POT.

\section{Discussion}
\label{sec:Discussion}

In this paper, we have introduced a novel mini-batch approach that is referred to as mini-batch partial optimal transport (m-POT). The new mini-batch approach is motivated by the issue of misspecified mappings in the conventional mini-batch optimal transport approach (m-OT). Via extensive experiment studies, we demonstrate that m-POT can perform better than current mini-batch methods including m-OT and m-UOT in domain adaptation applications. Furthermore, we propose the two-stage training approach for the deep DA that outperforms the conventional implementation. In other applications, including deep generative model, color transfer, and gradient flow, m-POT also show consistently favorable performance compared to m-OT and m-UOT. There are a few natural future directions arising from our work: (i) first, we will develop efficient algorithms to choose the fraction of masses $s$ of m-POT adaptively; (ii) second, we would like to explore further the dependence structure between mini-batches \cite{nguyen2021transportation}.
\clearpage


\bibliography{example_paper}
\bibliographystyle{icml2022}
\clearpage
\appendix
\onecolumn
\begin{center}
{\bf{\LARGE{Supplement to ``Improving Mini-batch Optimal Transport via Partial Transportation"}}}
\end{center}

In this supplement, we first derive concentration bounds for m-POT's value and m-POT's transportation plan in Appendix~\ref{sec:concentration}. Next, we discuss applications of OT that have been used in mini-batch fashion and their corresponding algorithms in Appendix~\ref{sec:mini-batchapps}. In Appendix~\ref{sec:add_exp}, we provide additional experimental results. In particular, we demonstrate the illustration of transportation plans from m-OT, m-UOT, and m-POT in Appendix~\ref{subsec:visual} and Appendix~\ref{subsec:matrix}. Moreover, we present detailed results on deep domain adaptation in Appendix~\ref{subsec:addexp_DA}, partial domain adaptation in Appendix~\ref{subsec:addexp_PDA}, deep generative model in Appendix~\ref{subsec:exp_dgm}. Furthermore, we conduct experiments on color transfer application and  gradient flow application to show the favorable performance of m-POT in Appendix~\ref{subsec:color_transfer}
and Appendix~\ref{subsec:gflow} respectively. The comparison between m-POT and m-OT (m-UOT) with one additional sample in a mini-batch is given in Appendix~\ref{subsec:equi_OT}. The computational time of mini-batch methods in applications is reported in Appendix~\ref{subsec:time}. Finally, we report the experimental settings including neural network architectures, hyper-parameter choices in Appendix~\ref{sec:setting}.

\section{Concentration of the m-POT}
\label{sec:concentration}
In this appendix, we provide concentration bounds for the m-POT's value and m-POT's transportation plan. To ease the presentation, we first define the \emph{full mini-batch} version of the m-POT's value and m-POT's transportation plan, namely, when we draw all the possible mini-batches from the set of all $m$ elements of the data. We denote by $\binom{X^n}{m}$ and $\binom{Y^n}{m}$ the set of all $m$ elements of $\{X_{1}, X_{2}, \ldots, X_{n}\}$ and $\{Y_{1}, Y_{2}, \ldots, Y_{n}\}$ respectively. We define $\binom{X^n}{m} = \{X_{(1)}^{m}, X_{(2)}^{m}, \ldots, X_{(S)}^{m}\}$ and $\binom{Y^n}{m} = \{Y_{(1)}^{m}, Y_{(2)}^{m}, \ldots, Y_{(S)}^{m}\}$ where $S : = \binom{n}{m}$.
\begin{definition}
\label{def:full_bactch_mPOT}
(Full Mini-batch Partial Optimal Transport) For $1 \leq m\leq n$, $0 < s \leq 1$, the full mini-batch POT transportation cost and transportation plan between $\mu_n$ and $\nu_n$ are defined as follow:
\begin{align}
    \text{m-POT}^{s}(\mu_n,\nu_n) &= \frac{1}{S^2} \sum_{i = 1}^{S} \sum_{j=1}^{S} POT_s(P_{X_{(i)}^m},P_{Y_{(j)}^m}); \nonumber \\
    \pi^{\text{m-POT}^s} &= \frac{1}{S^2} \sum_{i=1}^{S} \sum_{j = 1}^{S} \pi^{POT_s}_{P_{X_{(i)}^m},P_{Y_{(j)}^m}},
\end{align}
where $\pi^{POT_s}_{P_{X_{(i)}^m},P_{Y_{(j)}^m}}$ is a transportation matrix that is returned by solving $POT_s(P_{X_{(i)}^m},P_{Y_{(j)}^m})$. Note that $\pi^{POT_s}_{P_{X_{(i)}^m},P_{Y_{(j)}^m}}$ is expanded to a $n \times n$ matrix that has padded zero entries to indices which are different from of $X_{(i)}^m$ and $Y_{(j)}^m$.
\end{definition}
For the purpose of our theory, we also have the following definition for sub-exponential random variable.
\begin{definition}(Sub-exponential Random Variable)
\label{def:supexponent}
A random variable $X$ with mean $\mu = \mathbb{E}(X)$ is sub-exponential $\text{SE}(v^2,\gamma)$ where $v, \gamma$ are non-negative parameters if the following holds:
\begin{align*}
    \mathbb{E}[\exp(\lambda(X - \mu)] \leq \exp(v^2 \lambda^2/ 2)
\end{align*}
for all $|\lambda| < 1/ \alpha$.
\end{definition}
A simple example of sub-exponential random variable is chi-square random variable, which is $\text{SE}(v^2,\gamma)$ with $v = 2$ and $\gamma = 4$.
\subsection{Concentration of the m-POT's value}
\label{sec:concentration_value}
We provide the proof for Theorem~\ref{theorem:concentration_bound} using some results from~\cite{de2012decoupling}. We define the population mini-batch partial optimal transport:
\begin{definition}
Assume that $\mu$ and $\nu$ are two probability measures on $\mathcal{P}_p (\mathcal{X})$ for given positive integer $m\geq 1$ Then, the population mini-batch partial OT (m-POT) discrepancy between $\mu$ and $\nu$ is defined as follows:
\label{def:popmPOT}
\begin{align}
    \text{m-POT}^{s}(\mu,\nu) : = \mathbb{E}_{X \sim \mu^{\otimes m}, Y \sim \nu^{\otimes m}} \left[\text{POT}_{s}(P_{X^{m}}, P_{Y^{m}}) \right].
\end{align}
\end{definition}


Here we assume that matrix $C$ is a metric that means $C_{ij} = d(X_i,Y_j)$, and the distance $d(X_i,\mathbb{E}X_i)$ and $d(Y_j,\mathbb{E}Y_j)$ have sub-exponential distributions $SE(v^2,\gamma)$. 

\begin{proof}[Proof of Theorem~\ref{theorem:concentration_bound}]
An application of triangle inequality leads to
\begin{align}
    \big|\text{m-POT}_{k}^{s}(\mu_{n},\nu_{n}) - \text{m-POT}^{s}(\mu,\nu)\big| &\leq \nonumber  \big|\text{m-POT}_{k}^{s}(\mu_{n},\nu_{n}) - \text{m-POT}^{s}(\mu_{n},\nu_{n})\big|  \nonumber + \big|\text{m-POT}^{s}(\mu_{n},\nu_{n}) - \text{m-POT}^{s}(\mu,\nu)\big| \nonumber \\
    &= T_{1} + T_{2}. \label{eq:key_inequality_concentration_bound}
\end{align}
To deal with $T_1$ and $T_2$, we first find an upper bound for $\|C\|_{\max}$ and $\text{POT}_s(X_i^m,Y_j^m)$. Let's denote $\mathbb{E}X_i =\mu_X$ and $\mathbb{E}Y_j = \mu_Y$. By the triangle's inequality, we have
\begin{align*}
   d(X_i,Y_j) \leq d(X_i, \mu_X) + d(\mu_X, \mu_Y) + d(Y_j,\mu_Y).
\end{align*}
It means that
\begin{align*}
    \|C\|_{\max} \leq \max_{1\leq i\leq n} d(X_i,\mu_X) + d(\mu_X, \mu_Y) + \max_{1\leq j\leq n} d(Y_j,\mu_Y).
\end{align*}
The second term is a constant. For the first and third term, for $t > \frac{v^2}{\gamma}$, we have
\begin{align*}
    \mathbb{P}\big(\max_{1\leq i \leq n} d(X_i,\mu_X) \geq t \big) \leq \sum_{i=1}^n \mathbb{P}\big(d(X_i,\mu_X)\geq t \big) \leq n \mathbb{P}\big( d(X_i,\mu_X) \geq t \big)\leq 2n \exp\big(\frac{-t}{2\gamma}\big).
\end{align*}
Let's choose $t_2^* = \max\Big\{ 2\gamma \big[\log(2n) + \log(\frac{8}{\delta})\big], \frac{v^2}{\gamma} \Big\}$, then
\begin{align*}
    \mathbb{P}\big(\max_{1\leq i \leq n} d(X_i,\mu_X) \geq t_2^* \big) \leq \frac{\delta}{8}\\
    \mathbb{P}\big(\max_{1\leq i \leq n} d(Y_j,\mu_Y) \geq t_2^* \big) \leq \frac{\delta}{8}.
\end{align*}
It deduce that  
\begin{align}\label{ineq:up-bound-C}
 \|C\|_{\max} \leq \max_{1\leq i\leq n} d(X_i,\mu_X) + d(\mu_X, \mu_Y) + \max_{1\leq j\leq n} d(Y_j,\mu_Y) \leq d(\mu_X,\mu_Y) + 2t_2^*  
\end{align}
with probability at least $1- \frac{\delta}{4}$. Recall that 
\begin{align*}
    \text{POT}_s(P_{X^m_{(i)}}, P_{Y^m_{(j)}}) =\min_{\pi \in \Pi(\overline{\alpha},\overline{\alpha})} \big\langle \overline{C},\pi\big\rangle
\end{align*}
where $\overline{\alpha} = [\mathbf{u}_m,1-s]$ and $\overline{C} = \begin{pmatrix} C & \mathbf{0}_m \\ \mathbf{0}_m^{\top} & A \end{pmatrix}$. Define
\begin{align*}
    \overline{\pi} = \begin{pmatrix}s \mathbf{u}_{m\times m} & (1-s)\mathbf{u}_m \\
    (1-s)\mathbf{u}_m^{\top} & 0
    \end{pmatrix}
\end{align*}
where $\mathbf{u}_m = \frac{1}{m}\mathbf{1}_m$ and 
$\mathbf{u}_{m\times m} = \frac{1}{m^2} \mathbf{1}_{m\times m}$. Then 
$\overline{\pi}$ is the transport plan between  $\overline{\alpha}$ and  $\overline{\alpha}$. Denote $X^m_{(i)} = \{X_{i_1},\ldots,X_{i_m} \}$ and $Y^m_{(j)} = \{Y_{j_1},\ldots,Y_{j_m} \}$. It means that
\begin{align*}
    \text{POT}_s(P_{X^m_{(i)}},P_{Y^m_{(j)}}) &\leq \big\langle \overline{C},\overline{\pi}\big\rangle = \frac{1}{m^2}\sum_{\ell_1,\ell_2=1}^m s d(X_{i_{\ell_1}},Y_{j_{\ell_2}}) \\
    &\leq \frac{s}{m^2} \sum_{\ell_1,\ell_2=1}^m\Big\{d(X_{i_{\ell_1}},\mu_X) + d(\mu_X,\mu_Y) + d(\mu_Y,Y_{j_{\ell_2}}) \Big\}\\
    &=sd(\mu_X,\mu_Y) + \frac{s}{m} \sum_{\ell_1=1}^m d(X_{i_{\ell_1}},\mu_X) + \frac{s}{m}\sum_{\ell_2=1}^m d(Y_{j_{\ell_2}},\mu_Y) \\
    &\leq s \Big\{d(\mu_X,\mu_Y) + \max_{1\leq \ell_1\leq n} d(X_{i_{\ell_1}},\mu_X) + \max_{1\leq \ell_2 \leq n}d(Y_{j_{\ell_2}},\mu_Y) \Big\}.
\end{align*}
Combine it with \eqref{ineq:up-bound-C}, we obtain
\begin{align*}
    \max_{1\leq i,j\leq n}\text{POT}_s(P_{X^m_{(i)}},P_{Y^m_{(j)}}) \leq  s\big(d(\mu_X,\mu_Y) + 2t_2^*\big):=D_n
\end{align*}
with probability at least $1 - \frac{\delta}{4}$. Define
\begin{align*}
    E_n = \Big\{\max_{1\leq i,j\leq n}\text{POT}_s(P_{X^m_{(i)}},P_{Y^m_{(j)}}) \leq  s\big(d(\mu_X,\mu_Y) + 2t_2^*\big) \Big\}.
\end{align*}
Then $\mathbb{P}(E_n) \geq 1 - \frac{\delta}{4}$.

\textbf{Bound for $T_{2}$:} We consider the product space of $(X,Y)$. For $n$ samples $Z_i=(X_i, Y_i)$ for $1\leq i\leq n$. The U-statistics for $\text{POT}_s$ is defined as
\begin{align*}
   U_n(\text{POT}) = \frac{(n-m)!}{n!} \sum_{(i_1,\ldots,i_m)\in {\binom{n}{m}}}POT_s(Z_{i_1},\ldots,Z_{i_m} ).
\end{align*}
A reference for Hoeffding's inequality for U-statistics could be found in page 165 of ~\cite{de2012decoupling} where  $U_n(h)$ is defined in page 97 of ~\cite{de2012decoupling}.
We apply the Hoeffding's inequality for U-statistics for $POT_s$ conditioned on the event $E_n$ to obtain
\begin{align*}
    \mathbb{P}\left(T_2 \geq t |E_n\right) \leq \exp\left\{ - \frac{\lfloor n/m \rfloor t^2}{8D_n^2} \right\}.
\end{align*}
It follows that
\begin{align*}
    \mathbb{P}\left( T_2 \geq D_n \sqrt{\frac{8\log(4/\delta)}{\lfloor n/m \rfloor }}  \bigg| E_n \right) \leq \frac{\delta}{4}.
\end{align*}
Denote $\Big\{T_2 \leq D_n \sqrt{\frac{8\log(4/\delta)}{\lfloor n/m \rfloor }} \Big\} = F_{n,2}$.
We deduce that
\begin{align*}
    \mathbb{P}(F_{n,2}) &\geq \mathbb{P}(F_{n,2}\cap E_n) = \mathbb{P}(F_{n,2}|E_n) \mathbb{P}(E_n) \geq \Big(1- \frac{\delta}{4} \Big) \Big(1 -\frac{\delta}{4} \Big)\geq 1 - \frac{2\delta}{4}.
\end{align*}

\textbf{Bound for $T_{1}$:} Direct calculation shows that
\begin{align*}
    T_{1} = \left|\frac{1}{k} \sum_{i = 1}^{k} \sum_{1 \leq j,  l \leq S} \left(1_{\{(X_{i}^{m}, Y_{i}^{m}) \equiv (X_{(j)}^{m}, Y_{(l)}^{m})\}} - \frac{1}{S^2} \right) \text{POT}(X_{(j)}^{m}, Y_{(l)}^{m})\right|
\end{align*}
We denote $Z_{i} = \sum_{1 \leq j,  l \leq S} \left(1_{\{(X_{i}^{m}, Y_{i}^{m}) \equiv (X_{(j)}^{m}, Y_{(l)}^{m})\}} - \frac{1}{S^2} \right) \text{POT}(X_{(j)}^{m}, Y_{(l)}^{m})$ for $1 \leq i \leq k$. Then, given the data $X_{1}, \ldots, X_{n}$ and $Y_{1}, \ldots, Y_{n}$, it is clear that $Z_{1}, Z_{2}, \ldots, Z_{k}$ are independent variables bounded by $D_n$ with probability at least $1- \frac{\delta}{4}$. Apply Hoeffding's inequality for $Z_1,\ldots,Z_k$, conditioned on the event $E_n$,  we get
\begin{align*}
    \mathbb{P}\left( T_1 \geq t\big|E_n\right) \leq \exp\left\{-\frac{kt^2}{2D_n^2} \right\}.
\end{align*}
It means that 
\begin{align*}
    \mathbb{P}\left(T_1 \geq D_n\sqrt{\frac{2\log(4/\delta)}{k}} \bigg|E_n \right) \leq \frac{\delta}{4}.
\end{align*}
Denote $F_{n,1} = \Big\{T_1 \leq D_n\sqrt{\frac{2\log(4/\delta)}{k}}  \Big\} $, then $\mathbb{P}(F_{n,2})\geq 1 - \frac{\delta}{4}$. Using similar argument as the part of $F_{n,2}$, we get
\begin{align*}
    \mathbb{P}(F_{n,2})\geq 1 - \frac{2\delta}{4}.
\end{align*}
Together with the result $\mathbb{P}(F_{n,1})\geq 1 - \frac{2\delta}{4}$, we obtain
\begin{align*}
    \mathbb{P}(F_{n,1}\cap F_{n,2}) \geq 1 - \delta.
\end{align*}
We obtain the conclusion of the theorem. 

\end{proof}
\subsection{Concentration of the m-POT's transportation plan}
\label{sec:concentration_plan}
In this appendix, we study the concentration of the m-POT's transportation plan. In the following theorem, we demonstrate that the row/ column sum of the m-POT's transportation plan concentrates around the row/ column sum of the full m-POT's transportation plan in Definition~\ref{def:full_bactch_mPOT}. 
\begin{theorem}
\label{theorem:concentration_plan}
For any given number of batches $k \geq 1$ and minibatch size $1 \leq m \leq n$, there exists universal constant $C$ such that with probability $1 - \delta$ 
we have

\begin{align}
    \mathbb{P} \Bigg{(} &\left|\pi^{\text{m-POT}_{k}^s}\mathbf{1}_{n} - \pi^{\text{m-POT}^s}\mathbf{1}_{n}\right| \geq C \sqrt{\frac{2\log(2/\delta)}{k}} \Bigg{)}  \leq \delta, \label{eq:concentration_plan_first} \\
    \mathbb{P} \Bigg{(} &\left|\left(\pi^{\text{m-POT}_{k}^s}\right)^{\top}\mathbf{1}_{n} - \left(\pi^{\text{m-POT}^s}\right)^{\top} \mathbf{1}_{n}\right| \geq C \sqrt{\frac{2\log(2/\delta)}{k}} \Bigg{)}  \leq \delta, \label{eq:concentration_plan_second}
\end{align}

where $\mathbf{1}_{n}$ is the vector with all 1 values of its elements.
\end{theorem}
The proof of Theorem~\ref{theorem:concentration_plan} is in Appendix~\ref{sec:concentration_plan}. The results of Theorem~\ref{theorem:concentration_plan} indicate that m-POT has good concentration behaviors around its expectation and its full mini-batch version (see Definition~\ref{def:full_bactch_mPOT} in Appendix~\ref{sec:concentration}).

\begin{proof}[Proof of Theorem~\ref{theorem:concentration_plan}] Now, we provide proof for Theorem~\ref{theorem:concentration_plan}. It is sufficient to prove equation~(\ref{eq:concentration_plan_first}). Direct calculation shows that
\begin{align*}
    \pi^{\text{m-POT}_{k}^s}\mathbf{1}_{n} - \pi^{\text{m-POT}^s}\mathbf{1}_{n} = \frac{1}{k} \sum_{i = 1}^{k} \sum_{1 \leq j,  l \leq S} \left(1_{\{(X_{i}^{m}, Y_{i}^{m}) \equiv (X_{(j)}^{m}, Y_{(l)}^{m})\}} - \frac{1}{S^2} \right) \pi^{POT_s}_{P_{X_{(j)}^m},P_{Y_{(l)}^m}} \mathbf{1}_{n}.
\end{align*}
We define $Z_{i} = \sum_{1 \leq j,  l \leq S} \left(1_{\{(X_{i}^{m}, Y_{i}^{m}) \equiv (X_{(j)}^{m}, Y_{(l)}^{m})\}} - \frac{1}{S^2} \right) \pi^{POT_s}_{P_{X_{(j)}^m},P_{Y_{(l)}^m}} 1_{n}$ for $1 \leq i \leq k$. Conditioning on the data $X_{1}, X_{2}, \ldots, X_{n}$ and $Y_{1}, Y_{2}, \ldots, Y_{n}$, the random variables $Z_{1}, Z_{2}, \ldots, Z_{k}$ are independent and upper bounded by $2s$. Therefore, a direct application of Hoeffding's inequality shows that
\begin{align*}
    \mathbb{P} \left(\left|\pi^{\text{m-POT}_{k}^s}\mathbf{1}_{n} - \pi^{\text{m-POT}^s}\mathbf{1}_{n}\right| \geq 2s \sqrt{\frac{\log(2/\delta)}{k}} \right) \leq \delta.
\end{align*}
As a consequence, we obtain the conclusion of Theorem~\ref{theorem:concentration_plan}.
\end{proof}
 
\section{Applications to  Deep Unsupervised Domain Adaption, Deep Generative Model,  Color Transfer, and Gradient Flow}
\label{sec:mini-batchapps}
In this section, we first state two popular applications that benefit from using mini-batches, namely, deep domain adaption (including the two-stage implementation) and deep generative model in Appendix~\ref{subsec: mini-batchDA} and Appendix~\ref{subsec:mini-batchDGM}. We also include detailed algorithms for these two applications and the way that we evaluate them. Next, we review the mini-batch color transfer algorithm in Appendix~\ref{subsec:mini-batchcolortransfer}. Finally, we discuss the usage of mini-batch losses in gradient flow application in Appendix~\ref{subsec:mini-batchgradient}.

\subsection{Mini-batch Deep Domain Adaption}
\label{subsec: mini-batchDA}
We follow the setting of DeepJDOT~\cite{damodaran2018deepjdot} that is composed of two parts: an embedding function $G: \mathcal{X} \to \mathcal{Z}$ which maps data to the latent space; and a classifier $F: \mathcal{Z} \to \mathcal{Y}$  which maps the latent space to the label space in the target domain. The mini-batch version of DeepJDOT can be expressed as follow, for given the number of mini-batches $k$ and the size of mini-batches $m$, the goal is to minimize the following objective function:
\begin{align}
\label{eq:DAOT}
    \min_{G,F} \frac{1}{k}\sum_{i=1}^k L_{\text{OT}_i}; \hspace{1em} L_{\text{OT}_i} = \Bigg{(} \frac{1}{m} \sum_{j=1}^m L_s (y_{ij},F(G(s_{ij}))) + \min_{\pi \in \Pi(\boldsymbol{u}_m,\boldsymbol{u}_m)} \langle C_{S_i^m,Y_i^m,T_i^m}^{G,F},\pi \rangle \Bigg{)},
\end{align}
where $L_s$ is the source loss function, $S_1^m,\ldots,S_k^m$ are source mini-batches that are sampled with or without replacement from the source domain S, $Y_1^m,\ldots,Y_k^m$ are corresponding labels of $S_1^m,\ldots,S_k^m$, with $S_i^m = \{s_{i1},\ldots,s_{im}\}$ and $Y_i^m:=\{y_{i1},\ldots,y_{im}\}$. Similarly, $T_1^m,\ldots,T_k^m$ ($T_i^m:=\{t_{i1},\ldots,t_{im}\}$) are target mini-batches that are sampled with or without replacement from the target domain $\mathcal{T}$. The cost matrix $C_{S_i^m,Y_i^m,T_i^m}^{G,F}$ is defined as follows:
\begin{align}
\label{eq:cost_matrix}
    C_{1 \leq j,z \leq m} &= \alpha ||G(s_{ij}) - G(t_{iz})||^2 + \lambda_t L_t(y_{ij},F(G(t_{iz}))),
\end{align}
where $L_t$ is the target loss function, $\alpha$ and $\lambda_t$ are hyper-parameters that control two terms.

\textbf{JUMBOT:} By replacing OT with UOT in DeepJDOT, authors in \cite{fatras2021unbalanced} introduce JUMBOT (joint unbalanced mini-batch optimal transport) which improves domain adaptation results on various datasets. Therefore, the objective function in equation~(\ref{eq:DAOT}) turns into:
\begin{align}
    &\min_{G,F} \frac{1}{k}\sum_{i=1}^k L_{\text{UOT}_i};  \nonumber\\
    & L_{\text{UOT}_i}=\Bigg{(} \frac{1}{m} \sum_{j=1}^m L_s (y_{ij},F(G(s_{ij}))) + \min_{\pi \in \mathbb{R}_+^{m\times m}}[ \langle C_{S_i^m,Y_i^m,T_i^m}^{G,F},\pi \rangle  + \tau \left( \text{D}_\phi (\pi_{1},\boldsymbol{u}_m) + \text{D}_\phi (\pi_{2},\boldsymbol{u}_m) \right)] \Bigg{)}, 
\end{align}
where $\pi_{1}$ and $\pi_{2}$ are marginals of non-negative measure $\pi$.

\begin{figure*}[!t]
\begin{center}

  \begin{tabular}{c}
\widgraph{0.95\textwidth}{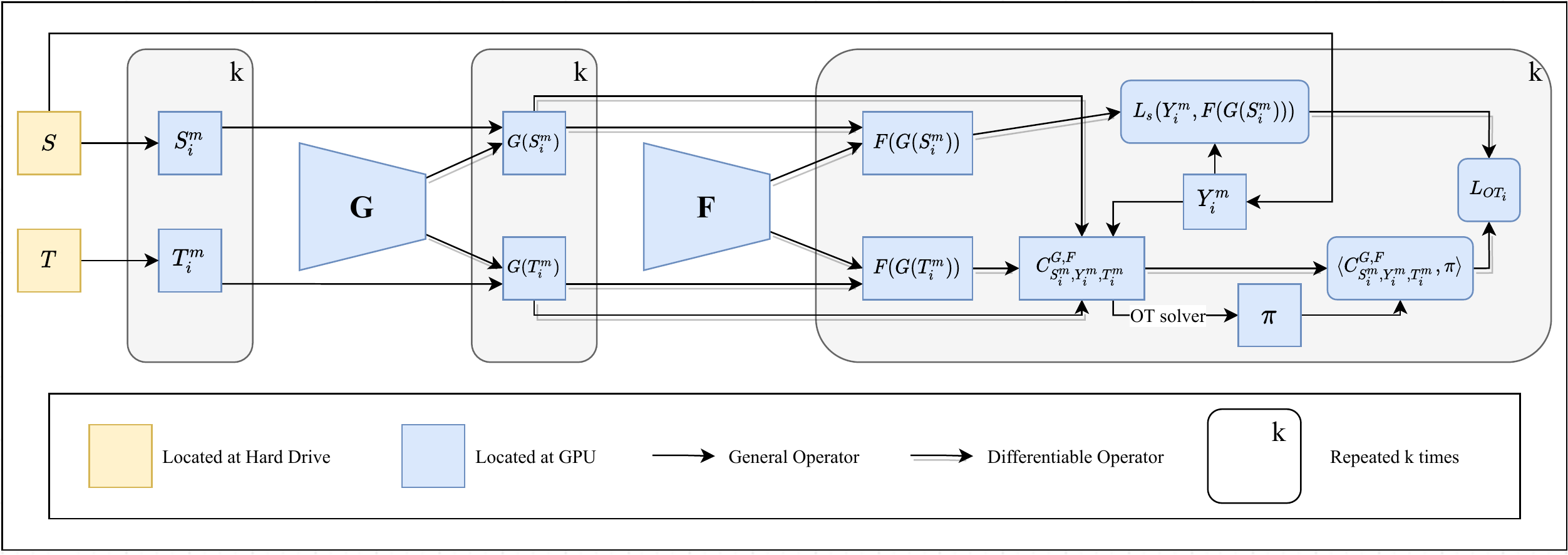} 
  \end{tabular}
  \end{center}
  \vskip -0.2in
  \caption{
  \footnotesize{The pseudo computational graph for the conventional deep domain adaptation.
}
} 
  \label{fig:conventionalDA}
  \vskip -0.1in
\end{figure*}

\textbf{Deep domain adaptation with m-POT:} Similar to JUMBOT, by changing OT into POT with the fraction of masses $s$, we obtain the following objective loss:
\begin{align}
    \min_{G,F} \frac{1}{k}\sum_{i=1}^k L_{\text{POT}_i}; \hspace{1em} L_{\text{POT}_i}=\Bigg{(} \frac{1}{m} \sum_{j=1}^m L_s (y_{ij},F(G(s_{ij}))) + \min_{\pi \in \Pi_s(\boldsymbol{u}_m,\boldsymbol{u}_m)} \langle C_{S_i^m,Y_i^m,T_i^m}^{G,F},\pi \rangle \Bigg{)},
\end{align}

\textbf{Training Algorithms:} We present the algorithm for training domain adaption with m-OT (mini-batch DeepJDOT), m-UOT (JUMBOT), and m-POT in a generalized algorithm which is stated in Algorithm~\ref{alg:DA}. We visualize the process in Figure~\ref{fig:conventionalDA}. In summary, the gradient of the neural net $G$ and the neural net $F$ are accumulated from $k$ pair of mini-batches.

We utilize Algorithm~\ref{alg:DA} to compare the performance of m-OT, m-UOT, and m-POT. The evaluation criterion is chosen based on the task on domains e.g. classification accuracy in the target domain (classification problems). 
\begin{algorithm}[!t]
   \caption{Mini-batch Deep Domain Adaptation}
   \label{alg:DA}
\begin{algorithmic}
   \STATE {\bfseries Input:} $k$, $m$, source domain $(S,Y)$, target domain $T$, chosen $L_{\text{DA}}\in \{L_{\text{OT}},L_{\text{UOT}},L_{\text{POT}}\}$
   \STATE Initialize $G_\theta$ (parametrized by $\theta$), $F_\phi$ (parametrized by $\phi$)
   \WHILE{$(\theta,\phi)$ do not converge}
   \STATE $\text{grad}_\theta  \leftarrow \boldsymbol{0}$; $\text{grad}_\phi \leftarrow \boldsymbol{0}$
   \FOR{$i=1$ {\bfseries to} $k$}
   \STATE Sample $(s_1,y_1),\ldots,(s_m,y_m)$ from $S$
   \STATE Sample $t_1,\ldots,t_m$ from $T$
   \STATE $S^m \leftarrow \{s_1,\ldots,s_m\}$; $Y^m \leftarrow \{y_1,\ldots,y_m\}$; $T^m \leftarrow \{t_1, \ldots, t_m \}$
   \STATE Compute $L_{\text{DA}} \leftarrow \frac{1}{k} L_{\text{DA}}(S^m,Y^m,T^m,G_\theta,F_\phi)$
   \STATE $\text{grad}_\theta \leftarrow \text{grad}_\theta + \nabla_\theta L_{\text{DA}}$
   \STATE $\text{grad}_\phi \leftarrow \text{grad}_\phi + \nabla_\phi L_{\text{DA}}$
   \ENDFOR
   \STATE $\theta \leftarrow \text{Adam}(\theta, \text{grad}_\theta)$
   \STATE $\phi \leftarrow \text{Adam}(\theta, \text{grad}_\phi)$
   \ENDWHILE
\end{algorithmic}
\end{algorithm}

\textbf{The two-stage implementation: } In the conventional implementation of DeepJDOT and its extension with UOT and POT, the transportation between each pair of mini-batches of size $m$ is solved on GPU in turn. However, the GPU memory is often small and it is often used for storing the big deep neural nets and the computational graph for automatic differentiation. Hence, the size of mini-batches $m$ is often small e.g., 100. We observe that it does not require doing back propagation through the transportation matrices and the memory of CPU (RAM) is much bigger than one of GPU. Hence, we propose to first estimate the transportation matrix on the CPU level then use it to align samples that will be computed their distance on GPU. We would like to recall that a bigger size of mini-batches leads to a better estimation of OT~\cite{sommerfeld2019optimal}.  We visualize the process in Figure~\ref{fig:twostageDA} and the corresponding algorithm in Algorithm~\ref{alg:efficient_DA}. Intuitively, the two-stage approach is using mini-batch methods (m-OT, m-UOT, m-POT) on the CPU level for better estimation of the transportation plan.

A related work was done in~\cite{lezama2021run}. In particular, the author proposed to first use the CPU to sort one-dimensional projected samples, then use GPU to compute distances of sorted supports for estimating sliced Wasserstein distance well. The above approach can be seen as a special case of our two-stage implementation in generative modeling using $\tilde{k}=1$ and a sliced OT distance (sliced Wasserstein).

Due to the nice property of OT and POT, the transportation plan has the form of a permutation matrix or at least nearly for POT. Utilizing the sparse property of OT and POT, we can divide a large transportation plan into multiple blocks of small transportation plans between mini-batches. Inspired by this, we propose a more efficient algorithm for training domain adaptation with m-OT and m-POT. Algorithm~\ref{alg:efficient_DA} details the training procedure which first solves a large-batch optimal transport problem, then computes an alignment between source and target samples. Finally, using the pre-computed alignment, we calculate the loss between the mini-batch source and its aligned mini-batch target before accumulating the gradient of the objective function.

\begin{algorithm}[!t]
  \caption{Two-stage mini-batch Deep Domain Adaptation}
  \label{alg:efficient_DA}
\begin{algorithmic}
  \STATE {\bfseries Input:} $\tilde{k}, \tilde{m}, m$, source domain $(S,Y)$, target domain $T$, chosen cost $L_{\text{DA}} \in \{L_{\text{OT}},L_{\text{POT}}\}$ 
  \STATE Initialize $G_\theta$ (parametrized by $\theta$), $F_\phi$ (parametrized by $\phi$)
  \WHILE{$(\theta,\phi)$ do not converge}
  \FOR{$i=1$ {\bfseries to} $\tilde{k}$}
  \STATE ----- \textit{On computer memory}-----
  \STATE Sample $(s_1,y_1),\ldots,(s_{\tilde{m}},y_{\tilde{m}})$ from $(S,Y)$
  \STATE Sample $t_1,\ldots,t_{\tilde{m}}$ from $T$
  \STATE $S^{\tilde{m}} \leftarrow \{s_1,\ldots,s_{\tilde{m}}\}$; $Y^{\tilde{m}} \leftarrow \{y_1,\ldots,y_{\tilde{m}}\}$; $T^{\tilde{m}} \leftarrow \{t_1, \ldots, t_{\tilde{m}} \}$
  \STATE Compute $C_{S^{\tilde{m}},Y^{\tilde{m}},T^{\tilde{m}}}^{G_\theta, F_\phi}$ (Equation~\ref{eq:cost_matrix})
  \STATE $\boldsymbol{u}_{\tilde{m}} \leftarrow \left(\frac{1}{\tilde{m}},\ldots,\frac{1}{\tilde{m}}\right)$
  \STATE Solve $\pi \leftarrow \text{(P-)OT}(\boldsymbol{u}_{\tilde{m}},\boldsymbol{u}_{\tilde{m}},C_{S^{\tilde{m}},Y^{\tilde{m}},T^{\tilde{m}}}^{G_\theta, F_\phi})$
  \STATE $\Gamma \leftarrow \argmax \pi$ 
  \STATE $k \leftarrow \floor{\frac{\tilde{m}}{m}}$
  \FOR{$i=1$ {\bfseries to} $k$}
  \STATE ----- \textit{On GPU}-----
  \STATE $S^m \leftarrow \{s_{(i-1)m+1},\ldots,s_{im}\}$; $Y^m \leftarrow \{y_{(i-1)m+1},\ldots,y_{im}\}$; $T^m \leftarrow \{t_{\Gamma_{(i-1)m+1}}, \ldots, t_{\Gamma_{im}}\}$
  \STATE Let $\pi^i \in \RR^{m \times m}$ be the transportation plan for $S^m$ and $T^m$ from the pre-computed transportation plan $\pi$.
  \STATE Compute $D_{S^m,Y^m,T^m}^{G_\theta,F_\phi} \leftarrow \langle C_{S^m,Y^m,T^m}^{G_\theta,F_\phi}, \pi^i \rangle $
  \STATE $L_{\text{DA}}^{\pi} \leftarrow \frac{1}{\tilde{k}} D_{S^m,Y^m,T^m}^{G_\theta,F_\phi}$
  \STATE $\text{grad}_\theta \leftarrow \text{grad}_\theta + \nabla_\theta L_{\text{DA}}^{\pi}$
  \STATE $\text{grad}_\phi \leftarrow \text{grad}_\phi + \nabla_\phi L_{\text{DA}}^{\pi}$
  \ENDFOR
  \ENDFOR
  \STATE ----- \textit{On GPU}-----
  \STATE $\theta \leftarrow \text{Adam}(\theta, \text{grad}_\theta)$
  \STATE $\phi \leftarrow \text{Adam}(\theta, \text{grad}_\phi)$
  \ENDWHILE
\end{algorithmic}
\end{algorithm}
\subsection{Mini-batch Deep Generative Model}
\label{subsec:mini-batchDGM}
We first recall the setting of deep generative models. Given the data distribution $\mu_n:=\frac{1}{n} \sum_{i=1}^n \delta_{x_i}$ with $x_i \in \mathcal{X}$, a prior distribution $p(z) \in \mathcal{P}(\mathcal{Z})$ e.g. $p(z)= \mathcal{N}(\boldsymbol{0},\boldsymbol{I})$, and a generator (generative function) $G_\theta: \mathcal{Z} \to \mathcal{X}$ (where $\theta \in \Theta$ is a neural net). The goal of the deep generative model is to find the parameter $\theta^*$ that minimizes the discrepancy (e.g. KL divergence, optimal transport distance, etc) between $\mu_n$ and $G_\theta \sharp p(z)$ where the $\sharp$ symbol denotes the push-forward operator.

Due to the intractable computation of optimal transport distance ($n$ is very large, the implicit density of $G_\theta \sharp p(z)$), mini-batch losses (m-OT, m-UOT, m-POT) are used as surrogate losses to train the generator $G_\theta$. The idea is to estimate the gradient of the mini-bath losses to update the neural network $\theta$. In practice, the real metric space of data samples is unknown, hence, adversarial training is used as unsupervised metric learning \cite{genevay2018learning,salimans2018improving}. In partial, a discriminator function $F_\phi: \mathcal{X} \to \mathcal{H}$ where $\mathcal{H}$ is a chosen feature space. The function $F_\phi$ is trained by maximizing the mini-batch OT loss. For greater detail, the training procedure is described in Algorithm~\ref{alg:DGM}. We would like to recall that there are also several works that uses sliced Wasserstein for generative models~\cite{deshpande2018generative,deshpande2019max,kolouri2019generalized,kolouri2016sliced,nguyen2022amortized,nguyen2022revisiting} in mini-batch setting. Therefore, we can also improve them by using sliced partial OT~\cite{bonneel2019spot}.

\begin{algorithm}[!t]
   \caption{Mini-batch Deep Generative Model}
   \label{alg:DGM}
\begin{algorithmic}
   \STATE {\bfseries Input:} $k$, $m$, data distribution $\mu_n$, prior distribution $p(z)$, chosen mini-batch loss $L_{\text{DGM}}\in \{\text{OT},\text{UOT},\text{POT}\}$
   \STATE Initialize $G_\theta$; $F_\phi$
   \WHILE{$\theta$ does not converge}
   \STATE $\text{grad}_\theta  \leftarrow \boldsymbol{0}$
   \STATE $\text{grad}_\phi  \leftarrow \boldsymbol{0}$
   \FOR{$i=1$ {\bfseries to} $k$}
   \STATE Sample $x_1,\ldots,x_m$ from $\mu_n$
   \STATE Sample $z_i,\ldots,z_m$ from $p(z)$
   \STATE Compute $y_i,\ldots,y_m \leftarrow G_\theta(z_i),\ldots,G_\theta(z_m)$ 
   \STATE $X^m \leftarrow \{x_1,\ldots,x_m\}$; $Y^m \leftarrow \{y_1,\ldots,y_m\}$
   \STATE Compute  $L_{\text{DGM}} \leftarrow \frac{1}{k} L_{\text{DGM}}(F_\phi(X^m),F_\phi(Y^m))$
   \STATE $\text{grad}_\theta \leftarrow \text{grad}_\theta + \nabla_\theta L_{\text{DGM}}$
   \STATE $\text{grad}_\phi \leftarrow \text{grad}_\phi - \nabla_\phi L_{\text{DGM}}$
   \ENDFOR
   \STATE $\theta \leftarrow \text{Adam}(\theta, \text{grad}_\theta)$
   \STATE $\phi \leftarrow \text{Adam}(\phi, \text{grad}_\phi)$
   \ENDWHILE
\end{algorithmic}
\end{algorithm}

\textbf{Evaluation: } To evaluate the quality of the generator $G_\phi$, we use the FID score \cite{heusel2017gans} to compare the closeness of $G_\theta \sharp p(z)$ to $\mu_n$.

\textbf{On debiased mini-batch energy: } Both m-OT and m-POT can be extended to debiased energy versions based on the work \cite{salimans2018improving}. However, in this paper, we want to focus on the effect of misspecified matchings on applications, hence, the original mini-batch losses are used.
\subsection{Mini-batch Color Transfer}
\label{subsec:mini-batchcolortransfer}
In this section, we first state the mini-batch color transfer algorithm that we used to compare mini-batch methods in Algorithm~\ref{alg:colortransfer}. The idea is to transfer color from a part of a source image to a part of a target image via a barycentric map that is obtained from two mini-batch measures (the source mini-batch and the target mini-batch). This process is repeated multiple times until reaching the chosen number of steps. The algorithm is introduced in ~\cite{fatras2020learning}.

\begin{algorithm}[!t]
   \caption{Color Transfer with mini-batches }
   \label{alg:colortransfer}
\begin{algorithmic}
   \STATE {\bfseries Input:} $k, m,$ source domain $X_s \in \mathbb{R}^{n \times d}$, target domain $X_t \in \mathbb{R}^{n \times d}$,  $T \in \{\text{OT},\text{UOT},\text{POT}\}$
   \STATE Initialize $Y_s \in \mathbb{R}^{n \times d}$
   \FOR{$i=1$ {\bfseries to} $k$}
   \STATE Select set $A$ of m samples in $X_s$
   \STATE Select set $B$ of m samples in $X_t$
   \STATE Compute the cost matrix $C_{A,B}$
   \STATE $\pi \leftarrow T (C_{A,B},\boldsymbol{u}_m,\boldsymbol{u}_m))$ 
   \STATE $\left.\left.Y_s\right|_{A} \leftarrow Y_s\right|_A+\left.\pi \cdot X_t\right|_{B}$
   \ENDFOR
   
   \STATE {\bfseries Output:} $Y_s$
\end{algorithmic}
\end{algorithm}

\subsection{Mini-batch Gradient Flow}
\label{subsec:mini-batchgradient}

In this appendix, we describe the gradient flow application and the usage of mini-batch methods in this case.

Gradient flow is a non-parametric method to learn a generative model. The goal is to find a distribution $\mu$ that is close to  the data distribution $\nu$ in sense of a discrepancy $D$. So, a gradient flow can be defined as:
\begin{align}
    \partial_t \mu_t = -\nabla_{\mu_t} D(\mu_t,\nu)
\end{align}
We follow the  Euler scheme to solve this equation as in \cite{feydy2019interpolating}, starting from an initial distribution at time $t = 0$. In this paper, we consider using mini-batch losses such as m-OT, m-UOT, and m-POT as the discrepancy $D$.

\section{Additional experiments}
\label{sec:add_exp}

We first visualize the transportation of the UOT and the POT in the context of mini-batch in Appendix~\ref{subsec:visual}. Next, we run a toy example of a $10 \times 10$ transportation problem to compare m-OT, m-UOT, m-POT, and full-OT in Appendix~\ref{subsec:matrix}. In the same section, we also investigate the role of transportation fraction $s$ in m-POT. We then show tables that contain results of all run settings in domain adaptation in Appendix~\ref{subsec:addexp_DA}. In Appendix~\ref{subsec:addexp_PDA}, we report detailed results of partial domain adaptation in each run. After that, we provide quantitative results on the deep generative model as well as show some randomly generated images in Appendix~\ref{subsec:exp_dgm}. Moreover, we discuss experimental results on color transfer and gradient flow in Appendices~\ref{subsec:color_transfer} and~\ref{subsec:gflow}, respectively. Next, we discuss the comparison between m-POT and m-OT (m-UOT) with the mini-batch size raised by 1 for m-OT (m-UOT) in Appendix~\ref{subsec:equi_OT}. Finally, we discuss the computational time of m-OT, m-UOT, m-POT, and their two-stage versions on different applications in Appendix~\ref{subsec:time}.

\subsection{Transportation visualization }
\label{subsec:visual}

\textbf{Transportation of UOT: } We use Example~\ref{example:illustration} and vary the parameter $\tau$ to see the changing of the transportation plan of UOT. We show in Figure~\ref{fig:mUOT} the behavior of UOT that we observed. According to the figure, UOT's transportation plan is always non-zero for every value of $\tau$. Increasing $\tau$ leads to a higher value for entries of the transportation matrix. In alleviating misspecified matchings, UOT can cure the problem to some extent with the right choice of $\tau$ (e.g. 0.2, 0.5). In particular, the mass of misspecified matchings is small compared to the right matchings. However, it is not easy to know how many masses are transported in total in UOT.

\begin{figure*}[!t]
\begin{center}

  \begin{tabular}{c}
\widgraph{0.95\textwidth}{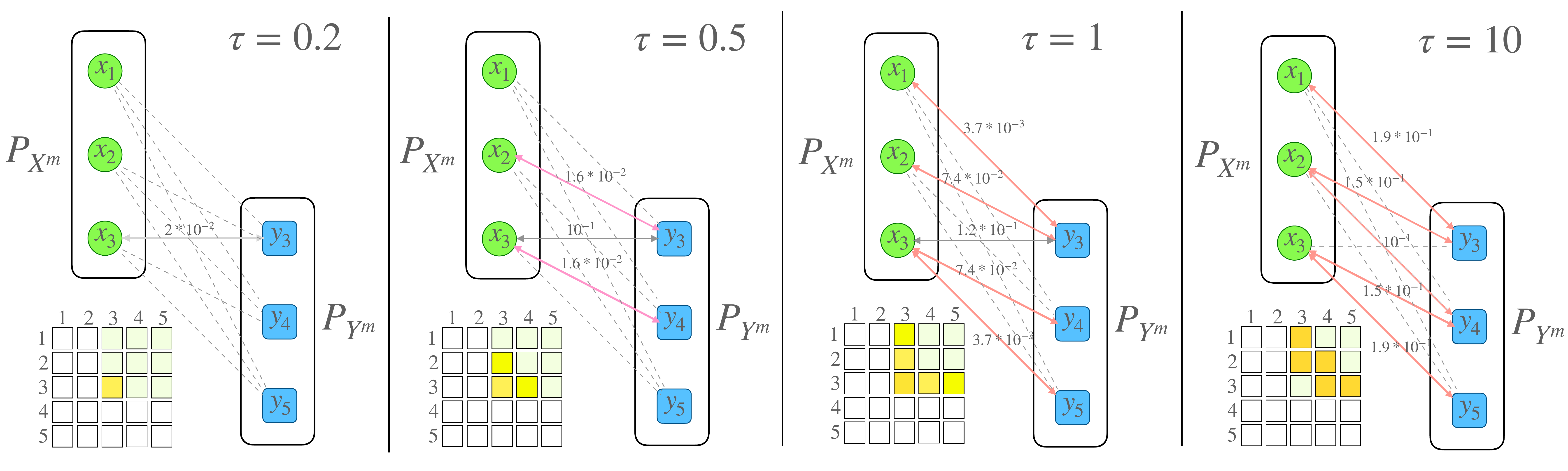} 
  \end{tabular}
  \end{center}
  \vskip -0.2in
  \caption{
  \footnotesize{The illustration of Example~\ref{example:illustration} for m-UOT. As in Figure~\ref{fig:mOT}, the green points and blue points are respectively the supports of the empirical measures $\mu_n$ and $\nu_n$. Black solid arrows represent the optimal mappings between $\mu_n$ and $\nu_n$. Red solid arrows represent misspecified mappings. Dashed arrows are mappings that have very small masses. The parameter $\tau$ is the coefficient of the marginal relaxation term of UOT. The $5\times 5$ matrix is the incomplete transportation matrix $\pi^{\text{UOT}^\tau_\phi}_{P_{X^m},P_{Y^m}}$ which is created from solving UOT between $P_{X^m}$ and $P_{Y^m}$. The boldness of the color of arrows and entries of the transportation matrix represent their mass values. 
}
} 
  \label{fig:mUOT}
  \vskip -0.2in
\end{figure*}

\textbf{Sensitivity to the scale of cost matrix: } We consider an extended version of Example~\ref{example:illustration} where all the supports are scaled by a scalar 10. Again, we use the same set of $\tau  \in \{0.2, 0.5, 1, 10\}$ then show the transportation matrices of UOT in Figure~\ref{fig:mUOTscale}. From the figure, we can see that the good choice of $\tau$ depends significantly on the scale of the cost matrix. So, UOT could not perform well in applications that change the supports of measures frequently such as deep generative models. Moreover, this sensitivity also leads to big efforts to search for a good value of $\tau$ in applications. In contrast, with the same set of the fraction of masses, $s$ in Figure~\ref{fig:mOT}, the obtained transportation from POT is unchanged when we scale the supports of measures. We show the results in Figure~\ref{fig:mPOTscale} that suggests POT is a stabler choice than UOT.

\begin{figure*}[!t]
\begin{center}

  \begin{tabular}{c}
\widgraph{0.95\textwidth}{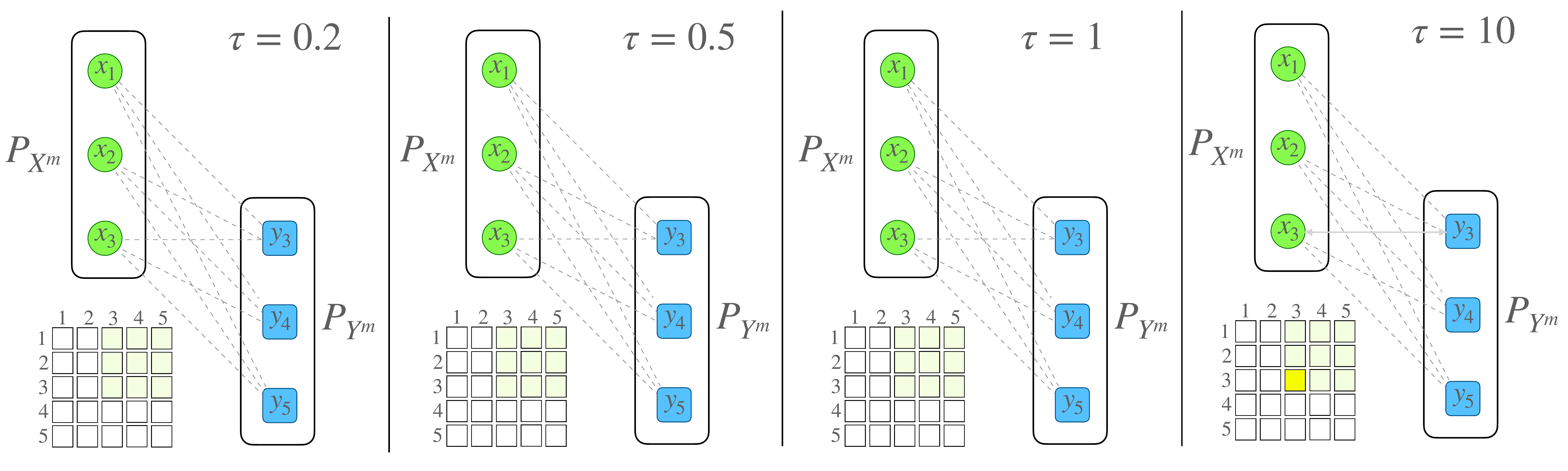} 
  \end{tabular}
  \end{center}
  \vskip -0.2in
  \caption{
  \footnotesize{The UOT's illustration of Example \ref{example:illustration} with supports of measures scaled by 10. As in Figure~\ref{fig:mOT}, the green points and blue points are respectively the supports of the empirical measures $\mu_n$ and $\nu_n$. Black solid arrows represent the optimal mappings between $\mu_n$ and $\nu_n$. Red solid arrows represent misspecified mappings. Dashed arrows are mappings that have very small masses. The parameter $\tau$ is the coefficient of the marginal relaxation term of UOT. The $5\times 5$ matrix is the incomplete transportation matrix $\pi^{\text{UOT}^\tau_\phi}_{P_{X^m},P_{Y^m}}$ which is created from solving UOT between $P_{X^m}$ and $P_{Y^m}$. The boldness of the color of arrows and entries of the transportation matrix represent their mass values. 
}
} 
  \label{fig:mUOTscale}
  \vskip -0.2in
\end{figure*}

\begin{figure*}[!t]
\begin{center}

  \begin{tabular}{c}
\widgraph{0.95\textwidth}{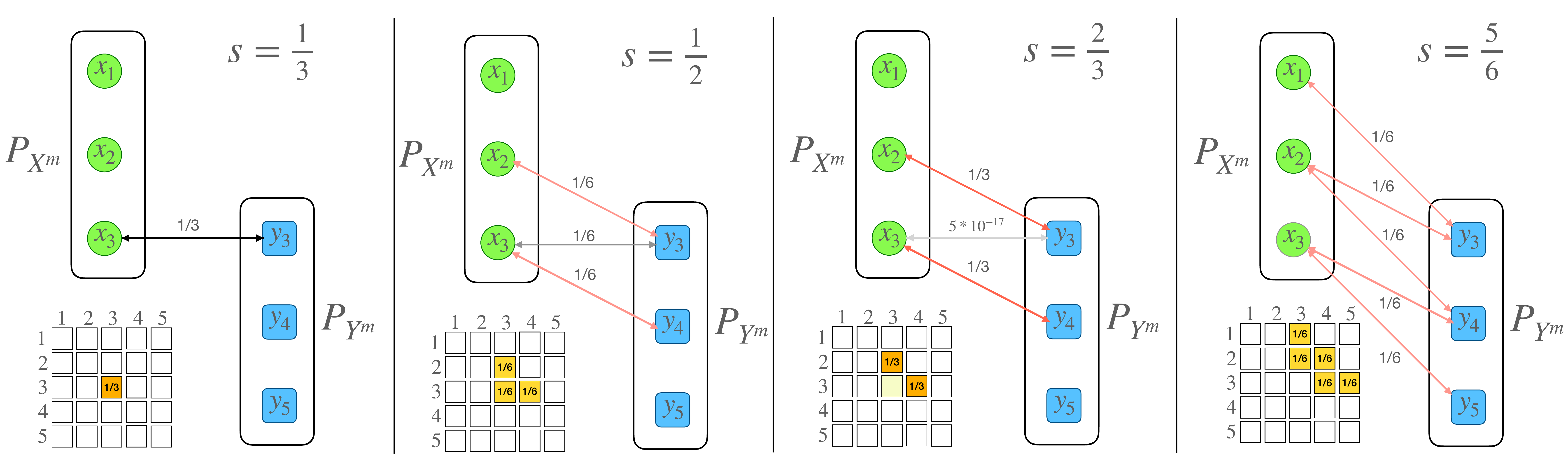} 
  \end{tabular}
  \end{center}
  \vskip -0.2in
  \caption{
  \footnotesize{The POT's illustration of Example~\ref{example:illustration} with supports of measures scaled by 10. As in Figure~\ref{fig:mOT}, the green points and blue points are respectively the supports of the empirical measures $\mu_n$ and $\nu_n$. Black solid arrows represent the optimal mappings between $\mu_n$ and $\nu_n$. Red solid arrows represent misspecified mappings. Dashed arrows are mappings that have very small masses. The parameter $s$ is the fraction of masses POT. The $5\times 5$ matrix is the incomplete transportation matrix $\pi^{\text{POT}_s}_{P_{X^m},P_{Y^m}}$ which is created from solving UOT between $P_{X^m}$ and $P_{Y^m}$. The boldness of the color of arrows and entries of the transportation matrix represent their mass values. 
}
} 
  \label{fig:mPOTscale}
  \vskip -0.2in
\end{figure*}

\textbf{Non-overlapping mini-batches: } We now demonstrate the behavior of OT, UOT, and POT when a pair of mini-batches does not contain any global optimal mappings. According to the result of UOT, POT, and OT (POT s=1) in Figure~\ref{fig:nonoverlap}, we can see that although UOT and POT cannot avoid misspecified matchings in this case, their solution is still better than OT in the sense that they can maps a sample to the nearest one to it.
\begin{figure*}[!t]
\begin{center}

  \begin{tabular}{c}
\widgraph{0.95\textwidth}{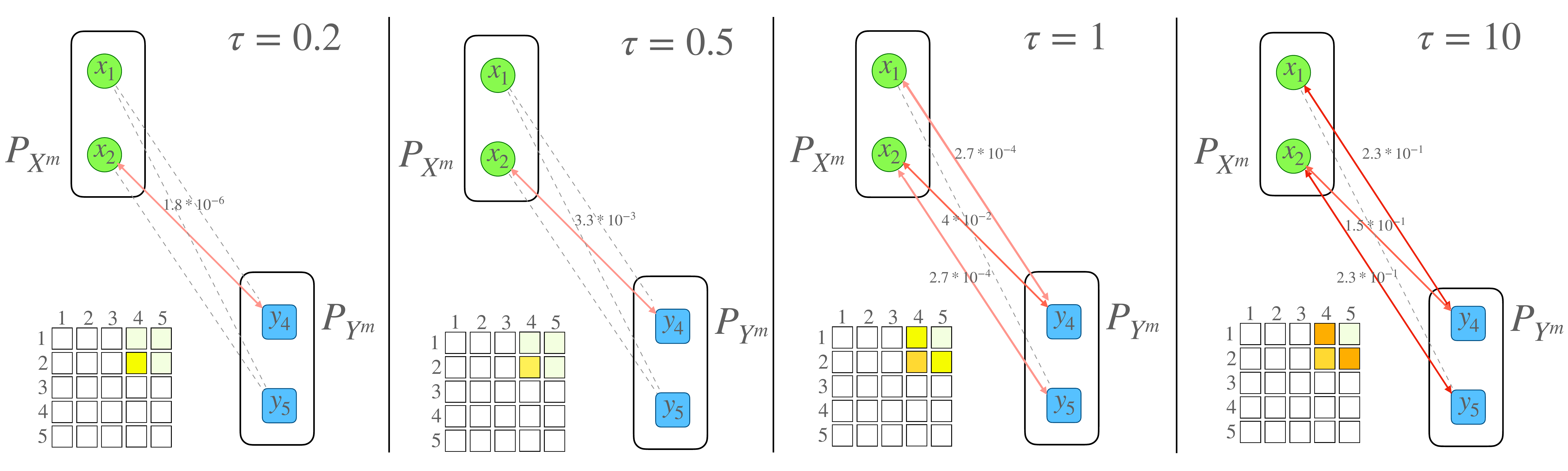}\\
\widgraph{0.95\textwidth}{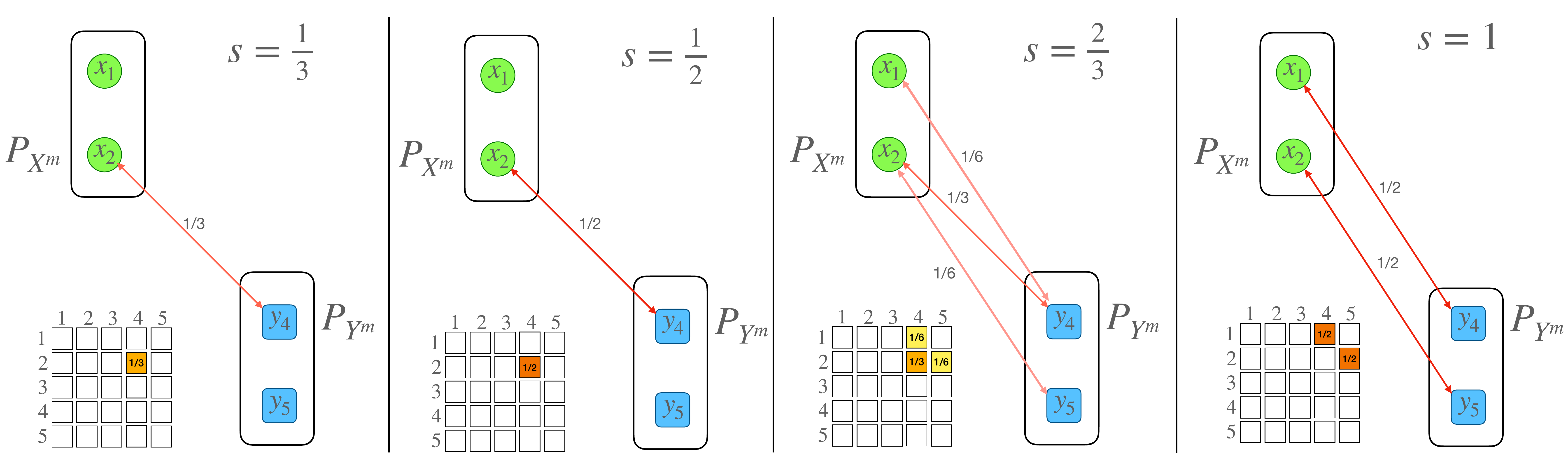}
  \end{tabular}
  \end{center}
  \vskip -0.2in
  \caption{
  \footnotesize{The illustration of Example~\ref{example:illustration} in the non-overlapping case for both UOT, POT, and OT (POT $s=1$). The appearances of these graphs are similar to previous figures.
}
} 
  \label{fig:nonoverlap}
  \vskip -0.2in
\end{figure*}

\textbf{Discussion on the fraction of masses $s$: } First, as indicated in toy examples in Figure~\ref{fig:mOT}, choosing small $s$ is a way to mitigate misspecified mappings of m-OT; however, it may also cut off other good mappings when they exist. Therefore, small $s$ tends to require a larger number of mini-batches to retrieve enough optimal mappings of the full-OT between original measures. Due to that issue, the fraction of masses $s$ should be chosen adaptively based on the number of good mappings that exist in a pair of mini-batches. In particular, if a pair of mini-batches contains several optimal pairs of samples, $s$ will be set to a high value. Nevertheless, knowing the number of optimal matchings in each pair of mini-batches requires additional information of original large-scale measures. Since designing an adaptive algorithm for $s$ is a non-trivial open question, we leave this to future work. In the paper, we will only use grid searching in a set of chosen values of $s$ in our experiments and use it for all mini-batches.

\begin{figure*}[!t]
\begin{center}
  \begin{tabular}{c}
  \widgraph{0.95\textwidth}{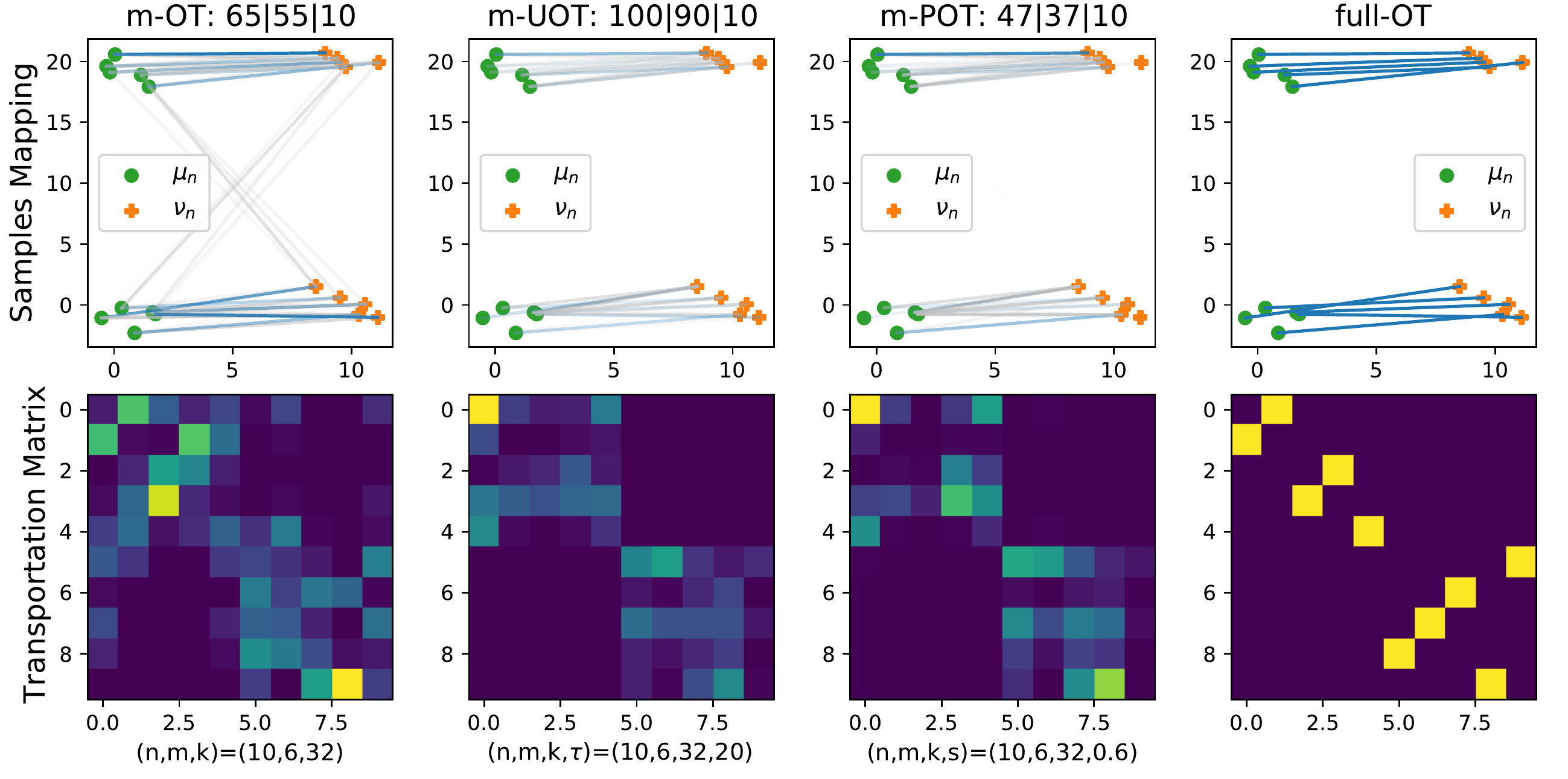} 
  \end{tabular}
  \end{center}
  \vskip -0.2in
  \caption{
  \footnotesize{The first row presents sample mappings between $\mu_n$ and $\nu_n$ from different methods. These mappings are extracted from the transportation matrices which are shown in the second row. Blue lines represent optimal mappings while silver lines represent misspecified mappings. The boldness of the lines and entries of matrices are subject to their masses. There are three numbers near the name of mini-batch methods that are the total number of mappings, the number of misspecified mappings, and the number of optimal mappings, respectively.
}
} 
  \label{fig:OTmatrix_main}
  \vskip -0.15in
\end{figure*}

\begin{figure}[!t]
\begin{center}
  \begin{tabular}{c}
  \widgraph{0.95\textwidth}{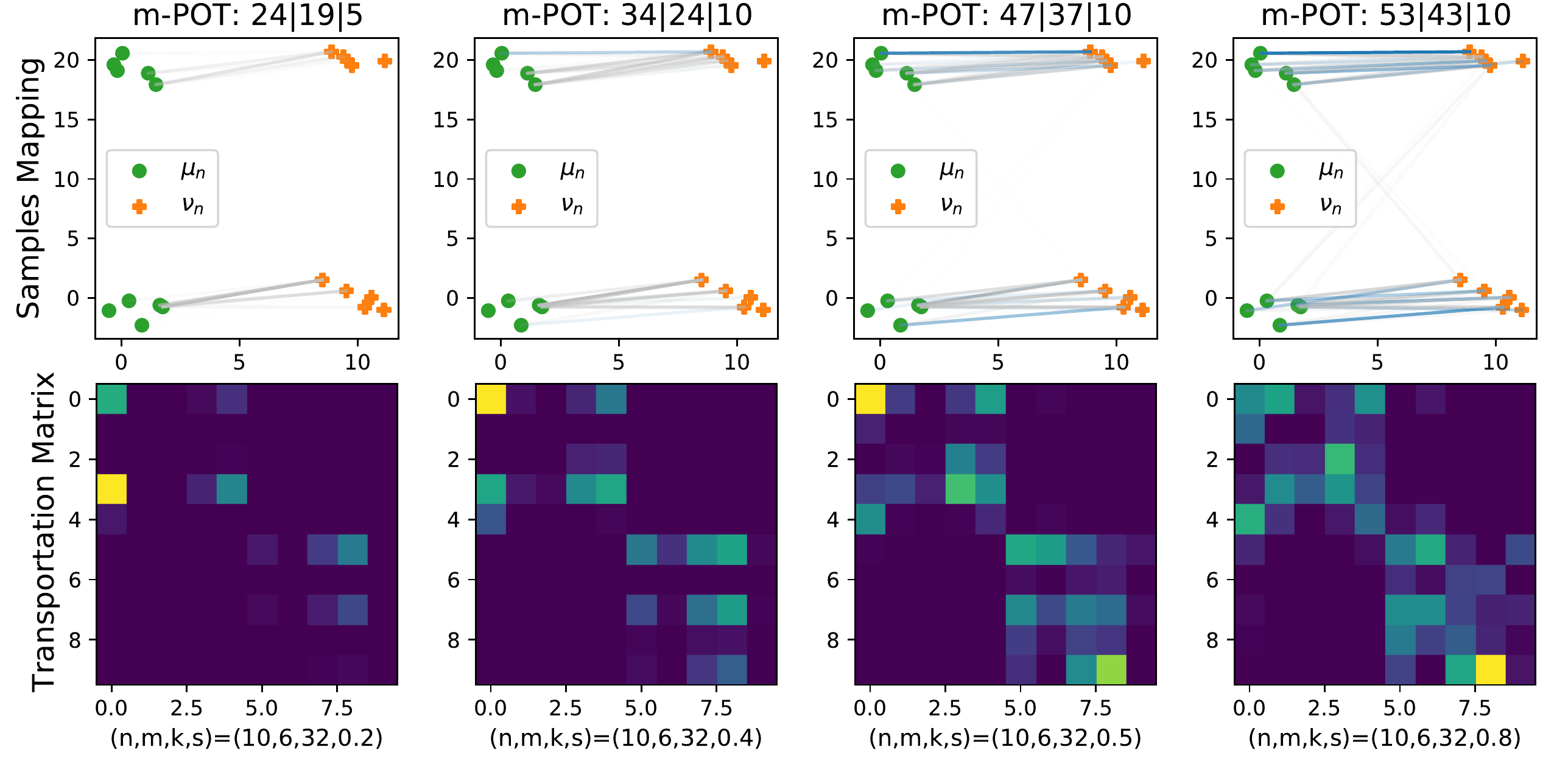} 
  \end{tabular}
  \end{center}
  \vskip -0.1in
  \caption{
  \footnotesize{The first row presents sample mappings between $\mu_n$ and $\nu_n$ from different methods. These mappings are extracted from the transportation matrices which are shown in the second row. Blue lines represent optimal mappings while silver lines represent misspecified mappings. The boldness of the lines and entries of matrices are subject to their masses. There are three numbers near the name of mini-batch methods that are the total number of mappings, the number of misspecified mappings, and the number of optimal mappings, respectively.
  }
  } 
  \label{fig:POTmatrix_main}
  \vskip -0.1in
\end{figure}

\subsection{Aggregated Transportation Matrix}
\label{subsec:matrix}

\textbf{Comparison to m-OT and m-UOT: } We demonstrate a simulation with two 10-point empirical measures $\mu_n$ and $\nu_n$ which are drawn from bi-modal distributions $ \frac{1}{2}\mathcal{N} \left( \scalebox{0.65}{$\begin{bmatrix} 0 \\0 \end{bmatrix},\begin{bmatrix} 1& 0 \\0& 1 \end{bmatrix}$} \right) + \frac{1}{2}\mathcal{N} \left( \scalebox{0.65}{$\begin{bmatrix} 0 \\20 \end{bmatrix},\begin{bmatrix} 1& 0 \\0& 1 \end{bmatrix}$} \right)$  and $ \frac{1}{2}\mathcal{N} \left( \scalebox{0.65}{$\begin{bmatrix} 10 \\0 \end{bmatrix},\begin{bmatrix} 1& -0.8 \\-0.8& 1 \end{bmatrix}$} \right) + \frac{1}{2}\mathcal{N} \left( \scalebox{0.65}{$\begin{bmatrix} 10 \\20 \end{bmatrix},\begin{bmatrix} 1& -0.8 \\-0.8& 1 \end{bmatrix}$} \right)$, respectively. We run m-OT, m-UOT (with KL divergence), and m-POT. All methods are run with $k=32,m=6$. Then, we visualize the mappings and transportation matrices in Figure~\ref{fig:OTmatrix_main}. For illustration purposes, there are three numbers shown near the names of mini-batch approaches. They are the total number of mappings, the number of misspecified mappings, and the number of optimal mappings in turn. The total number of mappings is the number of non-zero entries in a transportation matrix. The number of optimal mappings is the number of mappings that exist in the full-OT's transportation plan while the number of misspecified mappings is for mappings that do not exist. We can observe that m-POT and m-UOT provide more meaningful transportation plans than m-OT, namely, masses are put to connections between correct clusters. Importantly, we also observe that m-POT has lower misspecified matchings than m-OT (37 compared to 55). As mentioned in the background section, UOT tends to provide a transportation plan that contains non-zero entries; therefore, m-UOT still has many misspecified matchings through the weights of those matchings can be very small that can be ignored in practice. Note that, in the simulation results we select the best value of  $\tau \in \{0.1,0.5,1,2,5,10,15,20,30,40,50\}$ for m-UOT and best value of $s \in \{0.1,0.2,0.3,0.4,0.5,0.6,0.7,0.8,0.9\}$ for m-POT.


\textbf{The role of the fraction of masses $s$: } Here, we repeat the simulation in the main text with two empirical measures of 10 samples. The difference is that we run only m-POT with the value of the fraction of masses $s \in \{0.2, 0.4, 0.6, 0.8\}$. The result is shown in Figure~\ref{fig:POTmatrix_main}. It is easy to see that a smaller value of $s$ leads to a smaller number of misspecified mappings. However, a small value of $s$ also removes good mappings (e.g. $s=0.2$ only has $5$ correct mappings). With the right choice of $s$ (e.g. $s=0.4$), we can obtain enough good mappings while having a reasonable number of misspecified matchings.

\subsection{Deep Domain Adaptation Experiments}
\label{subsec:addexp_DA}
We first want to emphasize that we have used the best parameter settings for m-UOT that are reported in \cite{fatras2021unbalanced} in each experiment. We reproduce and compare the performance of our proposed methods including m-POT, TS-OT, TS-UOT, and TS-POT with two mini-batch methods (m-OT and m-UOT) and two non-OT baselines (DANN~\cite{ganin2016domain} and ALDA~\cite{chen2020adversarial}). 

\begin{table}[!t]
    \caption{Details of DA results on digit datasets over 3 runs. Entries of the table are classification accuracies in the target domain. The fraction of masses $s$ for SVHN to MNIST, USPS to MNIST, and MNIST to USPS are 0.85, 0.90, and 0.80, respectively.}
    \vskip 0.1in
    \label{table:DA_digits_details}
    \centering
    \scalebox{1}{
        \begin{tabular}{ccccccc}
        \toprule
        Run & Scenario & DANN & ALDA & m-OT & m-UOT & m-POT \\
        \midrule
        1 & SVHN to MNIST & 95.59 & 98.73 & 93.84 & 99.06 & \textbf{99.08} \\
        & USPS to MNIST & 94.87 & 98.38 & 96.97 & \textbf{98.75} & 98.72 \\   
        & MNIST to USPS & 92.03 & 95.46 & 87.59 & 95.76 & \textbf{96.06} \\
        \midrule
        2 & SVHN to MNIST & 96.20 & 98.92 & 94.08 & 98.86 & \textbf{98.97} \\
        & USPS to MNIST & 94.68 & 98.27 & 96.77 & 98.6 & \textbf{98.72} \\
        & MNIST to USPS & 91.98 & 95.32 & 85.30 & 95.86 & \textbf{96.01} \\
        \midrule
        3 & SVHN to MNIST & 95.60 & 98.77 & 94.61 & 98.74 & \textbf{98.88} \\
        & USPS to MNIST & 94.58 & 98.21 & 96.39 & 98.26 & \textbf{98.45} \\
        & MNIST to USPS & 90.88 & 95.07 & 87.89 & 95.86 & \textbf{96.06} \\
        \bottomrule
        \end{tabular}
    }
\end{table}

\begin{figure*}[t!]
    \begin{center}
        \begin{tabular}{c}
        \widgraph{0.95\textwidth}{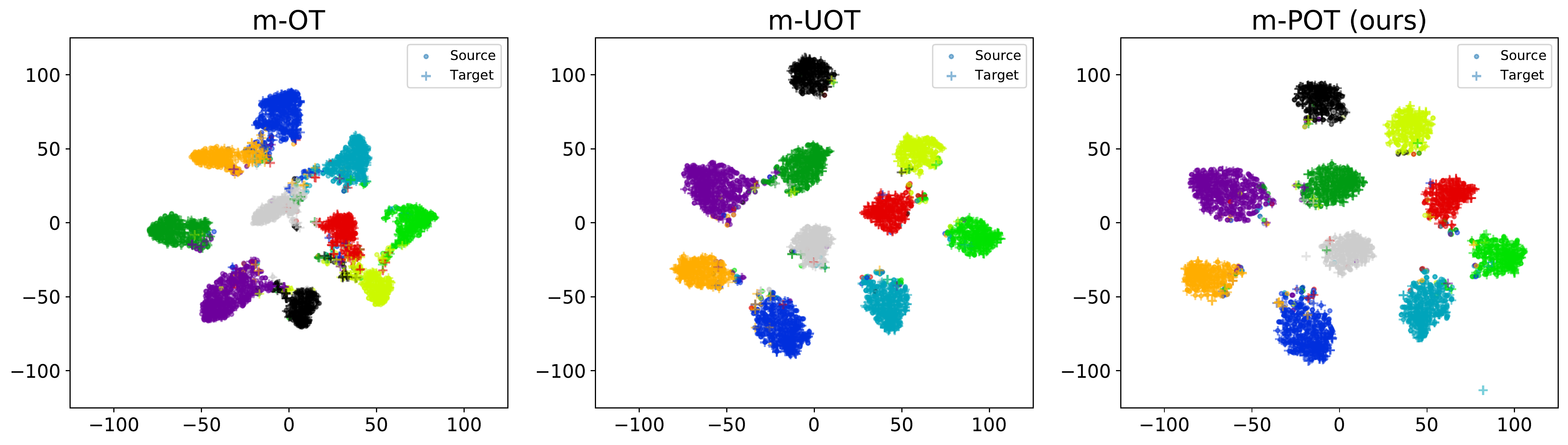} 
        \end{tabular}
    \end{center}
    \vskip -0.2in
    \caption{
        \footnotesize{T-SNE embeddings of 2000 test samples for SVHN (source) and MNIST (target) for three mini-batch methods. Samples are colored based on their class labels. T-SNE embeddings for the VisDA dataset are given in Figure~\ref{fig:tsne_visda} in Appendix~\ref{subsec:addexp_DA}.}
    } 
    \label{fig:tsne_digits}
    \vskip -0.1in
\end{figure*}

\textbf{Digits datasets: } In this experiment, we run five different methods DANN, ALDA, m-OT, m-UOT, and m-POT on three adaptation tasks including SVHN to MNIST, USPS to MNIST, and MNIST to USPS. Each method was run 3 times and the results are reported in Table~\ref{table:DA_digits_details}. We did not apply TS-POT on digits datasets because the accuracy of m-POT is already high enough and the used batch size $m = 500$ is also quite large. As can be seen from Table~\ref{table:DA_digits_details}, m-POT outperforms m-OT in all runs with a significant gap. In addition, m-POT also leads to a more favorable performance than m-UOT in almost all experiments. In comparison with DANN and ALDA, our method still yields better accuracy on all adaptation tasks over 3 runs.

\begin{table}[t]
    \caption{Details of DA results on the Office-Home dataset. Entries of the table are classification accuracies in the target domain. The number of mini-batches for m-OT, m-UOT, TS-OT, TS-UOT, and TS-POT are set to their best-performing values in Table~\ref{table:DA_efficient_office_detail}, which are 4, 1, 4, 4, and 2, respectively. For m-POT, the number of mini-batches $k$ is also set to 1 to have a fair comparison between m-UOT and m-POT. The fraction of mass $s$ of m-POT is selected from the set $\{ 0.5, 0.55, 0.6, 0.65, 0.7 \}$. The fraction of masses $s$ for TS-POT is set to 0.6.}
    \vskip 0.1in
    \label{table:DA_office_details}
    \centering
    \scalebox{0.85}{
        \begin{tabular}{ccccccccccccccc}
        \toprule
        Run & Method & A2C & A2P & A2R & C2A & C2P & C2R & P2A & P2C & P2R & R2A & R2C & R2P & Avg \\
        \midrule
        1 & DANN & 47.81 & 66.86 & 74.96 & 53.11 & 62.51 & 65.69 & 53.19 & 44.08 & 74.50 & 64.85 & 53.17 & 79.39 & 61.68 \\
        & ALDA & 54.07 & 74.81 & 77.28 & 61.97 & 71.64 & 72.96 & 59.54 & 51.34 & 76.57 & 68.15 & 56.40 & 82.11 & 67.24 \\
        & m-OT & 51.89 & 70.74 & 75.53 & 59.75 & 66.50 & 70.26 & 57.19 & 48.04 & 75.49 & 66.58 & 55.74 & 78.08 & 64.65 \\
        & m-UOT & 54.76 & 74.23 & 80.33 & 65.39 & \textbf{75.40} & 74.78 & 65.93 & 53.79 & 80.17 & 73.80 & 59.47 & 83.89 & 70.16 \\
        & m-POT (Ours) & 55.05 & 73.82 & 80.81 & 66.50 & 74.95 & 76.36 & 64.40 & 53.50 & \textbf{80.56} & 74.62 & 59.59 & 83.71 & 70.32 \\
        & TS-OT (Ours) & 53.91 & 71.35 & 77.12 & 59.79 & 69.39 & 71.93 & 59.21 & 51.13 & 76.54 & 66.09 & 57.14 & 80.20 & 66.15 \\
        & TS-UOT (Ours) & 56.11 & 73.67 & 80.03 & 65.18 & 73.19 & 76.43 & 63.66 & 54.73 & 79.87 & 71.45 & 60.32 & 82.92 & 69.80 \\
        & TS-POT (Ours) & \textbf{57.05} & \textbf{76.30} & \textbf{81.48} & \textbf{68.19} & 73.51 & \textbf{76.70} & \textbf{66.05} & \textbf{55.35} & 80.40 & \textbf{75.61} & \textbf{60.78} & \textbf{84.37} & \textbf{71.32} \\
        \midrule
        2 & DANN & 47.86 & 67.25 & 74.94 & 53.81 & 64.41 & 66.63 & 52.37 & 44.4 & 74.00 & 65.76 & 52.85 & 79.23 & 61.96 \\
        & ALDA & 54.43 & 74.77 & 77.23 & 61.23 & 69.7 & 72.96 & 60.24 & 49.67 & 76.52 & 67.45 & 55.69 & 81.93 & 66.82 \\
        & m-OT & 51.89 & 69.21 & 75.95 & 59.58 & 66.66 & 69.75 & 57.85 & 47.77 & 75.14 & 67.08 & 55.83 & 77.95 & 64.56 \\
        & m-UOT & 55.56 & 74.63 & 81.16 & 65.84 & 74.45 & 75.08 & 64.07 & 53.24 & 79.99 & 75.07 & 60.16 & 83.62 & 70.24 \\
        & m-POT (Ours) & 56.15 & 73.71 & 80.84 & 66.34 & \textbf{74.77} & 75.88 & 64.40 & 53.13 & \textbf{80.61} & 74.66 & 59.84 & 83.85 & 70.35 \\
        & TS-OT (Ours) & 54.04 & 70.85 & 77.14 & 59.50 & 68.93 & 72.05 & 59.00 & 51.32 & 76.50 & 66.67 & 56.95 & 80.38 & 66.11 \\
        & TS-UOT (Ours) & 56.52 & 73.46 & 80.22 & 64.98 & 73.03 & 76.50 & 63.63 & 54.46 & 80.03 & 70.99 & 59.84 & 82.86 & 69.71 \\
        & TS-POT (Ours) & \textbf{57.53} & \textbf{76.32} & \textbf{81.62} & \textbf{68.69} & 72.22 & \textbf{76.08} & \textbf{66.13} & \textbf{54.59} & 80.47 & \textbf{75.53} & \textbf{60.48} & \textbf{84.14} & \textbf{71.15} \\
        \midrule
        3 & DANN & 48.09 & 67.13 & 74.64 & 54.47 & 63.48 & 66.93 & 53.40 & 44.58 & 74.8 & 65.97 & 52.85 & 79.61 & 62.16 \\
        & ALDA & 53.61 & 75.08 & 76.91 & 60.9 & 70.53 & 72.32 & 61.19 & 52.07 & 76.89 & 68.11 & 55.74 & 81.57 & 67.08 \\
        & m-OT & 51.48 & 70.08 & 75.88 & 59.46 & 66.21 & 70.21 & 57.77 & 47.84 & 75.24 & 66.79 & 55.56 & 78.31 & 64.57 \\
        & m-UOT & 54.64 & 74.48 & 80.86 & 65.76 & \textbf{74.93} & 74.87 & 64.11 & 53.24 & 79.87 & 74.87 & 60.02 & 83.67 & 70.11 \\
        & m-POT (Ours) & 55.74 & 73.87 & 80.63 & 66.17 & \textbf{74.93} & 76.25 & 64.57 & 53.52 & \textbf{80.63} & 74.37 & 59.70 & 83.87 & 70.35 \\
        & TS-OT (Ours) & 53.72 & 70.83 & 77.12 & 60.16 & 69.27 & 71.88 & 59.33 & 51.07& 76.57 & 66.63 & 56.82 & 80.00 & 66.12 \\
        & TS-UOT (Ours) & 56.43 & 73.55 & 80.24 & 64.90 & 73.15 & 76.57 & 63.70 & 54.27 & 80.01 & 71.28 & 60.16 & 82.97 & 69.77 \\
        & TS-POT (Ours) & \textbf{56.59} & \textbf{75.78} & \textbf{81.50} & \textbf{68.44} & 72.74 & \textbf{76.82} & \textbf{66.46} & \textbf{54.66} & 80.31 & \textbf{75.57} & \textbf{60.23} & \textbf{84.41} & \textbf{71.13} \\
        \bottomrule
        \end{tabular}
    }
    \vskip -0.2in
\end{table}

\textbf{Office-Home dataset: } The mini-batch size is set to 65 for all methods on the Office-Home dataset. Utilizing the mini-batch strategy, we, however, can scale the number of samples in each stochastic update by increasing the number of mini-batches. In this experiment, the number of mini-batches $k$ for each OT-based method varies between 1 and 4. The detailed results for different values of $k$ are reported in Table~\ref{table:DA_efficient_office_detail} while Table~\ref{table:DA_office_details} only shows the results for the best choice of $k$ in each run. To have a fair comparison, the number of mini-batches of both m-UOT and m-POT is set to 1 even though the best choice of $k$ for m-POT is 4. According to Table~\ref{table:DA_office_details}, m-POT is still much better than m-OT on all tasks on the Office-Home dataset in all runs. Similar to digits datasets, m-POT produces slightly higher classification accuracy in the target domain than m-UOT over 3 runs. Compared to the old implementation, the two-stage implementation boosts the performance of both m-OT and m-POT but m-UOT. In each run, TS-OT witnesses a noticeable growth of at least $1.5\%$ in the average accuracy of 12 tasks. In terms of TS-POT, it consistently yields the highest average accuracy of more than $71\%$. In addition, TS-POT outperforms other methods on 10 out of 12 adaptation scenarios (excluding C2P and P2R). 

\begin{table}[t]
    \caption{Details of DA results on the VisDA dataset. Column ``All" indicates the accuracy when performing classification on all 12 classes. The following 12 columns represent the precision when classifying each class separately. The last column ``Avg" shows the average precision over 12 classes. Following JUMBOT's paper~\citep{fatras2021unbalanced}, the value in column ``All" is reported in Table~\ref{table:DA_visda_summary} and used to compare between methods. For both m-POT and TS-POT, the fraction of masses $s$ is set to 0.75. The number of mini-batches for m-OT, m-UOT, m-POT, TS-OT, TS-UOT, and TS-POT are all set to 4, which is the best performing value in Table~\ref{table:DA_efficient_visda_detail}.}
    \vskip 0.1in
    \label{table:DA_visda_details}
    \centering
    \scalebox{0.8}{
        \begin{tabular}{cccccccccccccccc}
        \toprule
        Run & Method & All & plane & bcycl & bus & car & house & knife & mcycl & person & plant & sktbrd & train & truck & Avg \\
        \midrule
        1 & DANN & 68.09 & 85.77 & 91.41 & 67.57 & 68.25 & 91.39 & 9.79 & 59.99 & 78.86 & 72.40 & 37.92 & 61.47 & 47.40 & 64.35 \\
        & ALDA & 71.38 & \textbf{95.15} & \textbf{96.61} & 68.28 & 70.77 & 91.39 & 9.21 & 61.26 & 78.77 & 75.50 & 39.38 & 70.57 & 69.34 & 68.85 \\
        & m-OT & 62.46 & 69.03 & 87.03 & \textbf{83.68} & 69.26 & 89.22 & 18.81 & 47.61 & 76.50 & 65.33 & 25.87 & 57.74 & 50.41 & 61.71 \\
        & m-UOT & 71.91 & 89.28 & 93.50 & 76.07 & 73.11 & 92.56 & 5.49 & 62.51 & 73.34 & 79.48 & 34.96 & 68.98 & 62.40 & 67.64 \\
        & m-POT (Ours) & 73.52 & 92.40 & 92.90 & 69.83 & 74.52 & 91.85 & \textbf{67.83} & 60.43 & 76.89 & 82.42 & 34.05 & 71.04 & 59.75 & 72.83 \\
        & TS-OT (Ours) & 68.12 & 83.89 & 82.43 & 78.13 & 70.64 & 90.40 & 6.93 & 54.08 & \textbf{82.26} & 83.46 & 30.78 & 57.01 & 62.71 & 65.23 \\
        & TS-UOT (Ours) & 71.03 & 86.27 & 89.01 & 71.27 & 70.04 & 90.57 & 5.00 & 63.49 & 77.65 & \textbf{86.16} & 38.17 & 60.06 & 63.08 & 66.73 \\
        & TS-POT (Ours) & \textbf{76.50} & 94.24 & 95.69 & 69.43 & \textbf{76.90} & \textbf{93.16} & 57.60 & \textbf{68.52} & 67.27 & 85.69 & \textbf{49.25} & \textbf{72.63} & \textbf{84.03} & \textbf{76.20} \\
        \midrule
        2 & DANN & 67.28 & 84.83 & 91.51 & 70.88 & 67.05 & \textbf{94.49} & 10.10 & 62.60 & 81.80 & 69.68 & 35.76 & 54.05 & 44.33 & 63.92 \\
        & ALDA & 71.10 & \textbf{95.37} & \textbf{96.62} & 73.46 & 69.01 & 94.21 & 8.24 & 54.51 & 79.61 & 80.95 & \textbf{41.11} & 70.94 & 74.09 & 69.84 \\
        & m-OT & 62.26 & 72.59 & 84.84 & \textbf{82.94} & 70.10 & 88.19 & 7.71 & 46.55 & 75.12 & 62.82 & 19.21 & 58.55 & 51.93 & 60.05 \\
        & m-UOT & 72.69 & 92.70 & 92.31 & 73.00 & 74.18 & 90.14 & 8.83 & \textbf{67.75} & 73.43 & 79.09 & 38.63 & 67.36 & 61.51 & 68.24 \\
        & m-POT (Ours) & 73.45 & 90.50 & 92.81 & 71.39 & 75.18 & 93.29 & \textbf{71.65} & 62.87 & 75.64 & 79.80 & 32.81 & 66.64 & 61.02 & 72.80 \\
        & TS-OT (Ours) & 69.66 & 86.68 & 85.67 & 72.50 & 72.64 & 89.24 & 3.85 & 57.74 & \textbf{81.88} & \textbf{83.58} & 27.33 & 62.70 & 58.82 & 65.22 \\
        & TS-UOT (Ours) & 70.94 & 89.98 & 82.82 & 74.37 & 70.75 & 90.21 & 9.79 & 62.56 & 74.60 & 79.04 & 33.81 & 64.93 & 61.03 & 66.16 \\
        & TS-POT (Ours) & \textbf{75.42} & 94.52 & 94.81 & 68.84 & \textbf{75.34} & 91.56 & 56.15 & 66.61 & 66.88 & 82.73 & 33.00 & \textbf{74.28} & \textbf{93.98} & \textbf{74.89} \\
        \midrule
        3 & DANN & 67.52 & 88.28 & 90.44 & 71.16 & 63.38 & 94.85 & 12.80 & 63.75 & \textbf{78.65} & 73.94 & 36.21 & 56.62 & 37.70 & 63.98 \\
        & ALDA & 71.18 & 94.43 & \textbf{97.21} & 65.56 & 72.76 & \textbf{95.93} & 8.61 & 58.14 & 75.80 & 86.20 & \textbf{40.35} & 63.16 & 74.99 & 69.43 \\
        & m-OT & 62.54 & 65.18 & 91.51 & \textbf{75.12} & 69.42 & 85.74 & 21.11 & 49.63 & 75.50 & 60.38 & 24.44 & 54.60 & 56.26 & 60.74 \\
        & m-UOT & 72.42 & 92.42 & 91.78 & 74.53 & 75.29 & 90.20 & 6.53 & 65.47 & 73.31 & 78.46 & 37.69 & 66.89 & 63.10 & 67.97 \\
        & m-POT (Ours) & 73.80 & 91.64 & 93.48 & 72.68 & 75.63 & 93.76 & \textbf{72.50} & 63.11 & 74.46 & 80.56 & 33.03 & 64.02 & 64.86 & 73.31 \\
        & TS-OT (Ours) & 69.64 & 85.35 & 87.85 & 73.38 & 71.68 & 85.97 & 7.77 & 58.88 & 78.62 & \textbf{87.18} & 27.76 & 56.82 & 64.07 & 65.44 \\
        & TS-UOT (Ours) & 70.77 & 83.53 & 90.42 & 73.71 & 71.63 & 91.72 & 11.61 & 62.93 & 76.84 & 80.20 & 37.86 & 58.49 & 65.77 & 67.06 \\
        & TS-POT (Ours) & \textbf{75.96} & \textbf{94.63} & 95.55 & 68.21 & \textbf{76.73} & 92.74 & 56.67 & \textbf{68.71} & 64.35 & 85.94 & 31.77 & \textbf{73.88} & \textbf{90.86} & \textbf{75.00} \\
        \bottomrule
        \end{tabular}
    }
\end{table}

\begin{figure}[t]
    \begin{center}
        \begin{tabular}{c}
        \widgraph{0.95\textwidth}{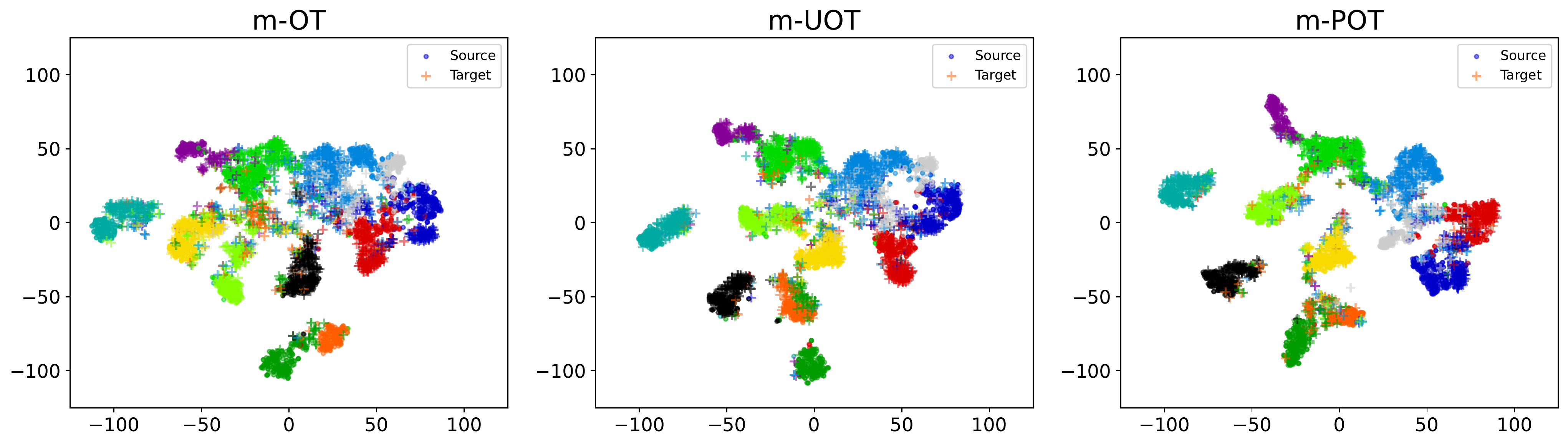} \\
        \widgraph{0.95\textwidth}{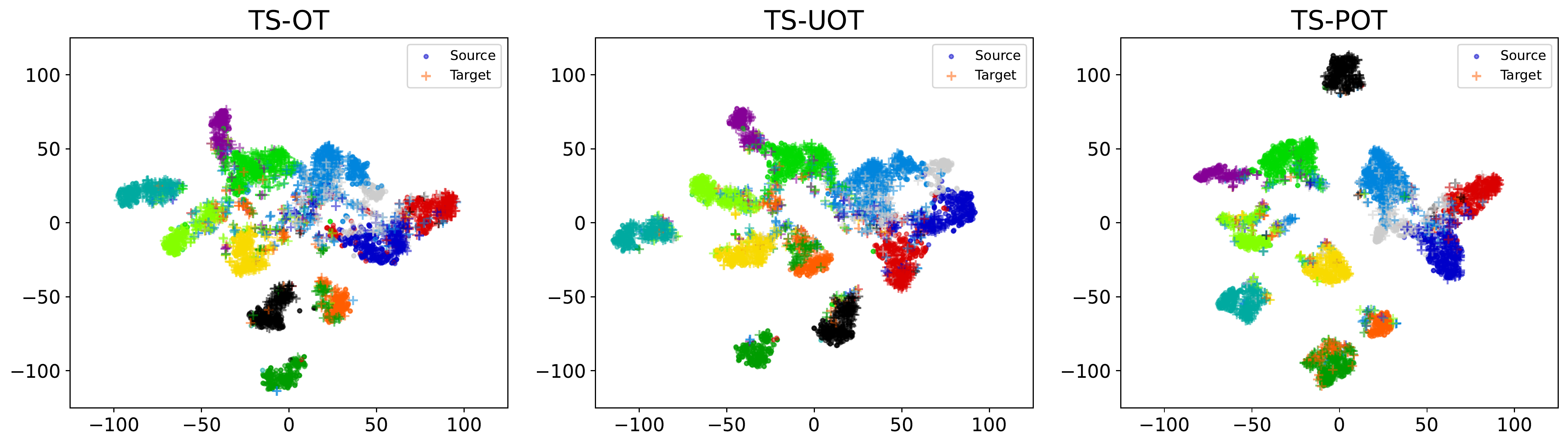} 
        \end{tabular}
    \end{center}
    \vskip -0.2in
    \caption{
      \footnotesize{T-SNE embeddings of 2000 samples for VisDA train (source) and validation (target) datasets for different methods. Samples are colored based on their class labels.}} 
    \label{fig:tsne_visda}
\end{figure}
 
\textbf{VisDA dataset: } Similar to the Office-Home dataset, the number of mini-batches $k$ for each OT-based method once again varies between 1 and 4. According to Table~\ref{table:DA_efficient_visda_detail}, the best choice of $k$ is 4 for all OT-based methods. Table~\ref{table:DA_visda_details} illustrates both the classification accuracy and precision of different methods on the VisDA dataset. On either metric, TS-POT always leads to the highest performance in all runs, demonstrating the effectiveness of our two-stage implementation on the deep domain adaptation application. On the choice of OT as the mini-batch loss, there is again a huge margin of more than $7\%$ between the accuracy of TS-OT and m-OT in the second and third runs. In terms of the conventional mini-batch implementation, m-POT yields the best classification accuracy with a clear margin of more than $1\%$ in 2 out of 3 runs. m-OT is consistently the worst performance method in all runs. Based on average class precision (the last column), the gap between m-POT and m-UOT is always larger than $4\%$ although their difference in accuracy is only about $1\%$ on average. The visualization of TSNE embeddings for the VisDA dataset in Figure~\ref{fig:tsne_visda} also reinforces the quantitative results. It can be seen clearly that clusters in the embeddings of m-POT and TS-POT are more separate than those of other methods. 

\textbf{Changing the number of mini-batches $k$: } We vary the number of mini-batches on GPU $k \in \{ 1, 2, 4\}$ and compare the performance between the new and the old implementation. The mini-batch size on GPU $m$ is set to 65 and 72 on Office-Home and VisDA datasets, respectively. To have a fair comparison, the number of mini-batches on CPU $\tilde{k}$ is set to 1 and the mini-batch size on CPU $\tilde{m}$ is set to $k \ast m$ so that the same batch size is used. Our focus here is not to obtain state-of-the-art performance, but rather to compare the two-stage implementation with the conventional implementation on the deep DA. Therefore, the fraction of masses $s$ and the entropic regularization coefficient $\epsilon$ of m-POT are fixed. This setting is different from the experiments conducted in Tables~\ref{table:DA_office_details} and~\ref{table:DA_visda_details} in which we finetune $s$ and $\epsilon$ to obtain the best performance. For a fair comparison, we use the same value of $s$ for TS-POT while the value of $\epsilon$ is fixed to 0. In terms of TS-UOT, we use the same set of $\tau$ and $\epsilon$ as that for computing m-UOT. Tables~\ref{table:DA_efficient_office_detail} and~\ref{table:DA_efficient_visda_detail} demonstrate the classification accuracy of the deep DA on Office-Home and VisDA datasets, respectively. Firstly, increasing the number of mini-batches does increase the performance of almost all mini-batch methods. One notable exception is the case of m-UOT on the Office-Home dataset where $k>1$ shows no improvement. Secondly, the improvement of the two-stage implementation tends to be more significant than that of the conventional implementation. Considering m-OT and TS-OT as an example, when $k$ goes up from 1 to 4, there are increases of $2.20$ and $1.91$ in the accuracy of m-OT on Office-Home and VisDA datasets, respectively. While the corresponding rises of TS-OT are $3.69$ and $10.8$, respectively. Thirdly, implemented using our proposed two-stage in Algorithm~\ref{alg:efficient_DA}, both TS-OT and TS-POT show a remarkable increase in the performance. TS-UOT, however, degrades the performance of m-UOT for all choices of $k$ on both datasets. The reason is that the transportation plan of UOT is denser than that of OT or POT. Note that our two-stage implementation is originally designed to boost the performance of m-OT and m-POT only. Finally, when implementing mini-batch partial transportation using the two-stage algorithm, TS-POT achieves the best accuracy of $71.20$ and $75.96$ on Office-Home and VisDA datasets, respectively.

\begin{table}[!t]
    \caption{Comparison between two mini-batch algorithms on the Office-Home dataset when increasing the number of mini-batches from 1 to 4. Each entry includes the average classification accuracy in the target domain over 3 runs. For m-POT and TS-POT, the value of $s$ is the number in the bracket next to the name.}
    \vskip 0.1in
    \label{table:DA_efficient_office_detail}
    \centering
    \scalebox{0.9}{
        \begin{tabular}{ccccccccccccccc}
        \toprule
        k & Method & A2C & A2P & A2R & C2A & C2P & C2R & P2A & P2C & P2R & R2A & R2C & R2P & Avg \\
        \midrule
        1 & m-OT & 49.50 & 66.22 & 73.41 & 57.38 & 64.50 & 67.17 & 54.95 & 47.07 & 73.33 & 64.51 & 53.72 & 76.88 & 62.39 \\
        & m-UOT & 55.41 & 74.79 & 80.71 & 65.65 & \textbf{74.40} & 74.98 & 64.72 & 53.14 & 79.92 & 74.85 & 60.08 & \textbf{83.40} & 70.17 \\
        & m-POT(0.6) & 55.49 & 72.90 & 80.65 & 64.54 & 73.47 & 75.55 & 62.56 & 52.66 & 79.86 & 72.27 & 59.24 & 82.84 & 69.34 \\
        & TS-OT & 50.97 & 67.34 & 73.14 & 56.52 & 63.24 & 67.40 & 54.70 & 47.03 & 73.45 & 64.28 & 54.53 & 76.68 & 62.44 \\
        & TS-UOT & 54.85 & 72.06 & 78.37 & 62.37 & 71.31 & 73.34 & 60.77 & 51.15 & 78.36 & 69.74 & 58.99 & 81.53 & 67.74 \\
        & TS-POT(0.6) & \textbf{56.28} & \textbf{76.00} & \textbf{81.27} & \textbf{68.51} & 73.44 & \textbf{76.40} & \textbf{67.15} & \textbf{54.51} & \textbf{80.65} & \textbf{75.21} & \textbf{60.74} & 83.38 & \textbf{71.13} \\
        \midrule
        2 & m-OT & 51.06 & 70.19 & 74.90 & 59.11 & 64.93 & 68.71 & 56.06 & 47.64 & 73.96 & 66.49 & 54.70 & 77.77 & 63.79 \\
        & m-UOT & 53.95 & 74.24 & 79.68 & 65.66 & \textbf{74.02} & 75.93 & 64.41 & 52.68 & 79.17 & 72.39 & 59.10 & 82.87 & 69.51 \\
        & m-POT(0.6) & 53.96 & 74.16 & 80.71 & 66.42 & 71.35 & 75.73 & 64.47 & 52.65 & 79.96 & 73.50 & 57.87 & 83.10 & 69.49 \\
        & TS-OT & 53.07 & 69.29 & 75.88 & 59.17 & 66.02 & 70.66 & 56.98 & 49.21 & 75.81 & 67.51 & 56.11 & 79.80 & 64.96 \\
        & TS-UOT & 55.98 & 72.97 & 79.86 & 64.48 & 73.14 & 75.72 & 62.78 & 54.02 & 78.79 & 70.91 & 59.54 & 82.97 & 69.26 \\
        & TS-POT(0.6) & \textbf{57.06} & \textbf{76.13} & \textbf{81.53} & \textbf{68.44} & 72.82 & \textbf{76.53} & \textbf{66.21} & \textbf{54.87} & \textbf{80.39} & \textbf{75.57} & \textbf{60.50} & \textbf{84.31} & \textbf{71.20} \\
        \midrule
        4 & m-OT & 51.75 & 70.01 & 75.79 & 59.60 & 66.46 & 70.07 & 57.60 & 47.88 & 75.29 & 66.82 & 55.71 & 78.11 & 64.59 \\
        & m-UOT & 54.96 & 74.06 & 79.57 & 66.12 & \textbf{73.76} & 75.80 & 64.73 & 52.82 & 79.46 & 72.75 & 58.51 & 82.62 & 69.60 \\
        & m-POT(0.6) & 56.14 & 74.72 & \textbf{81.27} & 65.90 & 72.42 & 75.34 & 64.30 & 52.93 & 79.98 & 73.83 & 58.49 & 83.34 & 69.89 \\
        & TS-OT & 53.89 & 71.01 & 77.13 & 59.82 & 69.20 & 71.95 & 59.18 & 51.17 & 76.54 & 66.46 & 56.97 & 80.19 & 66.13 \\
        & TS-UOT & 56.35 & 73.56 & 80.16 & 65.02 & 73.12 & 76.50 & 63.66 & 54.49 & 79.97 & 71.24 & 60.11 & 82.92 & 69.76 \\
        & TS-POT(0.6) & \textbf{56.93} & \textbf{76.17} & 81.05 & \textbf{67.92} & 72.90 & \textbf{76.49} & \textbf{66.48} & \textbf{55.70} & \textbf{80.43} & \textbf{75.72} & \textbf{60.24} & \textbf{84.12} & \textbf{71.18} \\
        \bottomrule
        \end{tabular}
    }
    \vskip -0.1in
\end{table}

\begin{table}[!t]
    \caption{Comparison between two mini-batch algorithms on the VisDA dataset when increasing the number of mini-batches from 1 to 4. Column ``All" indicates the accuracy when performing classification on all 12 classes. The following 12 columns represent the precision when classifying each class separately. The last column ``Avg" shows the average precision over 12 classes. Each entry includes the average value over 3 runs. For m-POT and TS-POT, the value of $s$ is the number in the bracket next to the name.}
    \vskip 0.1in
    \label{table:DA_efficient_visda_detail}
    \centering
    \scalebox{0.8}{
        \begin{tabular}{cccccccccccccccc}
        \toprule
        k & Method & All & plane & bcycl & bus & car & house & knife & mcycl & person & plant & sktbrd & train & truck & Avg \\
        \midrule
        1 & m-OT & 60.51 & 74.70 & 70.11 & 69.75 & 63.99 & 81.70 & 14.10 & 58.50 & 70.20 & 63.16 & 23.69 & 48.08 & 48.96 & 57.25 \\
        & m-UOT & 71.58 & 95.32 & \textbf{94.15} & 67.32 & 73.22 & 92.20 & 7.34 & 65.78 & 75.38 & 83.95 & 40.11 & 66.54 & 46.70 & 67.34 \\
        & m-POT(0.75) & 73.13 & 92.41 & 93.22 & 72.25 & 74.81 & \textbf{93.05} & \textbf{50.31} & 60.63 & 75.11 & 80.16 & 36.23 & 69.23 & 65.49 & 71.91 \\
        & TS-OT & 58.34 & 71.57 & 74.48 & 65.41 & 62.76 & 83.40 & 29.59 & 49.14 & 78.59 & 76.03 & 21.71 & 52.18 & 31.13 & 58.00 \\
        & TS-UOT & 69.76 & 84.77 & 86.89 & \textbf{73.15} & 70.59 & 87.94 & 10.85 & 60.44 & \textbf{80.61} & 80.65 & 30.82 & 63.46 & 52.49 & 65.22 \\
        & TS-POT(0.75) & \textbf{75.40} & \textbf{95.52} & 92.35 & 66.76 & \textbf{76.97} & 91.15 & 22.44 & \textbf{71.02} & 70.79 & \textbf{84.02} & \textbf{40.87} & \textbf{75.01} & \textbf{78.85} & \textbf{72.15} \\
        \midrule
        2 & m-OT & 61.83 & 68.92 & 81.26 & 67.97 & 72.33 & 89.85 & 24.80 & 63.28 & 68.45 & 52.87 & 26.44 & 50.31 & 47.22 & 59.48 \\
        & m-UOT & 71.22 & 91.80 & 93.02 & 72.47 & 73.15 & 92.36 & 7.35 & 68.39 & 73.67 & 76.86 & 36.74 & 62.63 & 48.73 & 66.43 \\
        & m-POT(0.75) & 72.84 & 90.10 & 93.61 & 68.46 & 73.04 & 91.51 & \textbf{77.50} & 60.38 & 75.28 & 79.12 & 34.12 & 70.70 & 59.76 & 72.80 \\
        & TS-OT & 66.10 & 80.24 & 82.90 & 68.14 & 73.93 & 87.10 & 13.05 & 64.59 & 68.61 & 75.89 & 26.08 & 50.74 & 49.11 & 61.70 \\
        & TS-UOT & 70.29 & 90.82 & 88.82 & 69.18 & 73.28 & 91.59 & 8.25 & 63.92 & \textbf{77.55} & 80.94 & \textbf{38.54} & 58.20 & 54.57 & 66.31 \\
        & TS-POT(0.75) & \textbf{75.28} & \textbf{94.31} & \textbf{95.10} & \textbf{73.56} & \textbf{76.20} & \textbf{92.55} & 70.83 & \textbf{67.48} & 68.26 & \textbf{83.99} & 35.62 & \textbf{73.00} & \textbf{74.27} & \textbf{75.43} \\
        \midrule
        4 & m-OT & 62.42 & 68.93 & 87.79 & 80.58 & 69.59 & 87.72 & 15.88 & 47.93 & 75.71 & 62.84 & 23.17 & 56.96 & 52.87 & 60.83 \\
        & m-UOT & 72.34 & 91.47 & 92.53 & 74.53 & 74.19 & 90.97 & 6.95 & 65.24 & 73.36 & 79.01 & 37.09 & 67.74 & 62.34 & 67.95 \\
        & m-POT(0.75) & 73.59 & 91.51 & 93.06 & 71.30 & 75.11 & \textbf{92.97} & \textbf{70.66} & 62.14 & 75.66 & 80.93 & 33.30 & 67.23 & 61.88 & 72.98 \\
        & TS-OT & 69.14 & 85.31 & 85.32 & \textbf{74.67} & 71.65 & 88.54 & 6.18 & 56.90 & \textbf{80.92} & 84.74 & 28.62 & 58.84 & 61.87 & 65.30 \\
        & TS-UOT & 70.91 & 86.59 & 87.42 & 73.12 & 70.81 & 90.83 & 8.80 & 62.99 & 76.36 & 81.80 & 36.61 & 61.16 & 63.29 & 66.65 \\
        & TS-POT(0.75) & \textbf{75.96} & \textbf{94.46} & \textbf{95.35} & 68.83 & \textbf{76.32} & 92.49 & 56.81 & \textbf{67.95} & 66.17 & \textbf{84.79} & \textbf{38.01} & \textbf{73.60} & \textbf{89.62} & \textbf{75.37} \\
        \bottomrule
        \end{tabular}
    }
    \vskip -0.1in
\end{table}

\subsection{Ablation studies on Domain Adaptation}
In this section, we first discuss how to reduce the misspecified matching problem and improve the performance of DA. Next, we verify the previous hypotheses by answering two questions: 1. How the number of mini-batches $k$ affect the choice of the fraction of masses $s$ for m-POT and the performance of mini-batch methods? 2. How the mini-batch size $m$ affect the choice of the fraction of masses $s$ for m-POT and the performance of mini-batch methods?

\textbf{How to mitigate the misspecified matching: } For m-POT, the misspecified matchings can be reduced when the mini-batch size $m$ increases. Namely, the probability of a pair of mini-batches that contains a global optimal matching increases when $m$ increases. However, this property does not hold for m-OT since m-OT can not detect global optimal matchings when they exist. On the other hand, the number of mini-batches $k$ cannot affect the behaviors of misspecified matchings. However, increasing $k$ helps to reduce variances of the stochastic gradient of mini-batch methods, which can improve the performance of DA. 

\begin{figure*}[t]
    \begin{center}
        \begin{tabular}{ccc}
            \widgraph{0.3\textwidth}{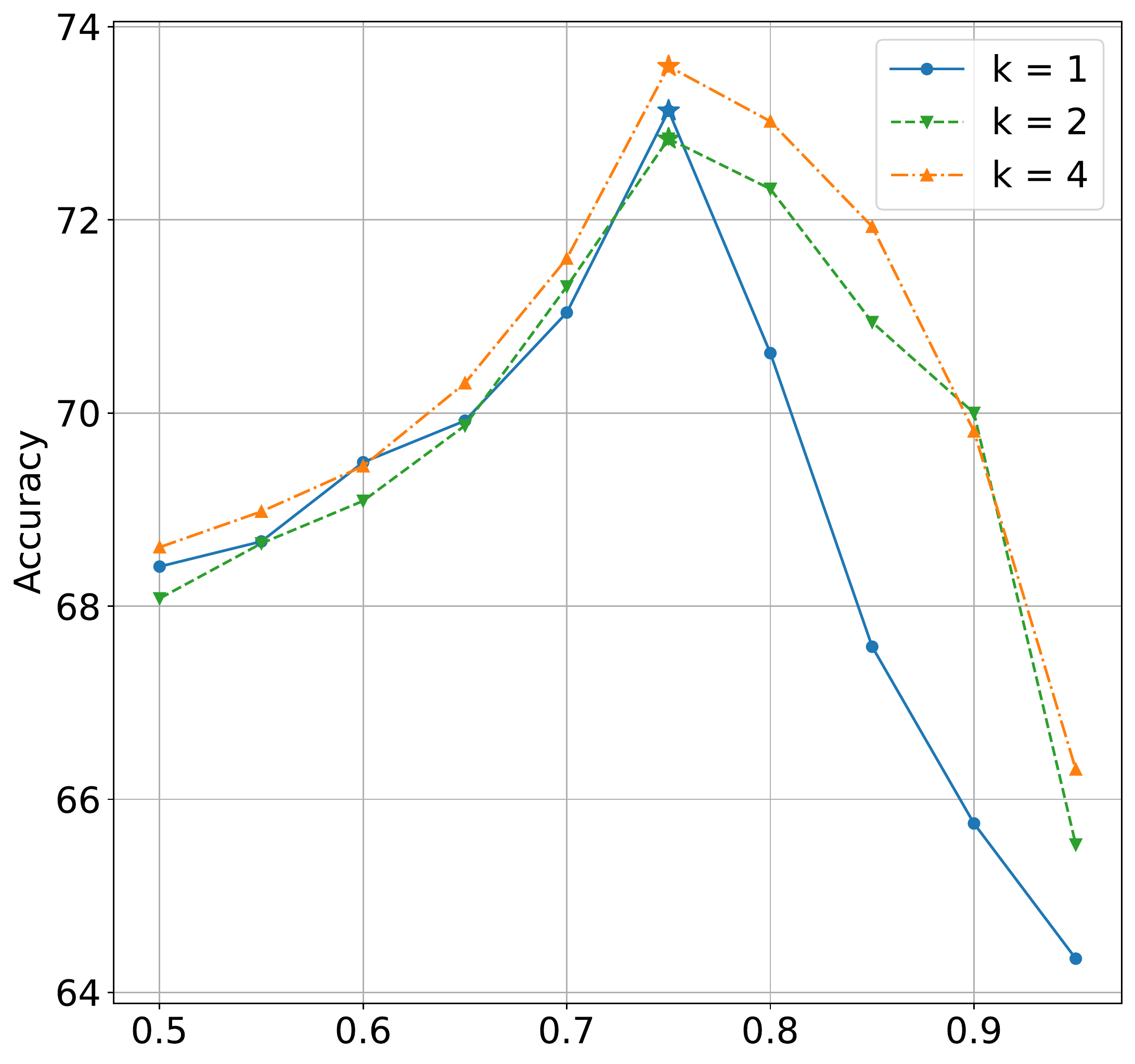} & \widgraph{0.3\textwidth}{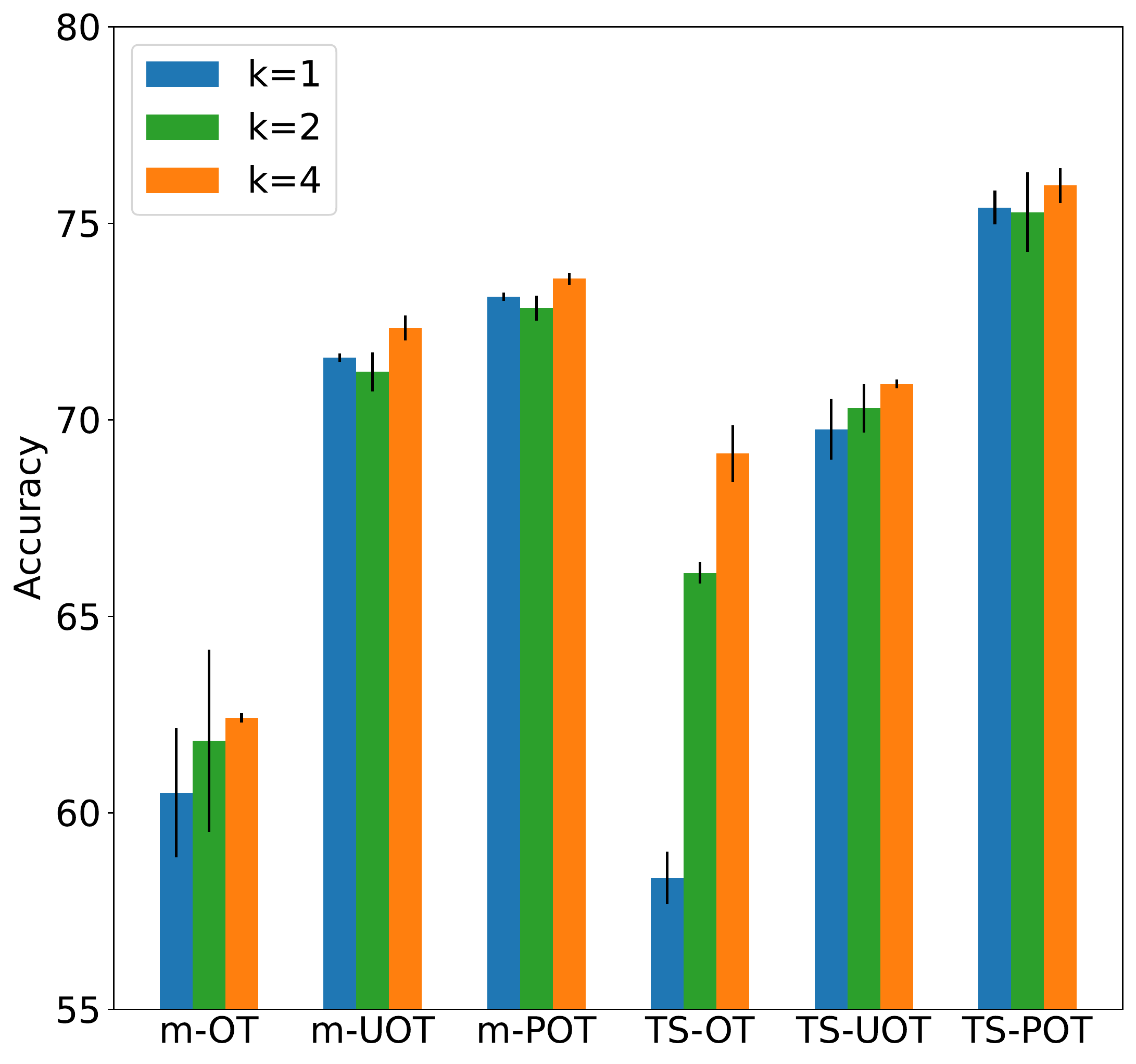} &
            \widgraph{0.3\textwidth}{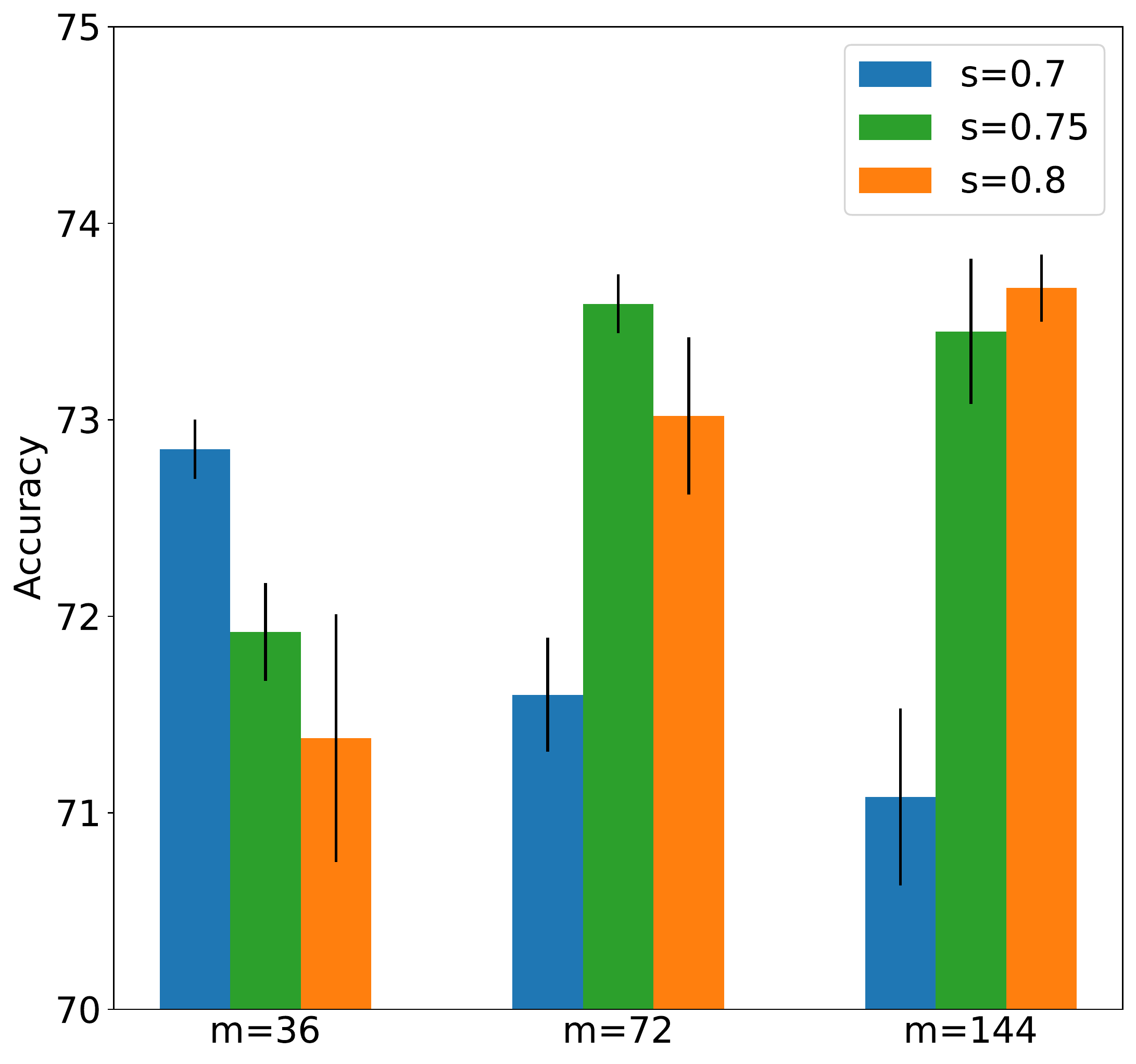}
        \end{tabular}
    \end{center}
    \vskip -0.1in
    \caption{
        \footnotesize{Ablation studies on the VisDA dataset. (Left) The effect of the number of mini-batches $k$ on the choice of the fraction of masses $s$ for m-POT. (Middle) The effect of the number of mini-batches $k$ on the performance of mini-batch methods. (Right) The effect of mini-batch size $m$ on the choice of the fraction of masses $s$ for m-POT.}} 
    \label{fig:ablation_studies_VisDA}
    \vskip -0.1in
\end{figure*} 

\textbf{Ablation studies on the number of mini-batches $k$: } In this experiment, we vary the number of mini-batches $k \in \{ 1, 2, 4\}$ while the mini-batch size $m$ is fixed to 72. The value of the fraction of masses $s$ is searched in the set $\{ 0.5, 0.55, \ldots, 0.95\}$. Figure~\ref{fig:ablation_studies_VisDA} (left) illustrates that the optimal mass is $0.75$ for all three choices of $k$ on the VisDA dataset. This is expected because increasing $k$ does not reduce the number of misspecified matchings. Next, we fix the fraction of masses $s$ to $0.75$ for m(TS)-POT. Figure~\ref{fig:ablation_studies_VisDA} (middle) shows the classification accuracies of different mini-batch methods when changing the number of mini-batches. Generally, increasing the number of mini-batches does increase the performance of almost all mini-batch methods. In other words, all methods achieve the best performance in the case of $k = 4$. When comparing between mini-batch methods, TS-POT yields better performance for all three choices of $k$. 



\textbf{Ablation studies on the mini-batch size $m$: } We run DA experiments with three different mini-batch sizes $m \in \{ 36, 72, 144 \}$. The fraction of masses $s$ is chosen from the set $\{ 0.7, 0.75, 0.8 \}$. Figure~\ref{fig:ablation_studies_VisDA} (right) illustrates the classification accuracy of m-POT and TS-POT for different combinations of $m$ and $s$. It can be seen that the optimal value of $s$ for $m = 36, 72, $ and $144$ are $0.7, 0.75,$ and $0.8$, respectively. The reason is increasing the mini-batch size $m$ reduces the number of misspecified matchings, thus increases the fraction of masses $s$. Next, we fix the fraction of masses $s$ for m(TS)-POT to its optimal value for each mini-batch size. Table~\ref{table:VisDA_change_m} demonstrates the performance of different mini-batch methods on VisDA when changing the mini-batch size. As can be seen from the table, increasing $m$ enhances the accuracy for both m-POT and TS-POT, which match our discussion above. Interestingly, other mini-batch methods including m-OT, m-UOT, TS-OT, and TS-UOT also experience the same effect. In addition, TS-POT consistently outperforms other methods for all choices of $m$.


\begin{table}[t]
    \caption{Performance of mini-batch methods on VisDA when varying the mini-batch size $m$. Experiments were run 3 times.}
    \vskip 0.1in
    \label{table:VisDA_change_m}
    \centering
    \scalebox{1.0}{
        \begin{tabular}{cccc}
        \toprule
        Method & $m = 36$ & $m = 72$ & $m = 144$ \\
        \midrule
        m-OT & 58.71 $\pm$ 3.21 & 62.42 $\pm$ 0.12 & 66.51 $\pm$ 3.72 \\
        m-UOT & 70.63 $\pm$ 0.30 & 72.34 $\pm$ 0.32 & 72.77 $\pm$ 0.25 \\
        m-POT (Ours) & 72.85 $\pm$ 0.15 & 73.59 $\pm$ 0.15 & 73.67 $\pm$ 0.17 \\
        TS-OT (Ours) & 63.01 $\pm$ 1.35 & 69.14 $\pm$ 0.72 & 70.66 $\pm$ 0.39 \\
        TS-UOT (Ours) & 70.27 $\pm$ 0.38 & 70.91 $\pm$ 0.11 & 72.23 $\pm$ 0.18 \\
        TS-POT (Ours) & \textbf{75.08 $\pm$ 0.37} & \textbf{75.96 $\pm$ 0.44} & \textbf{76.17 $\pm$ 0.42} \\
        \bottomrule
        \end{tabular}
    }
\end{table}

\begin{table}[t]
    \caption{Details of PDA results on the Office-Home dataset. Entries of the table are classification accuracies in the target domain. The fraction of mass $s$ of m-POT increases linearly from 0.01 to 0.325. Detailed settings are given in Appendix~\ref{subsec:setting_PDA}.}
    \vskip 0.1in
    \label{table:PDA_office_details}
    \centering
    \scalebox{0.9}{
        \begin{tabular}{ccccccccccccccc}
        \toprule
        Run & Method & A2C & A2P & A2R & C2A & C2P & C2R & P2A & P2C & P2R & R2A & R2C & R2P & Avg \\
        \midrule
        1 & BA3US & 60.00 & 75.29 & \textbf{88.46} & 72.91 & 71.65 & \textbf{85.75} & 74.29 & 59.76 & 85.53 & 79.80 & 63.70 & 85.99 & 75.26 \\
        & mOT & 48.06 & 66.11 & 77.64 & 58.95 & 56.98 & 65.93 & 58.49 & 44.96 & 74.05 & 67.59 & 50.15 & 74.57 & 61.96 \\
        & mUOT & 61.37 & \textbf{81.68} & 84.93 & 75.48 & 72.38 & 80.51 & 75.30 & 62.21 & 87.19 & \textbf{80.62} & 67.52 & 84.31 & 76.13 \\
        & mPOT & \textbf{63.34} & 79.61 & 86.58 & \textbf{75.67} & \textbf{77.31} & 83.82 & \textbf{77.04} & \textbf{64.42} & \textbf{88.90} & \textbf{80.62} & \textbf{68.60} & \textbf{86.67} & \textbf{77.72} \\
        \midrule
        2 & BA3US & 58.27 & 80.50 & \textbf{88.02} & 71.90 & 74.40 & 81.45 & 72.54 & 59.94 & 85.37 & 78.88 & 61.85 & 85.77 & 74.91 \\
        & mOT & 48.66 & 65.88 & 77.31 & 59.14 & 58.04 & 66.59 & 58.13 & 45.43 & 73.83 & 68.32 & 49.85 & 74.17 & 62.11 \\
        & mUOT & 62.51 & 80.73 & 85.48 & 76.31 & 73.11 & 78.41 & 74.01 & 62.81 & 84.93 & 80.81 & 66.45 & 85.15 & 75.89 \\
        & mPOT & \textbf{65.97} & \textbf{83.42} & 87.63 & \textbf{78.05} & \textbf{78.66} & \textbf{82.61} & \textbf{76.13} & \textbf{64.66} & \textbf{86.80} & \textbf{82.37} & \textbf{67.58} & \textbf{86.89} & \textbf{78.40} \\
        \midrule
        3 & BA3US & 59.76 & \textbf{80.39} & \textbf{88.79} & 73.28 & 70.98 & 83.43 & 72.73 & 60.90 & 86.86 & 78.70 & 63.46 & 85.94 & 75.44 \\
        & mOT & 47.28 & 65.99 & 77.47 & 59.60 & 58.54 & 67.20 & 58.68 & 45.37 & 74.43 & 68.32 & 49.67 & 74.01 & 62.21 \\
        & mUOT & 60.72 & 78.60 & 85.59 & 75.02 & 73.17 & 80.45 & 74.38 & 60.84 & \textbf{87.36} & 80.90 & 68.18 & 85.21 & 75.87 \\
        & mPOT & \textbf{64.48} & 78.82 & 87.30 & \textbf{75.57} & \textbf{76.86} & \textbf{84.32} & \textbf{78.05} & \textbf{62.15} & 87.19 & \textbf{81.27} & \textbf{69.31} & \textbf{88.57} & \textbf{77.82} \\
        \bottomrule
        \end{tabular}
    }
\end{table}

\subsection{Partial Domain Adaptation Experiments}
\label{subsec:addexp_PDA}
We first want to emphasize that we have used the best parameter settings for m-UOT that are reported in \cite{fatras2021unbalanced} on PDA experiments. We reproduce and compare the performance of our proposed method m-POT with two mini-batch methods (m-OT and m-UOT) and one non-OT baseline (BA3US~\cite{liang2020balanced}). Table~\ref{table:PDA_office_details} demonstrates the performance of different methods on the Office-Home dataset over 3 runs. It can be seen clearly that m-POT achieves state-of-the-art classification accuracy on at least three-fourths of tasks with some considerable gaps. On the C2P scenario, for instance, m-POT leads to an increase of almost $5\%$ compared to its competitors in the first and second runs. Based on the average classification accuracy, m-POT still consistently yields the best performance in every run. 

\subsection{Deep Generative Model Experiments}
\label{subsec:exp_dgm}
OT has been widely utilized to measure the discrepancy between a parametric distribution and real data distribution in a generative model. In this experiment, m-OT and m-POT are employed for the deep generative model application on CIFAR10~\cite{krizhevsky2009learning} and CelebA~\cite{liu2015faceattributes} datasets. The experiment settings including neural network architectures and the used parameters are described in Appendix~\ref{subsec:setting_GAN}. We report the best performing checkpoint, which is measured by FID score, along with their generated images from random noise.

\textbf{Comparison between m-OT and m-POT:} The quantitative results are described in Table~\ref{table:GM_details} where the number in bracket indicates the fraction of masses $s$. By evaluating the FID score, we observe that m-POT yields better generators than m-OT on both CIFAR10 and CelebA. In particular, m-POT leads to a slight improvement of $1.13$ over m-OT on the CIFAR10 dataset with $m = 100$ while it yields a larger difference of $7.6$ on the CelebA dataset with $m = 200$. On the effect of the mini-batch size, when $m = 100$, there is only one choice of $s$ that leads to a lower FID score. As the mini-batch size increases, we have more options to choose a good value of $s$. Specifically, m-POT yields better performance in $2$ and $6$ different values of $s$ for CIFAR10 and CelebA datasets, respectively. The qualitative results (randomly generated images) of m-OT and m-POT are given in Figures~\ref{fig:DGM_Cifar10} and~\ref{fig:DGM_CelebA}. From these images, we can conclude that m-POT produces more realistic generated images than m-OT. Thus, m-POT should be utilized as an alternative training loss for deep generative models such as those in~\cite{tolstikhin2018wasserstein,genevay2018learning,salimans2018improving}. 

\begin{table*}[!t]
    \caption{FID scores of generators that are trained with m-OT and m-POT on CIFAR10 and CelebA datasets (smaller is better). The number of mini-batches $k$ is set to 2. For m-POT, the fraction of masses $s$ is the number in the bracket next to the name.}
    \vskip 0.1in
    \label{table:GM_details}
    \centering
    \scalebox{0.75}{
        \begin{tabular}{cccccccccccc}
        \toprule
        Dataset & m & m-OT & m-POT(0.7) &  m-POT(0.75) & m-POT(0.8) &  m-POT(0.85) & m-POT(0.9) &  m-POT(0.95) & m-POT(0.98) &  m-POT(0.99) \\
        \midrule
        CIFAR10 & 100 & 67.90 & 74.97 & 71.34 & 74.36 & 69.34 & 69.50 & 68.30 & \textbf{66.77} & 69.85 \\
        & 200 & 66.42 & 70.49 & 67.95 & 67.53 & 69.72 & 70.31 & \textbf{66.29} & 67.98 & 66.34 \\
        \midrule
        CelebA & 100 & 54.16 & 60.92 & 58.86 & 67.71 & 65.67 & \textbf{53.76} & 61.53 & 59.3 & 55.13 \\
        & 200 & 56.85 & \textbf{49.25} & 53.95 & 58.34 & 53.36 & 52.52 & 55.67 & 49.76 & 59.32 \\
        \bottomrule
        \end{tabular}
    }
    \vskip -0.1in
\end{table*}
\begin{figure*}[!t]
    \begin{center}
        \begin{tabular}{c}
        \widgraph{0.9\textwidth}{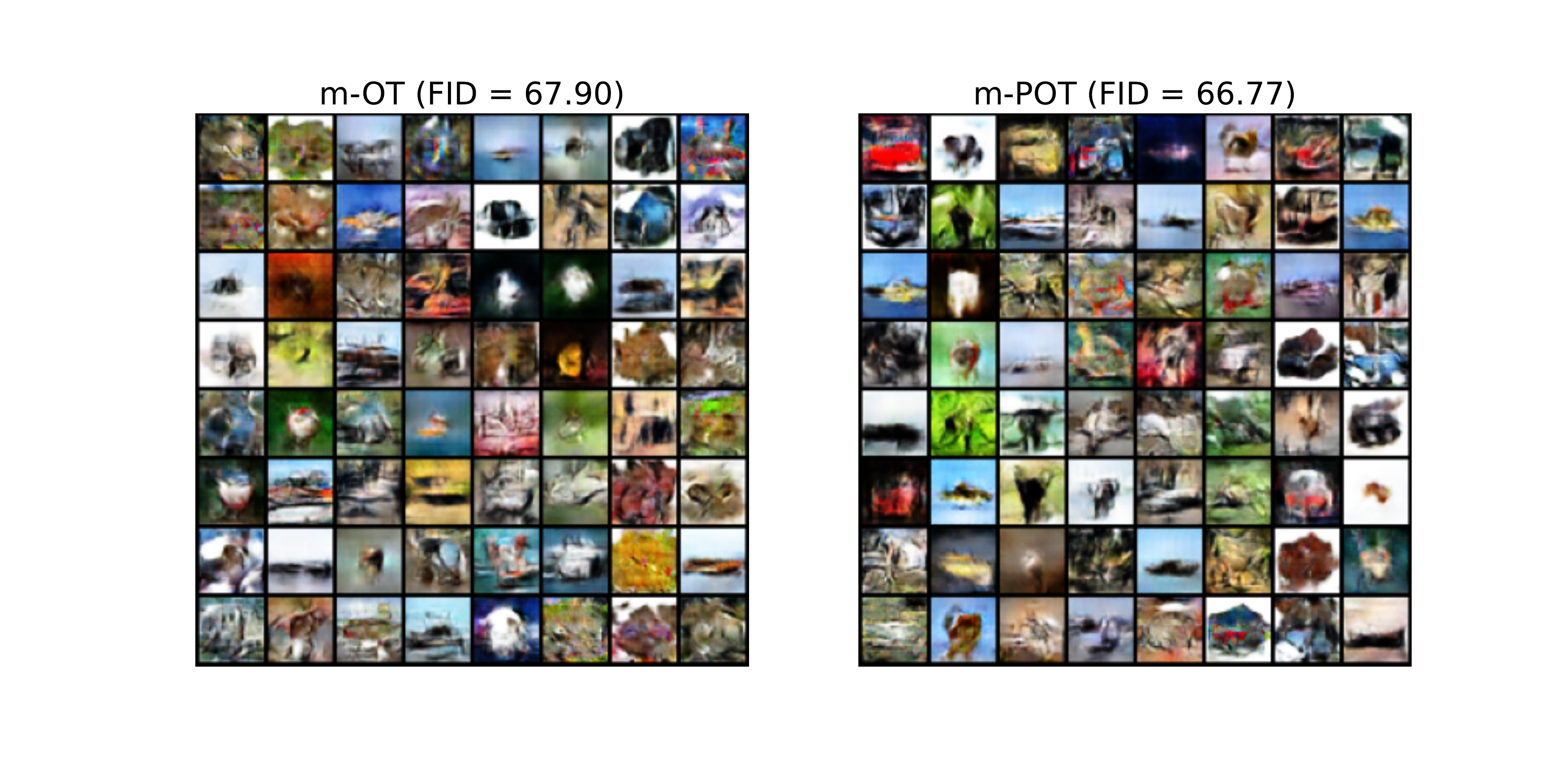} 
        \end{tabular}
    \end{center}
    \vskip -0.5in
    \caption{
        \footnotesize{CIFAR10 generated images from m-OT and m-POT (best choice of $s$), (k,m) = (2,100). The FID scores are reported in the title of the images.}
    }
    \label{fig:DGM_Cifar10}
    \vskip -0.25in
\end{figure*}

\begin{figure}[t!]
    \begin{center}
        \begin{tabular}{c}
        \widgraph{0.9\textwidth}{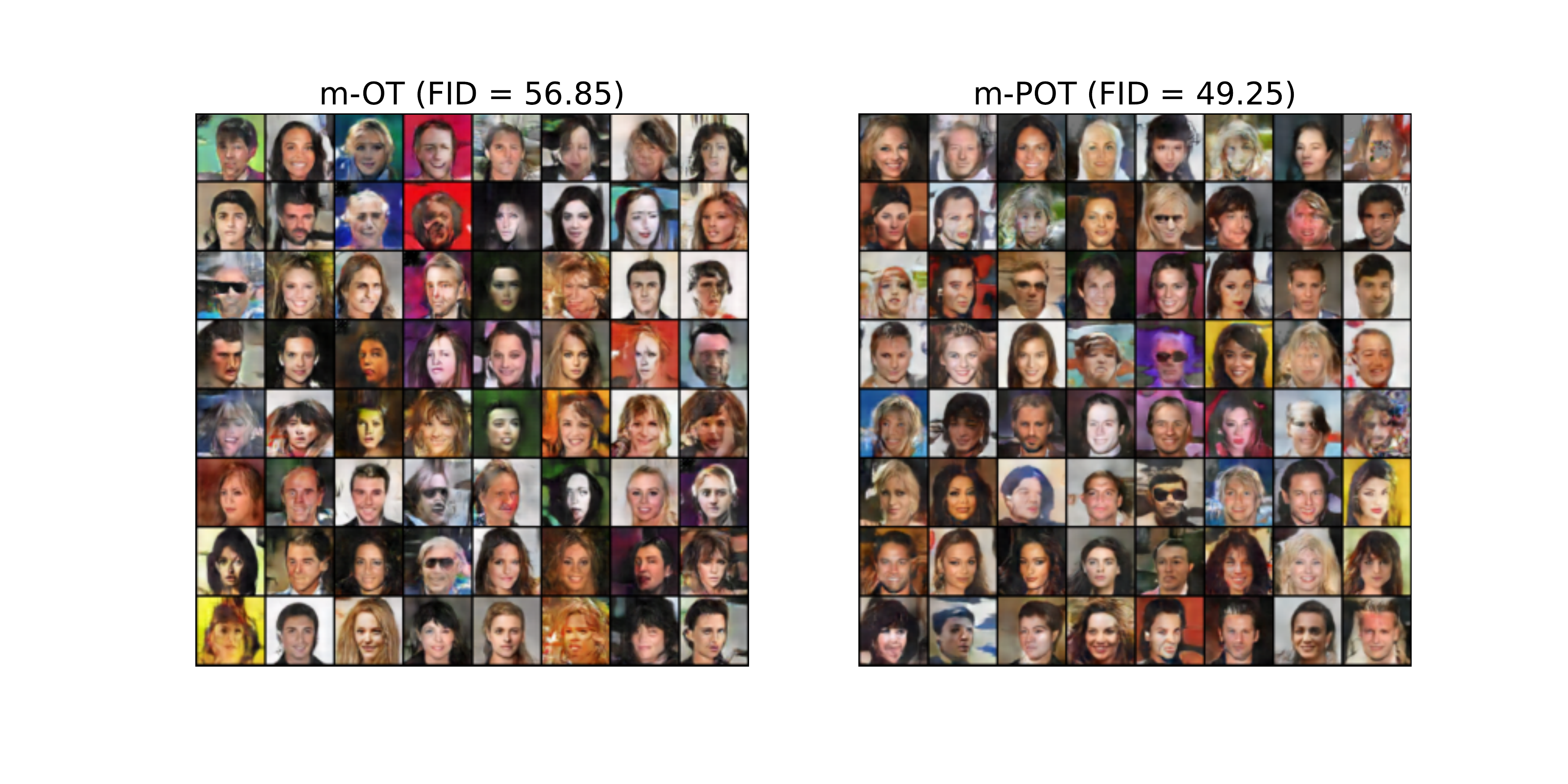} 
        \end{tabular}
    \end{center}
    \vskip -0.5in
    \caption{
        \footnotesize{CelebA generated images from m-OT and m-POT (best choice of $s$), (k,m) = (2,200). The FID scores are reported in the title of the images.}
    } 
    \label{fig:DGM_CelebA}
\end{figure}

\textbf{Applications of m-UOT: } We have not been able to achieve a good enough setting for m-UOT (a good enough setting means a setting that can provide a generative model that can generate reasonable images). To our best knowledge, there are not any generative models that are implemented with m-UOT, hence, we leave this comparison to future works. 

\begin{figure*}[!t] 
\begin{center}

  \begin{tabular}{c}
\widgraph{0.95\textwidth}{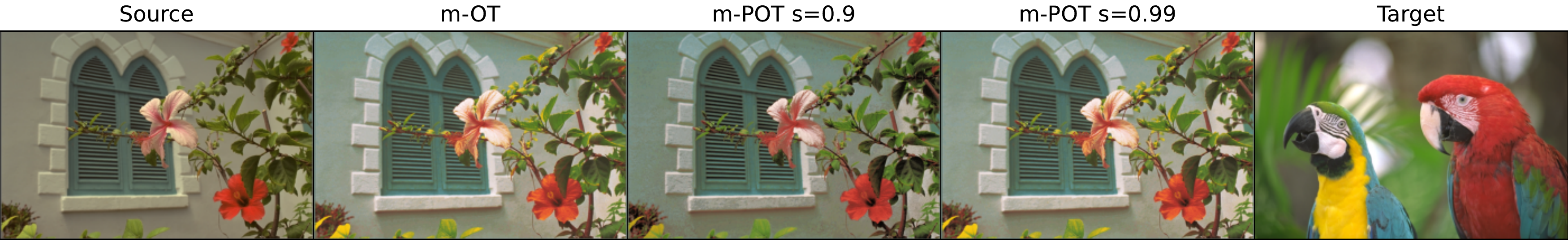}
\\
\widgraph{0.95\textwidth}{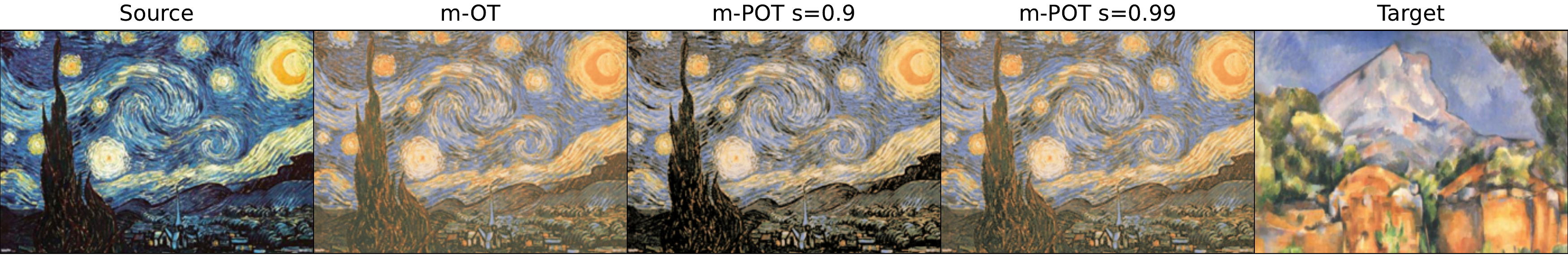} 
  \end{tabular}
  \end{center}
  \vskip -0.1in
  \caption{
  \footnotesize{Color Transfer results from m-OT, m-POT with $k=10000, m=100$. The leftmost image is the source image and the rightmost image is the target image. In between images transferred images with the corresponding name of the method on top.
}
} 
  \label{fig:colortransfer_main}
  \vskip -0.1in
\end{figure*}

\subsection{Color Transfer}
\label{subsec:color_transfer}
 
In this application, we run m-OT and m-POT to transfer the color of images on two domains: arts and natural images. We show some transferred images from the m-OT and m-POT with two values of $s$ ($0.9$ and $0.99$) in Figure~\ref{fig:colortransfer_main}. We observe that m-POT with the right choice of parameter $s$ generates more beautiful images than the m-OT, namely, m-POT ignores some color is too different between two images. Also, m-POT can work well on both two types of image domains. For m-UOT, We have tried to run m-UOT in this application, however, we have not found the setting that UOT can provide reasonable results yet.

\subsection{Gradient Flow}
\label{subsec:gflow}

\begin{figure*}[!t]
    \begin{center}
    \begin{tabular}{c}
        \widgraph{0.95\textwidth}{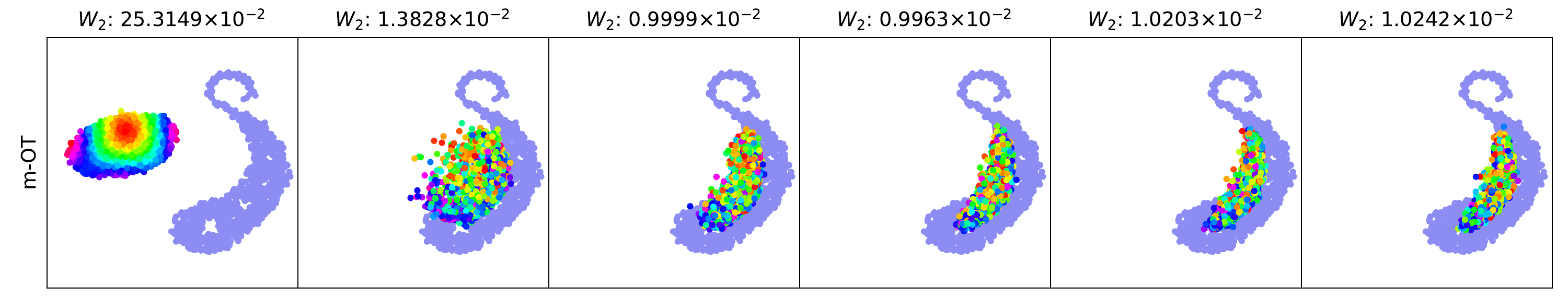} \\
        \widgraph{0.95\textwidth}{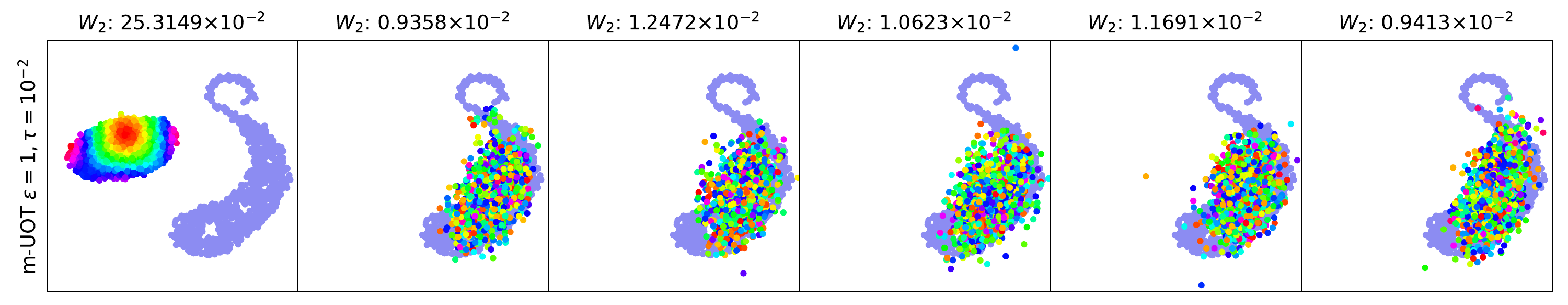} \\
        \widgraph{0.95\textwidth}{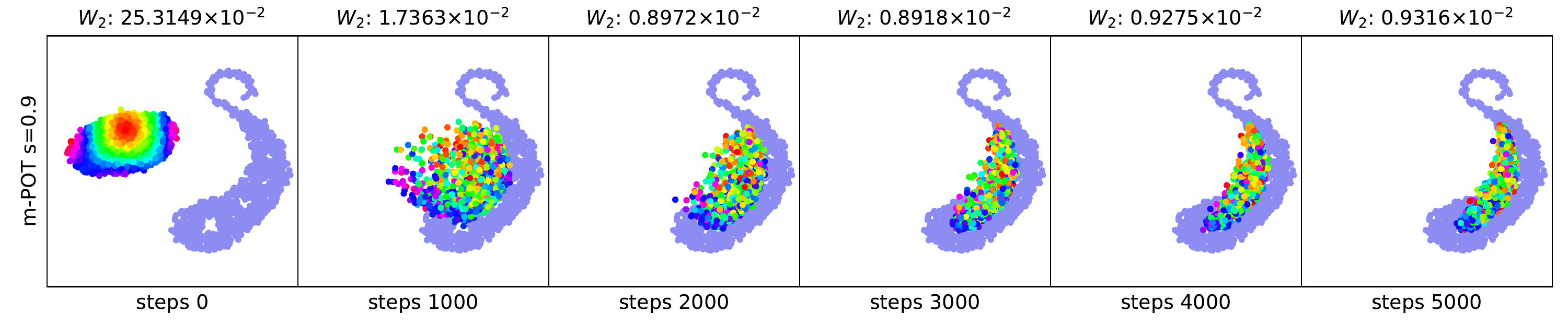} \\
        \widgraph{0.95\textwidth}{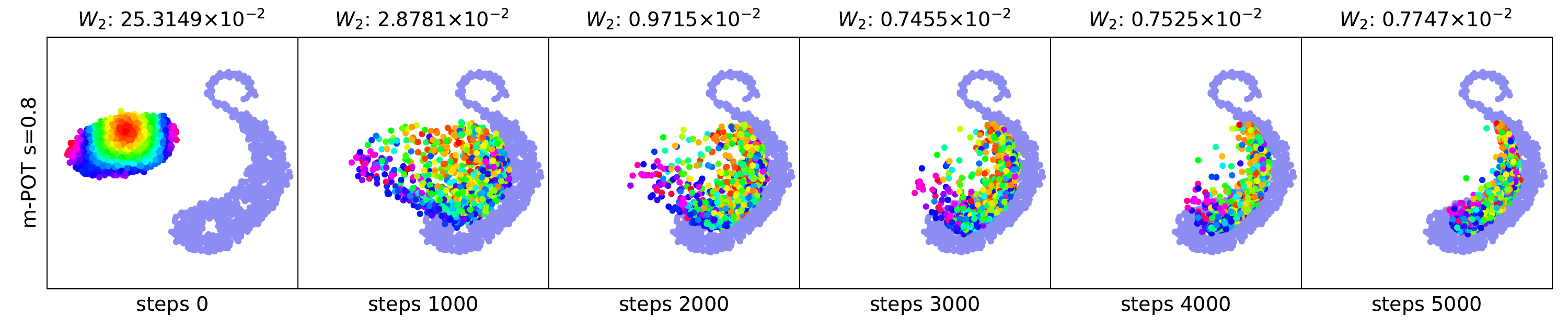}    
    \end{tabular}
    \end{center}
    \vskip -0.1in
    \caption{
        \footnotesize{Gradient flows from different mini-batch methods. The Wasserstein-2 distance between two measures in each step is reported at the top of the corresponding image.}
    } 
    \label{fig:gradientflow_main}
    \vskip -0.1in
\end{figure*}
 
We compare m-OT, m-UOT, and m-POT in a toy example as in \cite{feydy2019interpolating} and present our results in Figure~\ref{fig:gradientflow_main}. In short, we want to move the colorful empirical measure to the "S-shape" measure. Each measure has 1000 support points. Here, we choose $(k,m)=(4,4)$ and learning rate is set to $0.001$. We use entropic regularization \cite{chizat2018scaling} to solve m-UOT, we choose the best entropic regularization parameter $\epsilon \in \{0.1, 1, 2, 5, 10 \}$ and marginal relaxation parameter $\tau \in \{0.001, 0.01, 0.1, 1, 2, 5, 10\}$. From the figure, we observe that m-POT yields better flows than both m-OT and m-UOT, namely, the final Wasserstein-2 score of m-POT is the lowest. Interestingly, we see that reducing the amount of masses $s$ in m-POT from $0.9$ to $0.8$ improves considerably the result. Here, we observe that although lower values of $s$ provide a better generative model at the end, it can make the generative model learn slower. It suggests that $s$ should be changed adaptively in training processes. We will leave this question for future works.




    

\subsection{Equivalent Optimal Transport problem}
\label{subsec:equi_OT}

\begin{table}[t!]
    \parbox{.55\linewidth}{
        \caption{DA performance on digits datasets for m-OT and m-UOT with one more sample in a mini-batch. Experiments were run 3 times.}
        \vskip 0.1in
        \label{table:DA_digits_equi}
        \centering
        \scalebox{0.7}{
            \begin{tabular}{ccccc}
            \toprule
            Method & SVHN to MNIST & USPS to MNIST & MNIST to USPS & Avg \\
            \midrule
            m-OT(m=501) & 94.05 $\pm$ 0.05 & 96.70 $\pm$ 0.29 & 87.26 $\pm$ 0.64 & 92.67 \\
            m-UOT(m=501) & 98.82 $\pm$ 0.11 & 98.56 $\pm$ 0.18 & 95.72 $\pm$ 0.10 & 97.70 \\
            m-POT(m=500) & \textbf{98.98 $\pm$ 0.08} & \textbf{98.63 $\pm$ 0.13} & \textbf{96.04 $\pm$ 0.2} & \textbf{97.88} \\
            \bottomrule
            \end{tabular}
        }
    }
    \hfill
    \parbox{.4\linewidth}{
        \caption{DA performance on the VisDA dataset for m-OT and m-UOT with one more sample in a mini-batch. Experiments were run 3 times.}
        \vskip 0.1in
        \label{table:DA_visda_equi}
        \centering
        \scalebox{0.7}{
            \begin{tabular}{cc}
            \toprule
            Method & Accuracy \\
            \midrule
            m-OT(m=73) & 57.90 $\pm$ 2.81 \\
            m-UOT(m=73) & 71.44 $\pm$ 0.46 \\
            m-POT(m=72) & \textbf{72.72 $\pm$ 0.30} \\
            \bottomrule
            \end{tabular}
        }
    }
\end{table}

As discussed in Section~\ref{subsec:mPOT}, m-POT with mini-batch  size $m$ is equivalent to m-OT with mini-batch size $m+1$. Therefore, we run m-OT and m-UOT with the increasing mini-batch size to compare the performance. The classification accuracies on target datasets on the deep DA are reported in Tables~\ref{table:DA_digits_equi}-\ref{table:DA_office_equi}. The results for the deep generative model can be found in Table~\ref{table:GM_details_equi}. It can be observed clearly that m-POT still outperforms the corresponding version of m-OT on both deep domain adaptation and deep generative model applications. In addition, m-POT consistently yields higher performance than m-UOT on all datasets. For the deep generative model, m-POT leads to a significant drop of 11.79 in the FID score compared to m-OT when $m=100$.

\subsection{Timing information}
\label{subsec:time}
The following timing benchmarks for deep domain adaptation and deep generative model experiments are done on an NVIDIA Tesla V100 GPU.

\begin{table}[t!]
    \caption{DA performance of equivalent problems on the Office-Home dataset for m-OT and m-UOT with one more sample in a mini-batch. Experiments were run 3 times.}
    \vskip 0.1in
    \label{table:DA_office_equi}
    \centering
    \scalebox{0.8}{
        \begin{tabular}{cccccccccccccc}
        \toprule
        Method & A2C & A2P & A2R & C2A & C2P & C2R & P2A & P2C & P2R & R2A & R2C & R2P & Avg \\
        \midrule
        m-OT(m=66) & 49.92 & 68.63 & 73.81 & 57.78 & 64.23 & 67.36 & 55.83 & 45.93 & 73.46 & 64.83 & 53.86 & 77.03 & 62.72 \\
        & $\pm$ 0.26 & $\pm$ 0.80 & $\pm$ 0.25 & $\pm$ 0.52 & $\pm$ 0.26 & $\pm$ 0.28 & $\pm$ 0.12 & $\pm$ 0.55 & $\pm$ 0.09 & $\pm$ 0.33 & $\pm$ 0.60 & $\pm$ 0.27 & \\ 
        m-UOT(m=66) & 55.22 & \textbf{74.28} & 80.74 & 64.92 & 73.93 & 74.95 & \textbf{65.17} & 52.92 & 80.11 & 74.14 & 58.99 & 83.18 & 69.88 \\
        & 0.24 & $\pm$ 0.26 & $\pm$  0.07 & $\pm$ 0.28 & $\pm$ 0.31 & $\pm$ 0.08 & $\pm$ 0.20 & $\pm$ 0.10 & $\pm$ 0.01 & $\pm$ 0.18 & $\pm$ 0.11 & $\pm$ 0.18 & \\
        m-POT(m=65) & \textbf{55.65} & 73.80 & \textbf{80.76} & \textbf{66.34} & \textbf{74.88} & \textbf{76.16} & 64.46 & \textbf{53.38} & \textbf{80.60} & \textbf{74.55} & \textbf{59.71} & \textbf{83.81} & \textbf{70.34} \\
        & $\pm$ 0.45  & $\pm$ 0.07  & $\pm$ 0.09  & $\pm$ 0.13  & $\pm$ 0.08  & $\pm$ 0.21  & $\pm$ 0.08  & $\pm$ 0.18  & $\pm$ 0.03  & $\pm$ 0.13  & $\pm$ 0.10  & $\pm$ 0.07 & \\
        \bottomrule
        \end{tabular}
    }
\end{table}

\begin{table}[t!]
    \caption{FID scores (smaller is better) of m-OT with one more sample in a mini-batch on CIFAR10 and CelebA datasets.}
    \vskip 0.1in
    \label{table:GM_details_equi}
    \centering
    \scalebox{0.95}{
        \begin{tabular}{ccccccc}
        \toprule
        Dataset  & \multicolumn{3}{c}{CIFAR10} & \multicolumn{3}{c}{CelebA} \\
        \cmidrule(lr){2-7}
        & m-OT (m+1) & m-POT(m) & Improvement & m-OT(m+1) & m-POT(m) & Improvement \\
        \midrule
        m = 100 & 69.06 & \textbf{66.77} & \textbf{2.29 }& 65.55 &\textbf{ 53.76} & \textbf{11.79} \\
        m = 200 & 66.35 & \textbf{66.29} & \textbf{0.06} & 50.64 & \textbf{49.25} & \textbf{1.39} \\
        \bottomrule
        \end{tabular}
    }
\end{table}

\textbf{Deep domain adaptation: } Figure~\ref{fig:deepDA_timing} illustrates the running speed of different mini-batch methods on digits datasets. Due to the simplicity of the problem, m-OT is the fastest method in all experiments, followed by m-UOT. m-POT is generally slower than the other two methods, this is consistent with our analysis in Section~\ref{subsec:mPOT}. An interesting result is that the running speed of m-POT decreases as the fraction of masses $s$ increases. When $s$ approaches 1, the m-POT problem also approaches the m-OT problem, thus it becomes easier to compute. Table~\ref{table:two_stage_timing} describes the average time to run one iteration on Office-Home and VisDA datasets. In terms of the mini-batch loss, OT is consistently the fastest method on these two large datasets. Still behind OT, POT, however, runs slightly faster than UOT in this case. Compared to the old implementation, the new implementation which solves a large transportation problem on the CPU level requires almost twice as much time as the conventional implementation.

\textbf{Deep generative model: } The computational speed of the deep generative model is illustrated in Table~\ref{table:DGM_timing}. Both m-OT and m-POT share the same speed when the mini-batch size $m$ is set to $100$. For $m=200$, a training epoch of m-OT consumes 7 and 10 seconds fewer than that of m-POT on CIFAR10 and CelebA datasets, respectively. 

\begin{figure*}[!t] 
    \begin{center}
    \begin{tabular}{c}
        \widgraph{0.95\textwidth}{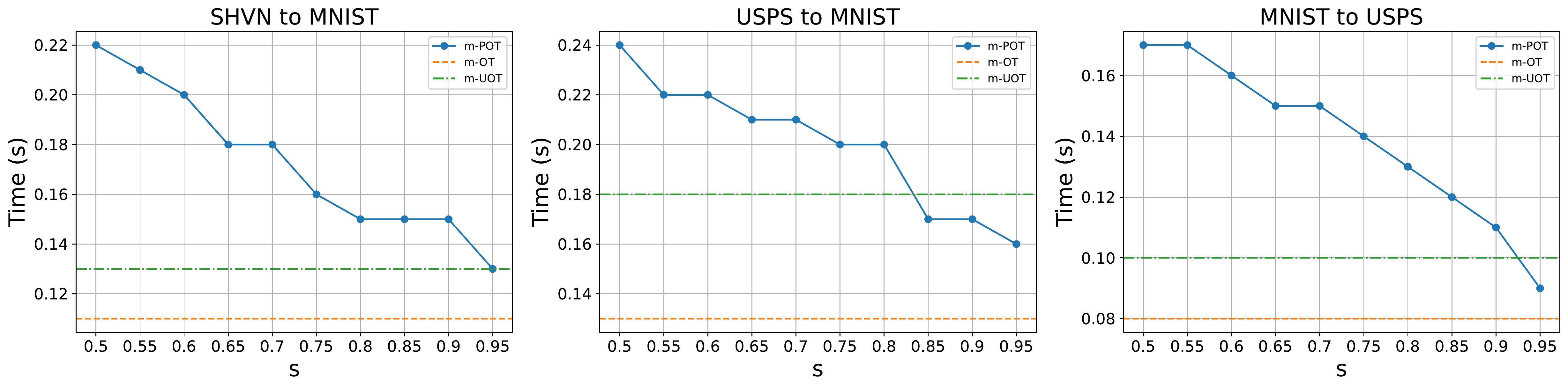}
        \\
        \widgraph{0.95\textwidth}{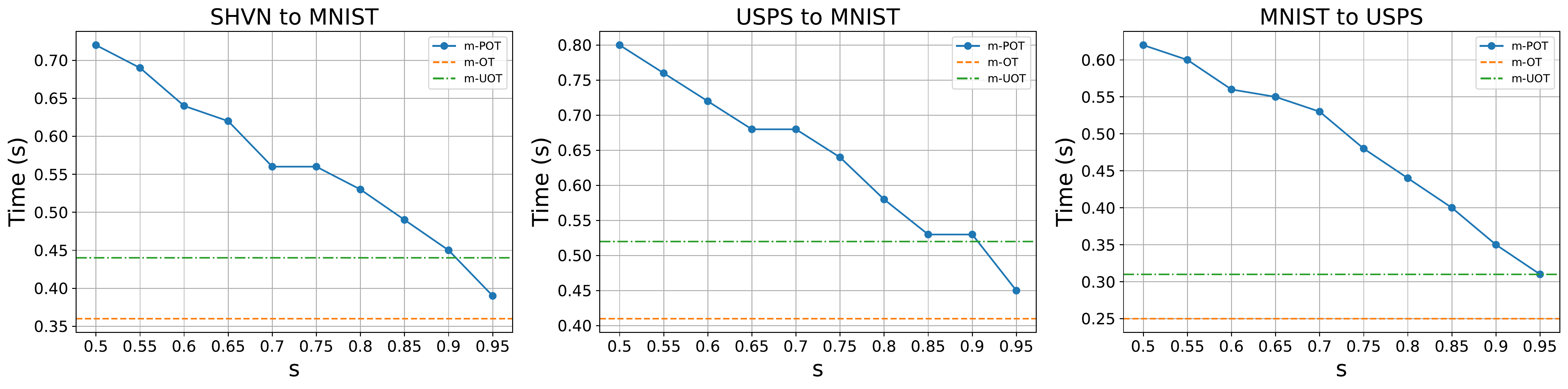} 
    \end{tabular}
    \end{center}
    \vskip -0.1in
    \caption{
        \footnotesize{The average time (in seconds) to run one iteration on the deep DA on digits datasets. All experiments are run with a mini-batch size $m=500$. $k$ is set to $1$ in the first row and $2$ in the second row.}
    } 
      \label{fig:deepDA_timing}
      \vskip -0.1in
\end{figure*}

\begin{table}[!t]
    \caption{The average time (in seconds) to run one iteration on the deep DA experiments. The number of mini-batches varies from 1 to 4. The fraction of masses $s$ for m-POT and TS-POT are set to 0.6 and 0.75 on Office-Home and VisDA datasets, respectively.}
    \vskip 0.1in
    \label{table:two_stage_timing}
    \centering
    \scalebox{1.0}{
        \begin{tabular}{ccccc|ccc}
        \toprule
        Dataset & $k$ & m-OT & m-UOT & m-POT & TS-OT & TS-UOT & TS-POT \\
        \midrule
        Office-Home & 1 & 0.39 & 0.44 & 0.42 & 0.39 & 0.44 & 0.42 \\
        & 2 & 0.71 & 0.82 & 0.71 & 1.20 & 1.26 & 1.21 \\ 
        & 4 & 1.36 & 1.52 & 1.37 & 2.48 & 2.54 & 2.50 \\
        \midrule
        VisDA & 1 & 0.37 & 0.42 & 0.38 & 0.37 & 0.42 & 0.38 \\
        & 2 & 0.71 & 0.79 & 0.71 & 1.25 & 1.39 & 1.25 \\ 
        & 4 & 1.36 & 1.50 & 1.37 & 2.49 & 2.66 & 2.50 \\
        \bottomrule
        \end{tabular}
    }
\end{table}

\begin{table}[t!]
    \caption{The average number of iterations per second on the deep generative model experiments. The fraction of masses $s$ is set to the best performing value in Table~\ref{table:GM_details}.}
    \vskip 0.1in
    \label{table:DGM_timing}
    \centering
    \scalebox{1.0}{
        \begin{tabular}{ccccccc}
        \toprule
        Dataset & k & m & Method & Iterations per epoch & Seconds per epoch & Iterations per second \\
        \midrule
        CIFAR10 & 2 & 100 & m-OT & 250 & 25 & 10.00 \\
        & & & m-POT & 250 & 25 & 10.00 \\
        \cmidrule(lr){3-7}
        & & 200 & m-OT & 125 & 23 & 5.43 \\
        & & & m-POT & 125 & 30 & 4.17 \\
        \midrule
        CelebA & 2 & 100 & m-OT & 814 & 295 & 2.76 \\
        & & & m-POT & 814 & 295 & 2.76 \\
        \cmidrule(lr){3-7}
        & & 200 & m-OT & 407 & 277 & 1.47 \\
        & & & m-POT & 407 & 287 & 1.42 \\
        \bottomrule
        \end{tabular}
    }
\end{table}


\section{Experimental Settings}
\label{sec:setting}
In this section, we provide details on datasets, architectures, and training procedures for deep domain adaptation, partial domain adaptation, and deep generative model experiments. For solving different types of OT problems, we use the POT package~\cite{flamary2021pot}. Notice that we have used the best values of $\tau$ and $\epsilon$ that are reported in the original paper of m-UOT~\cite{fatras2021unbalanced} for computing m-UOT in each experiment.

\subsection{Deep Domain Adaptation}
\label{subsec:setting_DA}

\textbf{Datasets:} We start with digits datasets. Following the evaluation protocol of DeepJDOT~\cite{damodaran2018deepjdot}, we experiment on three adaptation scenarios: SHVN to MNIST, USPS to MNIST, MNIST to USPS. The USPS dataset~\cite{hull1994database} consists of $7,291$ training and $2,007$ testing images, each one is a $16 \times 16$ grayscale handwritten image. The MNIST dataset~\cite{lecun1998gradient} contains $60,000$ training and $10,000$ testing grayscale images of size $28 \times 28$. The SVHN dataset~\cite{netzer2011reading} contains house numbers extracted from Google Street View images. This dataset has $73,212$ training images, and $26,032$ testing RGB images of size $32 \times 32$. Next, we consider the Office-Home dataset~\cite{venkateswara2017deep}, which is more difficult for domain adaptation. This dataset contains $15,500$ images in $65$ categories from four different domains: artistic (A) paintings, clipart (C), product (P), and real-world (R) images. We evaluate all methods in $12$ possible adaptation scenarios. Finally, VisDA-2017~\cite{peng2017visda} is a large-scale dataset for domain adaptation from synthetic to real images. VisDA consists of $152,397$ source (synthetic) images and $55,388$ target (real) images. Both source and target domains have the same 12 object categories. Following JUMBOT~\cite{fatras2021unbalanced}, we evaluate all methods on the validation set.

\textbf{Parameter settings for Digits datasets:} We optimize using Adam with the initial learning rate $\eta_{0} = 0.0004 \}$. The number of mini-batches $k$ for each method is set to either 1 or 2, whichever value gives the better performance. We train all algorithms with a batch size of $500$ during $100$ epochs. The hyperparameters in equation (\ref{eq:cost_matrix}) follow the setting in m-UOT: $\alpha = 0.1, \lambda_{t} = 0.1$. For computing UOT, we also set $\tau$ to $1$ and $\epsilon$ to $0.1$. For computing m-POT, the value of $s$ for SVHN to MNIST, USPS to MNIST, and MNIST to USPS is set to 0.85, 0.90, and 0.80, respectively. The entropic regularization coefficient $\epsilon$ of m-POT is chosen from the set $\{ 0, 0.1, 0.2 \}$ based on the specific task.

\textbf{Parameter settings for the Office-Home dataset:} We train all algorithms with a batch size of $65$ during $10000$ iterations. The number of mini-batches $k$ for each method is chosen from the set $\{ 1, 2, 4 \}$, whichever value gives the best performance. Following m-UOT, the hyperparameters for computing the cost matrix is as follows $\alpha = 0.01, \lambda_{t} = 0.5, \tau = 0.5, \epsilon = 0.01$. For m-POT, the value of $s$ is chosen from the set $\{ 0.5, 0.55, 0.6, 0.65, 0.7 \}$. Again, we select the best value of the entropic regularization coefficient $\epsilon \in \{ 0, 0.1, 0.2 \}$ based on the specific task.

\textbf{Parameter settings for the VisDA dataset:} All methods are trained using a batch size of $72$ during $10000$ iterations. The number of mini-batches $k$ for each method is selected from the set $\{ 1, 2, 4 \}$, whichever value gives the best performance. The coefficients in the cost formula is as follows $\alpha = 0.005, \lambda_{t} = 1, \tau = 0.3, \epsilon = 0.01$. The fraction of masses $s$ is set to $0.75$ for m-POT and the entropic regularization coefficient $\epsilon$ is set to $0.004$.

\textbf{Parameter settings for the two-stage implementation}: For computing TS-POT, the fraction of masses $s$ has a fixed value of $0.6$ on the Office-Home dataset and $0.75$ on the VisDA dataset. On both datasets, the entropic regularization coefficient $\epsilon$ is set to 0. In terms of TS-UOT, the value of $\tau$ and $\epsilon$ is the same as that of m-UOT. Other hyperparameters for computing the cost matrix follow the above settings of each dataset.

\textbf{Neural network architectures: } Similar to DeepJDOT and JUMBOT, we use CNN for our generator and 1 FC layer for our classifier. For Office-Home and VisDA, our generator is a ResNet50 pre-trained on ImageNet excluding the last FC layer, which is our classifier. 

Generator architecture was used for the SVHN dataset: \\
$z \in \RR^{32 \times 32 \times 3} \rightarrow Conv_{32} \rightarrow BatchNorm \rightarrow ReLU  \rightarrow Conv_{32} \rightarrow BatchNorm \rightarrow ReLU \rightarrow MaxPool2D \rightarrow Conv_{64} \rightarrow BatchNorm \rightarrow ReLU  \rightarrow Conv_{64} \rightarrow BatchNorm \rightarrow ReLU \rightarrow MaxPool2D \rightarrow Conv_{128} \rightarrow BatchNorm \rightarrow ReLU  \rightarrow Conv_{128} \rightarrow BatchNorm \rightarrow ReLU \rightarrow MaxPool2D \rightarrow Sigmoid \rightarrow FC_{128}$

Generator architecture was used for the USPS dataset: \\
$z \in \RR^{28 \times 28 \times 3} \rightarrow Conv_{32} \rightarrow BatchNorm \rightarrow ReLU  \rightarrow MaxPool2D \rightarrow Conv_{64} \rightarrow BatchNorm \rightarrow ReLU  \rightarrow Conv_{128} \rightarrow BatchNorm \rightarrow ReLU  \rightarrow MaxPool2D \rightarrow Sigmoid \rightarrow FC_{128}$

Classifier architecture was used for both SVHN and USPS datasets: \\
$z \in \RR^{128} \rightarrow FC_{10}$

Classifier architecture was used for the Office-Home dataset: \\
$z \in \RR^{512} \rightarrow FC_{65}$

Classifier architecture was used for the VisDA dataset: \\
$z \in \RR^{512} \rightarrow FC_{12}$

\textbf{Training details: } Similar to both DeepJDOT and JUMBOT, we stratify the data loaders so that each class has the same number of samples in the  mini-batches. For digits datasets, we also train our neural network on the source domain during $10$ epochs before applying our method. For Office-home and VisDA datasets, because the classifiers are trained from scratch, their learning rates are set to be 10 times that of the generator. We optimize the models using an SGD optimizer with momentum = $0.9$ and weight decay = $0.0005$. We schedule the learning rate with the same strategy used in~\cite{ganin2016domain}. The learning rate at iteration $p$ is $\eta_{p} = \frac{\eta_0}{(1 + \mu q)^{\nu}}$, where $q$ is the training progress linearly changing from $0$ to $1$, $\eta_0 = 0.01, \mu = 10, \nu = 0.75$. 

\textbf{Source code: } Our source code is based on \url{https://github.com/kilianFatras/JUMBOT}.

\subsection{Partial Domain Adaptation}
\label{subsec:setting_PDA}
Similar to the deep DA, we use the Office-Home dataset for evaluating our method. ResNet50 is again utilized to extract features. The training details also follow the deep DA experiment. We train each method for 5000 iterations with a batch size of 65. When evaluating the validation/test set, we do not use the ten crop technique to make a fair comparison with other methods. In terms of hyperparameters, we use the same set of hyperparameters for baseline methods. For m-OT, $\alpha = 0.001, \lambda_{t} = 0.0001$ while for m-UOT, $\alpha = 0.003, \lambda_{t} = 0.75, \tau = 0.06, \epsilon = 0.01$. In addition, m-UOT introduces an additional parameter $\eta_3$, which is set to 10, to control the scale of the cost matrix~\citep[Section~5.2]{fatras2021unbalanced}. Because m-OT and m-POT are insensitive to the scale of the cost matrix, $\eta_3$ is set to 1 for both of them. Multiplied by $\eta_3$, the coefficients $\alpha$ and $\lambda_{t}$ of m-POT are set to $0.03$ and $7.5$, respectively. The entropic regularization coefficient of m-POT is selected from $\epsilon \in \{0.5, 0.6, \ldots, 1.5\}$. The fraction of masses $s$ is not set to a constant value in this experiment. Instead, for the first 2500 iterations, its value increases linearly from $0.01$ to $0.325$, then remains constant for the last 2500 iterations.

\subsection{Deep Generative Model}
\label{subsec:setting_GAN}

\textbf{Datasets:} We train generators on CIFAR10~\cite{krizhevsky2009learning} and CelebA~\cite{liu2015faceattributes} datasets. The CIFAR10 dataset contains $10$ classes, with $50,000$ training and $10,000$ testing color images of size $32 \times 32$. CelebA is a large-scale face attributes dataset with more than $200K$ celebrity images. 

\textbf{FID scores: } We use 11000 samples from the generative model and all images from the training set to compute the FID score by the official code from authors in \cite{heusel2017gans}.

\textbf{Parameter settings: } We chose a learning rate equal to $0.0005$. When the mini-batch size $m$ is $100$, the number of epochs is set to $200$ for the CIFAR10 dataset and $100$ for the CelebA dataset. When increasing the mini-batch size to $200$, we double the number of epochs to keep the same number of gradient steps. 
 
\textbf{Neural network architectures: }
We used CNNs for both generators and discriminators on the CelebA and CIFAR10 datasets.


Generator architecture was used for CelebA:\\
$z \in \mathbb{R}^{32} \rightarrow TransposeConv_{512} \rightarrow BatchNorm \rightarrow ReLU  \rightarrow TransposeConv_{256} \rightarrow BatchNorm \rightarrow ReLU \rightarrow TransposeConv_{128} \rightarrow BatchNorm \rightarrow ReLU \rightarrow TransposeConv_{64} \rightarrow BatchNorm \rightarrow ReLU \rightarrow TransposeConv_{3}  \rightarrow Tanh$

Discriminator architecture was used for CelebA:\\
$x \in \mathbb{R}^{64\times 64 \times 3} \rightarrow Conv_{64}  \rightarrow LeakyReLU_{0.2} \rightarrow Conv_{128} \rightarrow BatchNorm \rightarrow LearkyReLU_{0.2} \rightarrow Conv_{256} \rightarrow BatchNorm\rightarrow  LearkyReLU_{0.2}  \rightarrow Conv_{512} \rightarrow BatchNorm\rightarrow Tanh$ 

Generator architecture was used for CIFAR10:\\
$z \in \mathbb{R}^{32} \rightarrow TransposeConv_{256} \rightarrow BatchNorm \rightarrow ReLU  \rightarrow TransposeConv_{128} \rightarrow BatchNorm \rightarrow ReLU \rightarrow TransposeConv_{64} \rightarrow BatchNorm \rightarrow ReLU  \rightarrow TransposeConv_{3}  \rightarrow Tanh$

Discriminator architecture was used for CIFAR10:\\
$x \in \mathbb{R}^{32\times 32 \times 3} \rightarrow Conv_{64}  \rightarrow LeakyReLU_{0.2} \rightarrow Conv_{128} \rightarrow BatchNorm \rightarrow LearkyReLU_{0.2} \rightarrow Conv_{256} \rightarrow Tanh$ 

\subsection{Computational infrastructure}

Color transfer and gradient flow applications are conducted on a MacBook Pro 11 inch M1. While other deep learning experiments are done on NVIDIA Tesla V100 GPUs.

\end{document}

%% file: math_commands.tex
\usepackage{amsmath,amsfonts,bm}

\def\eqref#1{equation~\ref{#1}}

\def\floor#1{\lfloor #1 \rfloor}
\def\1{\bm{1}}

\DeclareMathAlphabet{\mathsfit}{\encodingdefault}{\sfdefault}{m}{sl}
\SetMathAlphabet{\mathsfit}{bold}{\encodingdefault}{\sfdefault}{bx}{n}

\DeclareMathOperator*{\argmax}{arg\,max}